\documentclass[table]{article}

\usepackage{amsmath,amsfonts,bm}

\def\eqref#1{(\ref{#1})}

\def\1{\bm{1}}

\def\vb{{\bm{b}}}

\def\vr{{\bm{r}}}

\def\vu{{\bm{u}}}
\def\vv{{\bm{v}}}

\def\vx{{\bm{x}}}
\def\vy{{\bm{y}}}

\def\mA{{\bm{A}}}

\def\mF{{\bm{F}}}

\def\mI{{\bm{I}}}

\def\mK{{\bm{K}}}
\def\mL{{\bm{L}}}

\def\mP{{\bm{P}}}
\def\mQ{{\bm{Q}}}

\def\mU{{\bm{U}}}

\def\mX{{\bm{X}}}

\DeclareMathAlphabet{\mathsfit}{\encodingdefault}{\sfdefault}{m}{sl}
\SetMathAlphabet{\mathsfit}{bold}{\encodingdefault}{\sfdefault}{bx}{n}

\usepackage{microtype}
\usepackage{graphicx}
\usepackage{booktabs} 
\usepackage{amsmath}
\usepackage{amssymb}
\usepackage{amsthm}
\usepackage{bm}
\usepackage{graphicx}
\usepackage[group-separator={,}]{siunitx}
\usepackage{caption}
\usepackage{subcaption}
\usepackage{scalerel}
\usepackage{stackengine,wasysym}
\usepackage{algorithm}
\usepackage{algorithmic}
\usepackage{enumitem}
\usepackage{siunitx}
\usepackage{multirow}
\usepackage{pifont}
\usepackage[flushleft]{threeparttable}
\usepackage{xcolor}

\usepackage{hyperref}

\definecolor{myred}{HTML}{BA4F27}
\definecolor{mylightred}{HTML}{dda793}
\definecolor{mygreen}{HTML}{00823D}
\definecolor{mylightgreen}{HTML}{80c19e}
\definecolor{myblue}{HTML}{2474B5}
\definecolor{tablelightblue}{HTML}{E1E8FF}
\definecolor{tableblue}{HTML}{AFC2FF}
\definecolor{mypurple}{HTML}{766BBB}
\definecolor{mylightpurple}{HTML}{bbb5dd}

\newcommand{\vth}{\bm{\theta}}
\newcommand{\rawxmark}{\ding{55}}
\newcommand{\xmark}{{\rawxmark}}
\newcommand{\rddagger}{{\color{myred}{\bm\star}}}
\newcommand{\rawcmark}{\ding{51}}
\newcommand{\cmark}{\textcolor{mygreen}{\rawcmark}}
\newcommand*\diff{\mathop{}\!\mathrm{d}}

\newcommand{\appropto}{\mathrel{\vcenter{
  \offinterlineskip\halign{\hfil$##$\cr
    \propto\cr\noalign{\kern2pt}\sim\cr\noalign{\kern-2pt}}}}}

\usepackage[accepted]{icml2024}

\usepackage{amsmath}
\usepackage{amssymb}
\usepackage{mathtools}
\usepackage{amsthm}

\usepackage[capitalize,noabbrev]{cleveref}

\theoremstyle{plain}
\newtheorem{theorem}{Theorem}[section]
\newtheorem{proposition}[theorem]{Proposition}
\newtheorem{lemma}[theorem]{Lemma}
\newtheorem{corollary}[theorem]{Corollary}
\theoremstyle{definition}
\newtheorem{definition}[theorem]{Definition}
\newtheorem{assumption}[theorem]{Assumption}
\theoremstyle{remark}
\newtheorem{remark}[theorem]{Remark}

\usepackage[textsize=tiny]{todonotes}

\icmltitlerunning{Preconditioning for Physics-Informed Neural Networks}

\begin{document}

\twocolumn[
\icmltitle{Preconditioning for Physics-Informed Neural Networks}



\icmlsetsymbol{equal}{*}

\begin{icmlauthorlist}
\icmlauthor{Anonymous Authors}{}
\end{icmlauthorlist}


\icmlcorrespondingauthor{Anonymous Author}{anon.email@domain.com}

\icmlkeywords{Machine Learning, ICML}

\vskip 0.3in
]



\printAffiliationsAndNotice{}  

\begin{abstract}
Physics-informed neural networks (PINNs) have shown promise in solving various partial differential equations (PDEs). However, training pathologies have negatively affected the convergence and prediction accuracy of PINNs, which further limits their practical applications. In this paper, we propose to use condition number as a metric to diagnose and mitigate the pathologies in PINNs. Inspired by classical numerical analysis, where the condition number measures sensitivity and stability, we highlight its pivotal role in the training dynamics of PINNs. We prove theorems to reveal how condition number is related to both the error control and convergence of PINNs. Subsequently, we present an algorithm that leverages preconditioning to improve the condition number. Evaluations of 18 PDE problems showcase the superior performance of our method. Significantly, in 7 of these problems, our method reduces errors by an order of magnitude. These empirical findings verify the critical role of the condition number in PINNs' training. The codes are included in the supplementary material.
\end{abstract}


\section{Introduction}
Numerical methods, such as finite difference and finite element methods, discretize partial differential equations (PDEs) into linear equations to obtain approximate solutions. Such discretizations can be computationally expensive, especially for PDE-constrained problems that require frequently solving PDEs. Recently, physics-informed neural network (PINN) \citep{raissi2019physics} and its extensions \citep{pang2019fpinns,yang2021b,liu2022unified} have emerged as powerful tools for tackling these challenges. By integrating PDE residuals into the loss function, PINNs not only ensure that the neural network adheres to the physical constraints but also maintain its adaptability to specific optimization objectives (e.g., minimum dissipation) in applications such as inverse problems \citep{chen2020physics,jagtap2022physics} and physics-informed reinforcement learning (PIRL) \citep{liu2021physics,martin2022reinforcement}. While PINNs have achieved success over various domains \citep{zhu2021machine,cai2021physics,huang2022applications}, their full potential and capabilities remain under-explored.

\begin{figure}[bt]
     \centering
     \begin{subfigure}[b]{0.45\textwidth}
         \centering
         \includegraphics[height=0.45\textwidth]{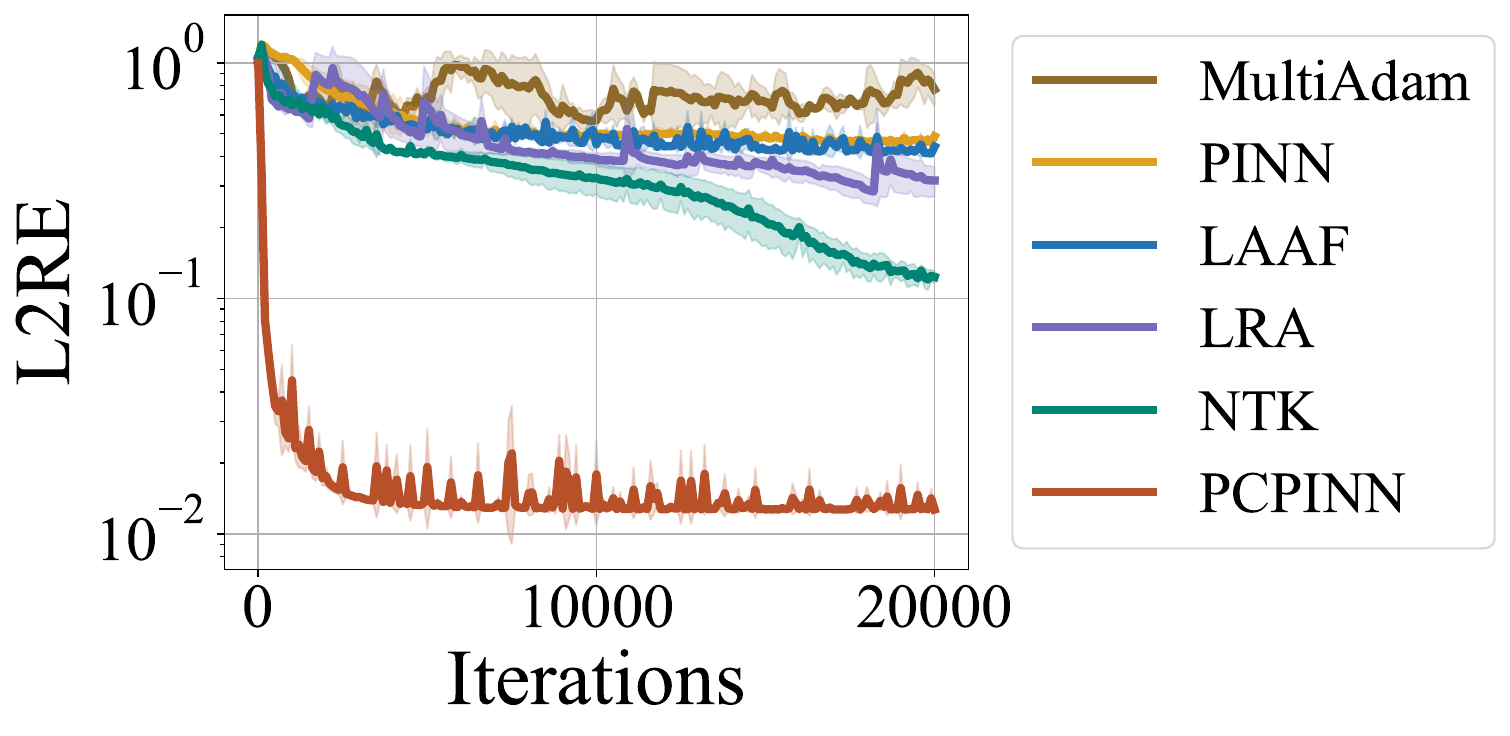}
         \caption{Convergence dynamics: mean $\pm$ std}
         \label{fig:intro:1}
     \end{subfigure}
     \hfill
     \begin{subfigure}[b]{0.5\textwidth}
         \centering
         \includegraphics[height=0.49\textwidth]{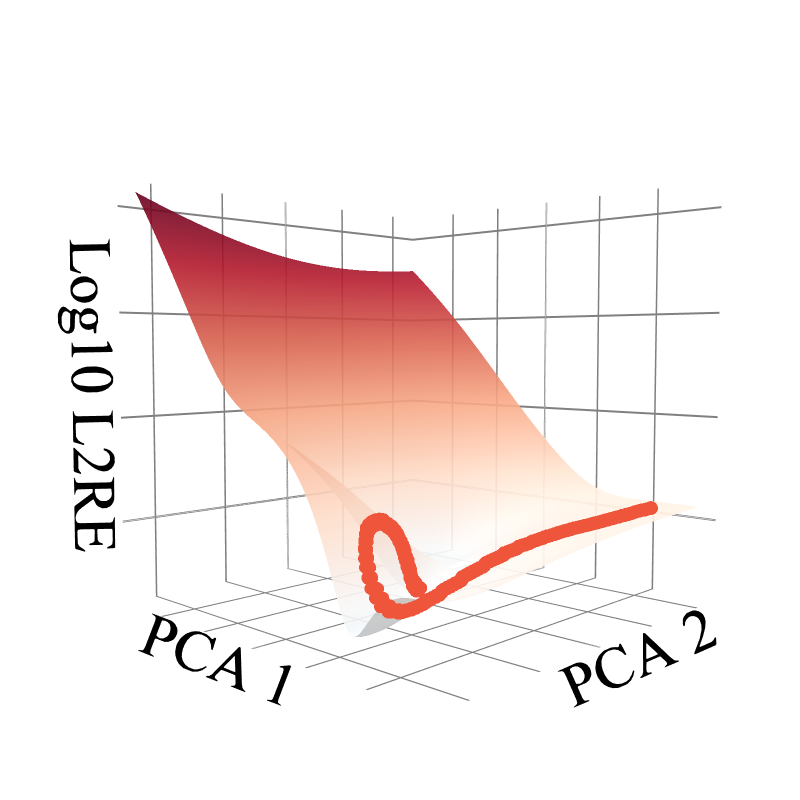}
         \includegraphics[height=0.49\textwidth]{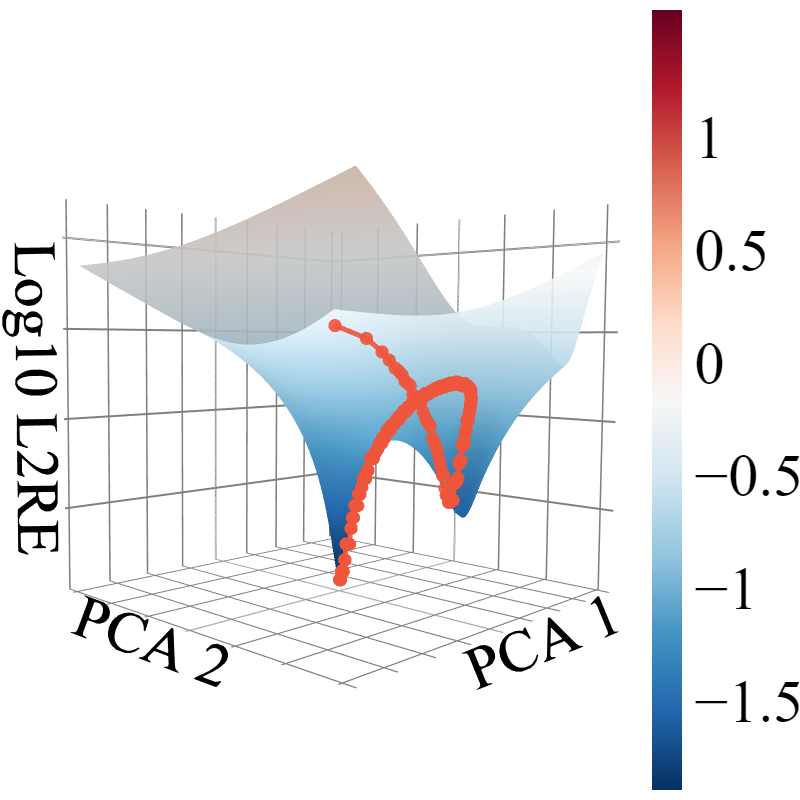}
         \caption{Error landscape: PINN (left) vs. Ours (right)}
         \label{fig:intro:2}
     \end{subfigure}
\caption{An illustrative example of learning 1D wave equation. \textbf{(a)} PINN baselines (only a subset are shown) struggle with long plateaus and severe oscillations during training. In contrast, our preconditioned PINN (PCPINN) can converge quickly and achieve much lower $L^2$ relative error (L2RE). \textbf{(b)} PINN wanders in the high-error zone (red), while ours dives deep and eventually converges. Red scatters mark the model parameters in each iteration. Details are elaborated in Section~\ref{sec:exp:forward}.}
\label{fig:intro}
\end{figure}

Several studies \citep{mishra2022estimates,de2022error,de2022error2,guo2022energy} have theoretically demonstrated the feasibility of PINNs in addressing a vast majority of well-posed PDE problems. Yet, \citet{krishnapriyan2021characterizing} spotlights \emph{training pathologies} inherent to PINNs and shows their failures in even moderately complex problems\footnote{The term ``complex problems'' is employed here to describe PDEs characterized by nonlinearity, irregular geometries, multi-scale phenomena, or chaotic behaviors. For an in-depth discussion, we refer to \citet{hao2022physics}.} encountered in real-world scenarios. As illustrated in Figure~\ref{fig:intro}, such pathologies can substantially hinder convergence and decrease prediction accuracy. Some researchers attribute the pathologies to the unbalanced competition between PDE and boundary condition (BC) loss terms \citep{wang2021understanding,wang2022and}. Based on this analysis, others have proposed methods to enforce the BCs on the PINN, eliminating BC loss terms \citep{berg2018unified,sheng2021pfnn,lu2021physics,sheng2022pfnn,liu2022unified}. However, the challenge persists as the unbalanced competition only partially explains pathologies, especially when dealing with complex PDEs like the Navier-Stokes equations \citep{liu2022unified}. Thus, how to understand and effectively mitigate these pathologies remains open.

In this work, we introduce the condition number as a novel metric, motivated by its pivotal role in understanding computational stability and sensitivity, to measure training pathologies in PINNs. Further, we present an algorithm to optimize this metric, enhancing both accuracy and convergence. In traditional numerical analysis, the condition number characterizes the sensitivity of a problem's output relative to its input. A large condition number typically indicates a high sensitivity to noises and errors, resulting in a slow and unstable convergence. This insight is particularly relevant in deep learning's complex optimization landscape. In this context, the condition number becomes a vital tool to identify potential convergence issues. Based on this background, we suggest resorting to condition numbers to analyze the training pathologies of PINNs.


Specifically, we theoretically demonstrate that a lower condition number correlates with improved error control. Through the lens of the neural tangent kernel (NTK), we further show that the condition number plays a decisive role in the convergence speed of PINN. Based on these findings, we propose an algorithm that mitigates the condition number by incorporating a preconditioner into the loss function. To validate our theoretical framework, we evaluate our approach on a comprehensive PINN benchmark \citep{hao2023pinnacle}, which encompasses $20$ distinct forward PDEs and $2$ inverse scenarios. Our results consistently show state-of-the-art performance across most test cases. Notably, our method makes several previously unsolvable problems with PINNs (e.g., a 3D Poisson equation with intricate geometry) solvable by reducing relative errors from nearly $100\%$ to below $25\%$.


\section{Preliminaries}

We start by presenting the problem formulation and reviewing physics-informed neural networks (PINNs). We consider low-dimensional boundary value problems (BVPs) \footnote{Although not discussed, our method readily extends to problems involving vector-valued functions and more general boundary conditions. Relevant experimental details can be found in Appendix~\ref{app:exp:forward}.} that expect a solution $u$ satisfying that:
\begin{equation}
\mathcal{F}[u] = f \quad\text{in } \Omega,
\label{eq:bvp}
\end{equation}
with a boundary condition (BC) of $u|_{\partial\Omega}=g$, where $\Omega$ is an open, bounded subset of $\mathbb{R}^d$ with dimension $d\le 4$. Here, $f\colon \Omega \rightarrow \mathbb{R}$ and $g\colon \partial\Omega \rightarrow \mathbb{R}$ are known functions; $\mathcal{F}\colon V \rightarrow W$ is a partial differential operator including at most $k$-order partial derivatives, where $k\in \mathbb{N}^+$ and $V,W$ are normed subspaces of $L^2(\Omega)$.

Assuming the well-posedness of our BVP, a fundamental property of formulations for physical problems, as indicated by \citet{hilditch2013introduction}, we can find a subspace $S \subset \mathcal{F}(V)$. For every $w\in S$, there exists a unique $v\in V$ such that $\mathcal{F}[v] = w$ and that $v|_{\partial\Omega}=g$ (that is, the BC). This allows us to define $\mathcal{F}^{-1}\colon S \rightarrow V$ as $\mathcal{F}^{-1}[w] = v$. Again, owing to the well-posedness, $\mathcal{F}^{-1}$ is continuous within $S$. Conclusively, our solution can be expressed as $u=\mathcal{F}^{-1}[f]$.

PINNs use a neural network $u_{\vth}$ with parameters $\vth\in \Theta$ to approximate the solution $u$, where $\Theta=\mathbb{R}^n$ represents the parameter space and $n\in \mathbb{N}^+$ is the number of parameters. The optimization problem of PINNs can be formalized as a constrained optimization problem:
\begin{equation}\label{eq:hcopt}
    \min_{\vth\in\Theta} \left\| \mathcal{F}[u_{\vth}] - f \right\|, \quad \text{subject to } u_{\vth}|_{\partial\Omega}=g.
\end{equation}
Two primary strategies to enforce the BC constraint are:
\begin{equation}
\begin{aligned}
    \mathcal{L}_{\text{soft}}(\vth)&= \left\| \mathcal{F}[u_{\vth}] - f \right\|^2 + \alpha \| u_{\vth}-g \|_{\partial\Omega}^2 \\
    \mathcal{L}_{\text{hard}}(\vth)&=  \left\| \mathcal{F}[\hat{u}_{\vth}] - f \right\|^2,
\end{aligned}
\end{equation}
where $\alpha \in \mathbb{R}^+$, $\| \cdot \|_{\partial\Omega}$ denotes the $L^2$ norm evaluated at $\partial\Omega$, and all the norms are estimated via Monte Carlo integration. The first approach adds a penalty term for BC enforcement. However, as highlighted by \cite{wang2021understanding}, this can induce loss imbalances, leading to training instability. In contrast, the second approach, as advocated by \citep{berg2018unified, lu2021physics, liu2022unified}, employs a specialized ansatz: $\hat{u}_{\vth}(\vx) = l^{\partial\Omega}(\vx) u_{\vth}(\vx) + g(\vx)$, with $l^{\partial\Omega}$ being a smoothed distance function to $\partial\Omega$. Such ansatz naturally adheres to the BC, eliminating loss imbalances. We favor this strategy and, for clarity, will subsequently omit the hat notation, assuming $u_{\vth}$ fulfills the BC.



\paragraph{Training Pathologies.}
Despite hard-constraint methods, training pathologies still occur in moderately complex PDEs \citep{liu2022unified}. As noted by \cite{krishnapriyan2021characterizing}, minor imperfectness during optimization can lead to an unexpectedly large error, substantially destabilizing training. Our subsequent analysis will delve further into such pathologies.






\section{Analyzing PINNs' Training Pathologies via Condition Number}\label{sec:analyze}

\subsection{Introducing Condition Number}\label{sec:model}

In the field of numerical analysis, condition number has long been a touchstone for understanding the problem's pathological nature \citep{suli2003introduction}. For instance, in linear algebra, the condition number of a matrix provides insights into the error amplification from input to output, thus indicating potential stability issues. Furthermore, in deep learning, the condition number can be used to characterize the sensitivity of the network prediction. A ``sensitive'' model could be vulnerable to some adversarial noise \citep{beerens2023adversarial}.


Drawing inspiration from this knowledge, we propose to use condition numbers to analyze PINNs' training pathologies, offering a fresh perspective on their behavior.

\begin{definition}[\textbf{Condition Number}]\label{def:cond}
For the boundary value problem (BVP) in Eq.~\eqref{eq:bvp}, denoted by $\mathcal{P}$, by assuming the neural network has sufficient approximation capability (see Assumption~\ref{ass:7}), 
the \emph{relative} condition number for solving $\mathcal{P}$ with a PINN is defined as:
\begin{equation}\label{eq:cond}
    \mathrm{cond}(\mathcal{P}) = \lim_{\epsilon\to 0^+} \sup_{\substack{0<\| \delta f \| \le \epsilon\\ \vth\in \Theta}} \frac{\| \delta u \| \big/ \| u \|}{\| \delta f \| \big/ \| f \|},
\end{equation}
provided $\| u \| \neq 0$, $\| f \| \neq 0$\footnote{If $\| u \| = 0$ or $\| f \| = 0$, we can similarly define the \emph{absolute} condition number by removing the two terms.}, where $\delta u = u_{\vth} - u$ and $\delta f =  \mathcal{F}[u_{\vth}] - f$. 
\end{definition}

\begin{remark}
The condition number signifies the asymptotic worst-case relative error in prediction for a relative error in optimization (noticing that $\mathcal{L}(\vth)=\| \delta f \|^2$). The problem is said to be \emph{ill-conditioned} if the condition number is large, indicating that a small optimization imperfectness can result in a large prediction error. Since gradient descent has certain inherent errors, it will be difficult for the neural network to approximate the exact solution.

\end{remark}

Aligning with the observation that most real-world physical phenomena exhibit smooth behavior with respect to their sources, we assume that $\mathcal{F}^{-1}$ is locally Lipschitz continuous and present the subsequent theorem.




\begin{theorem}\label{theo:cond:bound}
If $\mathcal{F}^{-1}$ is $K$-Lipschitz continuous with $K\ge 0$ in some neighbourhood of $f$, we have: 
\begin{equation}
    \mathrm{cond}(\mathcal{P}) \le \frac{\| f \|}{\| u \|} K.
\end{equation}
\end{theorem}

\begin{proof}
We defer the proof to Appendix~\ref{app:cond:bound}.
\end{proof}



\begin{remark}
It is worth emphasizing that $K$ fundamentally depends on the intrinsic nature of the problem and it is independent of the specific algorithm. Consequently, algorithmic enhancements, whether in network architecture or training strategy, may not substantially mitigate the pathology unless the problem is reformulated.
\end{remark}



For specific cases such as linear PDEs, we could have weaker theorems to guarantee the condition number's existence (refer to Appendix~\ref{app:prop}).



To give readers a more specific understanding of condition numbers, we consider a simple model problem of the 1D Poisson equation:
\begin{equation}\label{eq:model:bvp}
\begin{aligned}
    \Delta u(x) &= f(x), &&x\in \Omega=(0, 2\pi/P ),\\
    u(x) &= 0, &&x\in \partial\Omega=\{0, 2\pi/P\},
\end{aligned}
\end{equation}
where $P$ is a system parameter. In this simple scenario, we can derive an analytical expression for the condition number. Firstly, we present an analytical expression for the norm of $\mathcal{F}^{-1}$.
\begin{theorem}\label{theo:poisson_condnumber}
    Consider the function spaces $V = H^2(\Omega)$ and $W = L^2(\Omega)$. Let $\mathcal{F}$ denote the Laplacian operator mapping from $V$ to $W$, i.e., $\mathcal F=\Delta: V\to W$. Define the inverse operator $\mathcal{F}^{-1}\colon \mathcal{F}(V) \rightarrow V$ such that for every $w \in \mathcal{F}(V)$,  $\mathcal{F}^{-1}[w] = v$, where $v\in V$ is the unique function satisfying $\mathcal{F}[v] = w$ with boundary condition $ v(0)=v(2\pi / P)=0$. Then, the norm of $\mathcal{F}^{-1}$ is:
    \begin{equation}\label{eq:poisson_condnumber}
            \|\mathcal F^{-1}\| = \frac4{P^2}.
    \end{equation}
\end{theorem}
\begin{proof}
    For a detailed derivation, refer to Appendix \ref{app:poisson_condnumber}.
\end{proof}

        


Secondly, according to Proposition~\ref{theo:prop}, the condition number is given by $\mathrm{cond}(\mathcal{P}) = \frac{\| f \|}{\| u \|} \| \mathcal{F}^{-1} \| = \frac{4\|f\|}{P^2\|u\|}$. Although this example is foundational, it sheds light on the relationship between the condition number and the intrinsic problem property. What is more, in Section~\ref{sec:exp:cond}, we will delve deeper, exploring three more practical problems and study how to numerically estimate the condition number when the analytical expression is not available.

\subsection{How Condition Number Affects Error \& Convergence}

Next, we will discuss the relationship between the condition number and the error control as well as the convergence rate of PINNs.

\begin{corollary}[\textbf{Error Control}]\label{theo:error_control}
Assuming that $\mathrm{cond}(\mathcal{P}) < \infty$, there exists a function $\alpha\colon (0, \xi) \rightarrow \mathbb{R}, \xi > 0$ with $\lim_{x \to 0^+} \alpha(x) = 0$, such that for any $\epsilon \in (0, \xi), \vth \in \Theta \land \sqrt{\mathcal{L}(\vth)} \le \epsilon$, it holds that:
\begin{equation}
    \frac{\| u_{\vth} - u \|}{\| u \|} \le 
    \left(\mathrm{cond}(\mathcal{P}) + \alpha(\epsilon) \right) \frac{\sqrt{\mathcal{L}(\vth)}}{\| f \|}.
\end{equation}
\end{corollary}

\begin{proof}
This theorem can be derived directly from Definition~\ref{def:cond} (see Appendix~\ref{app:error_control} for details).
\end{proof}

\begin{remark}
For well-posed BVPs, it is known that there is no error when the loss $\mathcal{L}(\vth)$ is precisely zero. However, the magnitude of the error is uncontrolled when $\mathcal{L}(\vth)$ is a small (but non-zero) value due to optimization errors. This theorem bridges the gap between the error and the loss value by establishing an asymptotic relationship, where the condition number serves as a scaling factor. Consequently, improving the condition number becomes a critical step to ensuring greater accuracy, as empirically validated in our experiment (see Section~\ref{sec:exp:forward}, \textbf{effect of preconditioner precision}).
\end{remark}




Then, we will study how the condition number affects the convergence of PINNs through the lens of the neural tangent kernel (NTK) theory \citep{jacot2018ntk,wang2022ntk}. Firstly, we discretize the loss function $\mathcal{L}(\vth)$ on a set of collocation points $\{ \vx^{(i)} \}_{i=1}^N$:
\begin{equation}
    \mathcal{L}(\vth) \appropto \mathcal{\hat{L}}(\vth) = \frac{1}{2}\| \mathcal{F}[u_{\vth}](\mX) - f(\mX) \|^2,
\end{equation}
where $\mX\in \mathbb{R}^{N\times d} = [\vx^{(1)},\dots, \vx^{(N)}]^\top$. We consider optimizing the discretized loss function $\mathcal{\hat{L}}(\vth)$ with an infinitesimally small learning rate, which yields the following continuous-time gradient flow:
\begin{equation}
    \frac{\diff \vth}{\diff t} = -\nabla \mathcal{\hat{L}}(\vth), \quad t\in (0,+\infty),
\end{equation}
where $\vth = \vth(t), t\in [0, +\infty)$ and $\vth(0)$ is the randomly initialized parameters. 

Secondly, we define the NTK for PINNs $\mK(t) \in \mathbb{R}^{N\times N}$ in this context:
\begin{equation}
    \mK_{ij}(t) = \frac{\partial \mathcal{F}[u_{\vth(t)}](\vx^{(i)})}{\partial \vth} \cdot \frac{\partial \mathcal{F}[u_{\vth(t)}](\vx^{(j)})}{\partial \vth},
\end{equation}
where $1 \le i,j \le N, t\in [0,+\infty)$. According to the NTK theory \citep{jacot2018ntk, wang2022ntk}, the following evolution dynamics holds in the gradient flow:
\begin{equation}\label{eq:24}
    \frac{\partial \mathcal{F}[u_{\vth(t)}](\mX)}{\partial t} = -\mK(t) (\mathcal{F}[u_{\vth(t)}](\mX) - f(\mX)),
\end{equation}
where $t\in (0, +\infty)$. From \citet{jacot2018ntk, wang2022ntk}, $\mK(t)$ nearly stays invariant during the training process when the width of PINNs approaches infinity:
\begin{equation}
    \mK(t) \approx \mK^{\infty}, \quad  t\in[0,+\infty),
\end{equation}
where $\mK^{\infty}$ is a fixed kernel. Therefore, Eq.~\eqref{eq:24} can be further rewritten as:
\begin{equation}
     \mathcal{F}[u_{\vth(t)}](\mX) \approx \left( \mI - e^{-\mK(t)t} \right) f(\mX).
\end{equation}

Thirdly, since $\mK(t)$ is positive semi-definite \citep{wang2022ntk} and is nearly time-invariant, we can take its spectral decomposition and make the orthogonal part time-invariant: $\mK(t) \approx \mQ^\top \Lambda(t) \mQ$, where $\mQ$ is a time-invariant orthogonal matrix and $\Lambda(t)$ is a diagonal matrix with entries being the eigenvalues $\lambda_i(t) \ge 0$ of $\mK(t)$. Consequently, we can further derive that:
\begin{equation}
    \mathcal{F}[u_{\vth(t)}](\mX) - f(\mX) \approx - \mQ^\top e^{-\Lambda(t)t} \mQ f(\mX),
\end{equation}
which is equivalent to:
\begin{equation}
    \mQ\left(\mathcal{F}[u_{\vth(t)}](\mX) - f(\mX) \right) \approx -e^{-\Lambda(t)t} \mQ f(\mX).
\end{equation}
The equation suggests that the $i$-th element of the left-hand side will diminish approximately at the rate of $e^{-\lambda_i(t)t}$. Therefore, the eigenvalues of the kernel will serve as critical factors, characterizing the rate at which the training loss declines. As suggested by \citet{wang2022ntk}, this motivates us to adopt the following definition. 

\begin{definition}[\textbf{Average Convergence Rate}]\label{def:avgrate}
The average convergence rate $c(t)$ of a positive semi-definite kernel matrix $\mK(t)\in \mathbb{R}^{N\times N}$ is defined as taking the average of all its eigenvalues:
\begin{equation}
    c(t) = \frac{1}{N} \sum_{i=1}^N \lambda_i(t) = \frac{1}{N} \mathrm{tr}(\mK(t)).
\end{equation}
\end{definition}

Finally, we prove that a lower bound of the average convergence rate $c(t)$ is determined by the condition number.

\begin{theorem}[\textbf{Convergence Rate}]\label{theo:ntk}
Let $U$ be a set such that $\{ u_{\vth(t)} \mid t\in [0,+\infty)\} \subset U$. Suppose that $\mathcal{F}^{-1}$ is well-defined and Fréchet differentiable in $\mathcal{F}(U)$. Under the assumption that $\mathrm{cond}(\mathcal{P}) < \infty$ and other assumptions in the NTK \citep{jacot2018ntk,wang2022ntk}, the average convergence rate $c(t)$ at time $t$ satisfies that:
\begin{equation}
    c(t) \gtrapprox \underbrace{\frac{\|f\|^2/(\|u\|^2 | \Omega |)}{(\mathrm{cond}(\mathcal{P}))^2 + \alpha(\mathcal{L}(\vth(t))) }}_{\text{condition number and physics}}
    \quad
    \underbrace{\left\|   \frac{\partial u_{\vth(t)}}{\partial \vth} \right\|^2}_{\text{neural network}},
\end{equation}
where $\alpha\colon (0, \xi) \rightarrow \mathbb{R}, \xi > \sup_{t\in [0,+\infty)} \mathcal{L}(\vth(t))$ with $\lim_{x \to 0^+} \alpha(x) = 0$.
\end{theorem}

\begin{proof}
The complete proof is given by Appendix~\ref{app:ntk}.
\end{proof}

\begin{remark}
According to the above theorem, a small condition number could greatly accelerate the convergence. We empirically validate this finding in Section~\ref{sec:exp:cond}.
\end{remark}



\section{Training PINNs with a Preconditioner}\label{sec:algorithm}


In this section, we present a preconditioning method to improve the condition number inherent to the PDE problem addressed by PINNs, thereby enhancing prediction accuracy and convergence.

\paragraph{Discretization of PDEs.}
We begin with well-posed linear BVPs defined on a rectangular domain $\Omega$, where the differential operator $\mathcal{F}$ is linear. We employ the finite difference method (FDM) to discretize the BVP on a $N$-point uniform mesh $\{ \vx^{(i)} \}_{i=1}^{N}$: $\mA \vu = \vb$. Here, $\mA\in \mathbb{R}^{N\times N}$ is an invertible sparse matrix, $\vu=( u(\vx^{(i)}) )_{i=1}^{N}$\footnote{To be precise, due to errors in the numerical format, $\vu$ is only approximately equal to the values of the true solution $u$ at corresponding points.}, and $\vb=( f(\vx^{(i)}) )_{i=1}^{N}$.

\paragraph{Preconditioning Algorithm.}
For slightly complex problems, the condition number may reach the level of $10^3$ (see Section~\ref{sec:exp:cond}). To improve it, a preconditioning algorithm is employed to compute a matrix $\mP$ to construct an equivalent linear system: $\mP^{-1}\mA \vu = \mP^{-1}\vb$. Prevalent preconditioning algorithms such as incomplete LU (ILU) factorization (i.e., $\mP=\widehat{\mL}\widehat{\mU}\approx \mA$, where $\widehat{\mL}, \widehat{\mU}$ are sparse invertible lower and upper triangular matrices, respectively) can reduce the condition number by several orders of magnitude while keeping the time cost much cheaper than solving $\mA \vu = \vb$ \citep{shabat2018randomized}. This can be formulated as:
\begin{equation}\label{eq:original_cond}
\begin{aligned}
\mathrm{cond}(\mathcal{P}) \approx \frac{\| \vb \|}{\| \vu \|} \| \mA^{-1} \|  &\longrightarrow \frac{\| \mP^{-1}\vb \|}{\| \vu \|} \| \mA^{-1}\mP \| \\
&\approx \frac{\| \mA^{-1}\vb \|}{\| \vu \|} \| \mA^{-1}\mA \| = 1,
\end{aligned}
\end{equation}
where $\| \cdot \|$ is the $L^2$ vector/matrix norm. A detailed derivation is provided in Appendix~\ref{app:original_cond}. Finally, we can train PINNs with precomputed preconditioners as displayed in Algorithm~\ref{alg}.

\begin{algorithm}[tb]
   \caption{Training PINNs with a preconditioner}
   \label{alg}
\begin{algorithmic}[1]
   \STATE {\bfseries Input:} number of iterations $K$, mesh size $N$, learning rate $\eta$, and initial parameters $\vth^{(0)}$
   \STATE {\bfseries Output:} optimized parameters $\vth^{(K)}$
   \STATE Generate a mesh $\{ \vx^{(i)} \}_{i=1}^{N}$ for the problem domain $\Omega$
   \STATE Assemble the linear system $\mA, \vb$, where $\mA$ is a sparse matrix
   \STATE Compute the preconditioner $\mP = \widehat{\mL}\widehat{\mU}$ via ILU, where $\widehat{\mL}, \widehat{\mU}$ are both sparse matrices
   \FOR{$k=1,\dots,K$}
        \STATE Evaluate the neural network $u_{\vth^{(k-1)}}$ on mesh points to obtain: $\vu_{\vth^{(k-1)}}=( u_{\vth^{(k-1)}}(\vx^{(i)}) )_{i=1}^{N}$
        \STATE Compute the loss function $\mathcal{L}^\dagger(\vth^{(k-1)})$ using:
            \begin{equation}
            \begin{aligned}
            \mathcal{L}^\dagger(\vth) &= \left\| \mP^{-1} (\mA \vu_{\vth} - \vb)  \right\|^2 \\
            &= \left\| \widehat{\mU}^{-1} \widehat{\mL}^{-1} (\mA \vu_{\vth} - \vb)  \right\|^2,
            \end{aligned}
        \end{equation}
        which incorporates the following steps:
        \begin{enumerate}[label=(\alph*)]
            \item Compute the residual $\vr\leftarrow \mA \vu_{\vth^{(k-1)}} - \vb$
            \item Solve $\widehat{\mL}\vy = \vr$ and let $\vr \leftarrow \vy$, which should be very fast since $\widehat{\mL}$ is sparse
            \item Solve $\widehat{\mU}\vy = \vr$ and let $\vr \leftarrow \vy$
            \item Compute $\mathcal{L}^\dagger(\vth^{(k-1)})= \| \vr \|^2$
        \end{enumerate}
        \STATE Update the parameters via gradient descent: $\vth^{(k)} \leftarrow \vth^{(k-1)} - \eta \nabla_{\vth} \mathcal{L}^\dagger(\vth^{(k-1)})$
   \ENDFOR
   
   {\bfseries Note:} In our implementation, there is no requirement to design a hard-constraint ansatz for $u_{\vth}$ to adhere to the boundary conditions (BC).  This is because our linear equation $\mA \vu = \vb$ inherently encompasses the BC. Further details can be found in Appendix~\ref{app:algo:bc}.
\end{algorithmic}
\end{algorithm}

\paragraph{Time-Dependent \& Nonlinear Problems.}
While our primary focus in this section is on linear and time-independent PDEs, our approach is readily extended to handle both time-dependent and nonlinear problems with moderate adaptations. For time-dependent cases, there are strategies like treating time as an additional spatial dimension or a time-stepping iterative approach. As for nonlinear problems, techniques involve moving nonlinear terms to the bias $\vb$ or utilizing iterative methods such as the Newton-Raphson method. We have elaborated on these adaptation strategies in Appendix~\ref{app:algo:tmnl} for further reading.

\paragraph{Non-Uniform Mesh \& Modern Numerical Schemes.} While we employed the FDM with a uniform mesh to simplify the formulation, it is essential to emphasize that this choice does not restrict our method's adaptability. In our implementation, we leverage more modern numerical schemes, such as the finite element method (FEM) paired with a non-uniform mesh. To align the theory with this implementation, some definitions, including norms, may need to be adjusted to a minor extent. For instance, a non-uniform mesh might demand a norm definition like $\| \cdot \| = ( \int_{\Omega} |  w(\vx) \cdot (\cdot)|^2 \diff{\vx} )^{\frac{1}{2}}$, where $w\colon \Omega \rightarrow \mathbb{R}$ represents a weight function.



\begin{figure*}[bt]
     \centering
     \begin{subfigure}[b]{0.35\textwidth}
         \centering
         \includegraphics[height=0.7\textwidth]{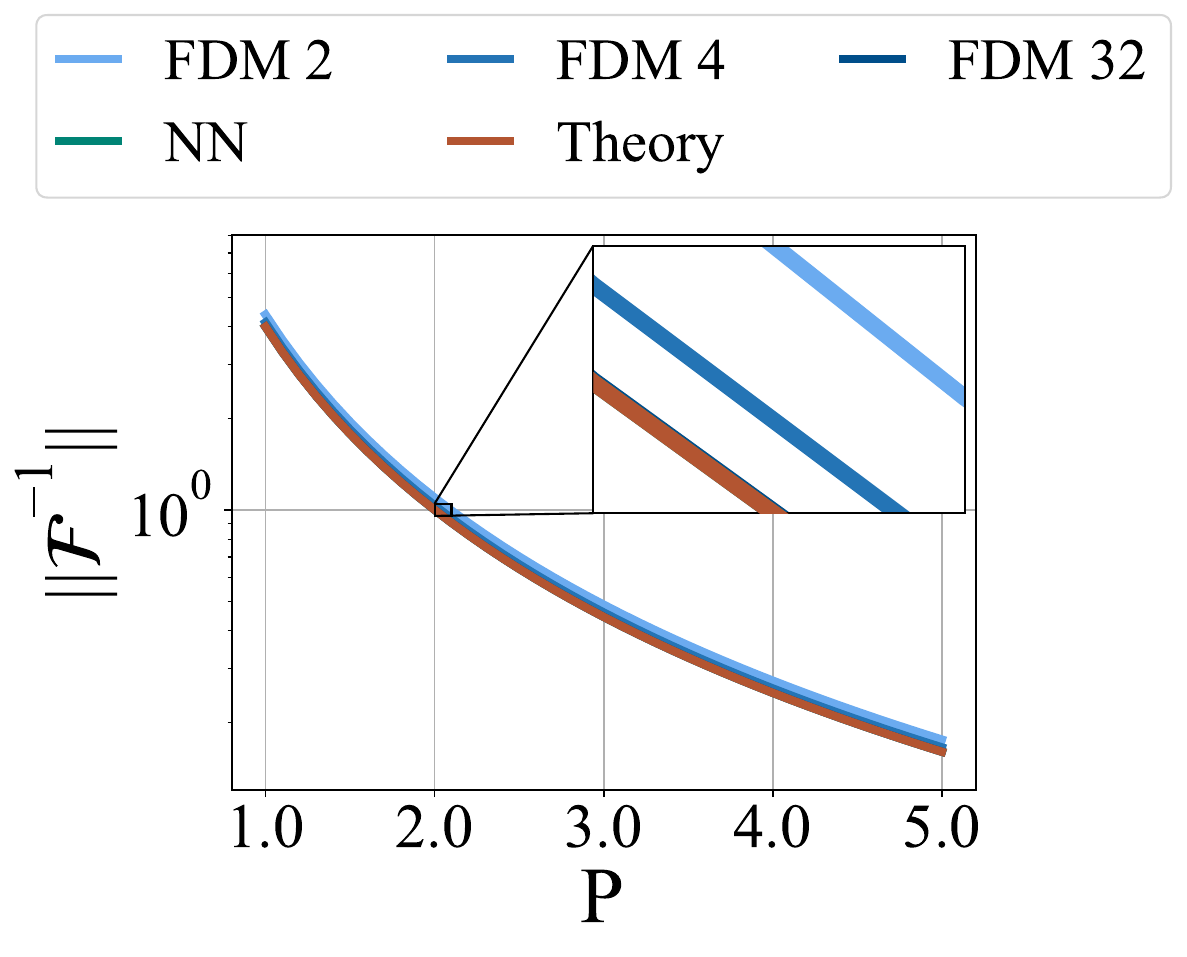}
         \caption{Estimation of $\| \mathcal{F}^{-1} \|$ vs. $P$}
         \label{fig:theory_numerical_cond}
     \end{subfigure}
     \begin{subfigure}[b]{0.35\textwidth}
         \centering
         \includegraphics[height=0.7\textwidth]{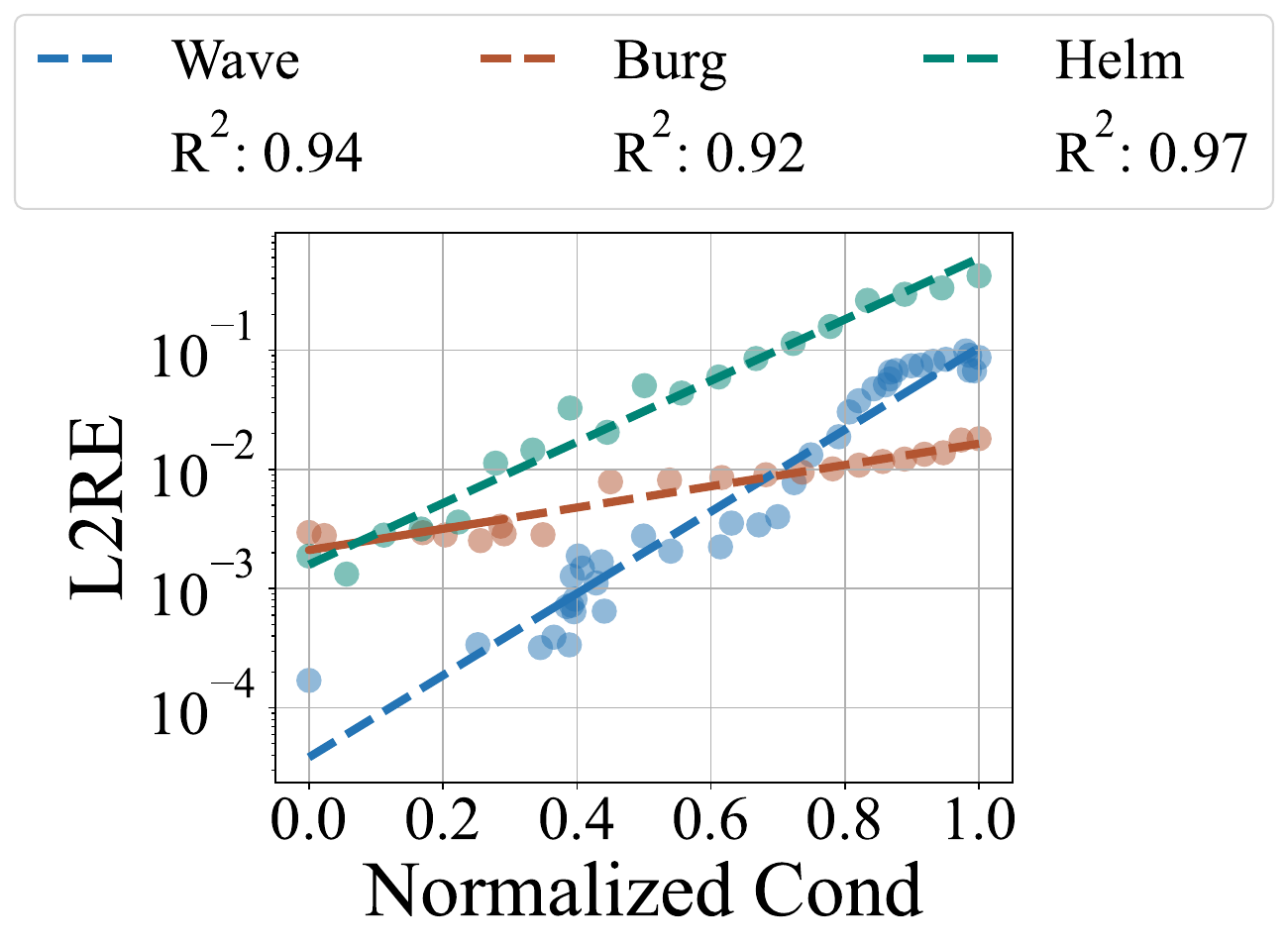}
         \caption{L2RE vs. $\mathrm{cond}(\mathcal{P})$}
         \label{fig:cond_convergence}
     \end{subfigure}
     \begin{subfigure}[b]{0.28\textwidth}
         \centering
         \includegraphics[height=0.76\textwidth]{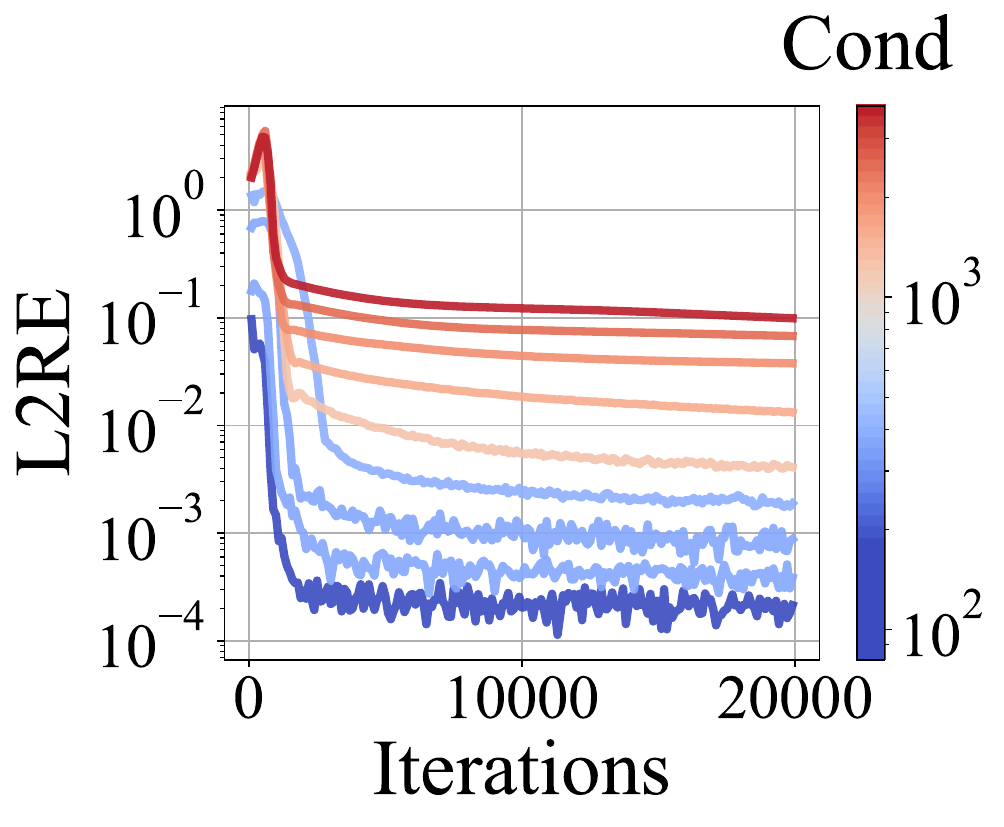}
         \caption{Convergence dynamics}
         \label{fig:cond_errorhistory}
     \end{subfigure}
\caption{\textbf{(a):} Estimations of $\| \mathcal{F}^{-1} \|$ across different $P$ values, with the number after ``FDM'' indicating the mesh size. \textbf{(b):} Strong linear correlation between normalized condition numbers and associated errors. \textbf{(c):} Convergence in the wave equation across different condition numbers.}
\label{fig:exp_relationship}
\end{figure*}

\section{Numerical Experiments}\label{sec:exp}

\subsection{Overview}
In this section, we design numerical experiments to address the following key questions:
\begin{itemize}
    \item \textbf{Q1:} How can we calculate the condition number, and can it characterize pathologies affecting PINNs' prediction accuracy and convergence?

    In Section~\ref{sec:exp:cond}, we propose two estimation methods, validated on a problem with a known analytic condition number. We then apply these methods to approximate the condition number for three practical problems and study its relationship to PINNs' performance. Our results underscore a strong correlation, indicating the correctness of our theory.

    \item \textbf{Q2:} Can the proposed preconditioning algorithm improve the pathology, thereby boosting the performance in solving PDE problems?

    In Section~\ref{sec:exp:forward}, we evaluate our preconditioned PINN (PCPINN) on a comprehensive PINN benchmark \citep{hao2023pinnacle} encompassing 18 PDEs from diverse fields. Employing the $L^2$ relative error (L2RE) as a primary metric (and MSE, L1RE as auxiliary ones), our approach sets a new benchmark: it reduces the error for 7 problems by a magnitude and makes 2 previously unsolvable (L2RE $\approx 100\%$) problems solvable. 

    \item \textbf{Q3:} Does our method require extensive computation time?
    
Figure~\ref{fig:time:1} demonstrates that our approach is comparable to PINNs in terms of computational efficiency and even outpaces it in some cases. Furthermore, although Figure~\ref{fig:time:2} shows that neural network-based methods may not yet be able to outperform traditional solvers in speed, they show promising advantages in the scaling law. This shows that neural networks have potentially significant speed advantages when solving larger problems.

    
\end{itemize}

Besides, in Appendix~\ref{app:exp:real:abla}, we perform extensive ablation studies on hyperparameters to demonstrate the robustness of our method. In Appendix~\ref{app:exp:inverse}, we study two inverse problems to showcase the effectiveness of our method over the traditional adjoint method and the SOTA PINN baseline. The supplementary experimental materials are deferred in Appendix~\ref{app:exp:cond}, \ref{app:exp:forward}, and Appendix~\ref{app:exp:forward:res}.






  \begin{table*}[t]
      \centering
      \renewcommand{\arraystretch}{1.1}
      \caption{Summary of the benchmark challenges. A ``\cmark (*)'' denotes that all problems in the category have the property. Otherwise, it is limited to the listed problems. The serial numbers correspond to the order of problems in Table 2.}
      \vspace{0.1in}
      \resizebox{\linewidth}{!}{
      \begin{tabular}{l|lllllll}
    \hline
    Problem                 & Time-Dependency & Nonlinearity & Complex Geometry & Multi-Scale & Discontinuity & High Frequency  \\ \hline
    Burgers$^{1\sim 2}$ & \cmark ($*$) & \cmark ($*$) & \xmark & \xmark  & \xmark & \cmark ($2$)      \\
    Poisson$^{3\sim 6}$ & \xmark & \xmark & \cmark ($3\sim 5$) & \cmark ($6$)  & \cmark ($5, 6$) & \xmark       \\
    Heat$^{7\sim10}$        & \cmark ($*$) & \cmark ($10$) & \cmark ($9$) & \cmark ($7,8,10$)  & \xmark & \cmark ($8$)       \\
    NS$^{11\sim13}$           & \cmark ($*$)  & \cmark ($*$) & \cmark ($12$) & \cmark ($13$)  & \xmark & \xmark        \\
    Wave$^{14\sim16}$         & \cmark ($*$) & \xmark & \xmark & \cmark ($16$)  & \xmark & \cmark ($15$)       \\
    Chaotic$^{17\sim18}$      & \cmark ($*$) & \cmark ($*$) & \xmark & \cmark ($*$)  & \xmark & \cmark ($*$)         \\
    \hline
    \end{tabular}}
      \label{tab:overview}
  \end{table*}

\begin{table*}[t]
    \centering
      \caption{Comparison of the average L2RE over 5 trials between our method and top PINN baselines. Best results are highlighted in {\setlength{\fboxsep}{3pt}\colorbox{tableblue}{\textbf{blue}}} and second-places in {\setlength{\fboxsep}{2pt}\colorbox{tablelightblue}{lightblue}}. ``NA'' denotes non-convergence or unsuitability for a given case. ``$\rddagger$'' signifies our method outperforming others by an order of magnitude or being the sole method to bring error under $100\%$ notably.}\label{tb:l2re}
      \vspace{0.1in}
      \begin{threeparttable}
      \resizebox{\linewidth}{!}{
\begin{tabular}{ll|cccccccccc}
\hline
\multicolumn{2}{c|}{}                               & \multicolumn{1}{c|}{}                                     & \multicolumn{2}{c|}{Vanilla}                   & \multicolumn{2}{c|}{Loss Reweighting}                                       & \multicolumn{1}{c|}{Optim}       & \multicolumn{1}{c|}{Loss Fn}  & \multicolumn{3}{c}{Architecture}                                                                                   \\ \cline{4-12} 
\multicolumn{2}{c|}{\multirow{-2}{*}{\textbf{L2RE} \resizebox{0.6em}{1.2em}{$\boldsymbol{\downarrow}$}}} & \multicolumn{1}{c|}{\multirow{-2}{*}{Ours}}               & PINN    & \multicolumn{1}{c|}{PINN-w}          & LRA                                  & \multicolumn{1}{c|}{NTK}             & \multicolumn{1}{c|}{MAdam}       & \multicolumn{1}{c|}{gPINN}           & LAAF                                 & GAAF                                 & \multicolumn{1}{l}{FBPINN}           \\ \hline
                                            & 1d-C  & \cellcolor{tableblue}\textbf{1.42e-2} & 1.45e-2 & 2.63e-2                              & 2.61e-2                              & 1.84e-2                              & 4.85e-2                              & 2.16e-1                              & \cellcolor{tablelightblue}{1.43e-2} & 5.20e-2                              & 2.32e-1                              \\
\multirow{-2}{*}{Burgers}                   & 2d-C  & 5.23e-1                           & 3.24e-1 & \cellcolor{tablelightblue}{2.70e-1} & \cellcolor{tableblue}{\textbf{2.60e-1}}    & 2.75e-1                              & 3.33e-1                              & 3.27e-1                              & 2.77e-1                              & 2.95e-1                              & NA                \\ \hline
                                            & $\text{2d-C}^\rddagger$  & 
                                            \cellcolor{tableblue}\textbf{3.98e-3} & 6.94e-1 & 3.49e-2                              & 1.17e-1                              & \cellcolor{tablelightblue}{1.23e-2} & 2.63e-2                              & 6.87e-1                              & 7.68e-1                              & 6.04e-1                              & 4.49e-2                              \\
                                            & $\text{2d-CG}^\rddagger$ & \cellcolor{tableblue}\textbf{5.07e-3} & 6.36e-1 & 6.08e-2                              & 4.34e-2                              & \cellcolor{tablelightblue}{1.43e-2} & 2.76e-1                              & 7.92e-1                              & 4.80e-1                              & 8.71e-1                              & 2.90e-2                              \\
                                            & $\text{3d-CG}^\rddagger$ & \cellcolor{tableblue}\textbf{4.16e-2} & 5.60e-1 & 3.74e-1                              & \cellcolor{tablelightblue}1.02e-1                              & 9.47e-1                              & 3.63e-1                              & 4.85e-1                              & 5.79e-1                              & 5.02e-1                              & 7.39e-1                              \\
\multirow{-4}{*}{Poisson}                   & $\text{2d-MS}^\rddagger$ & \cellcolor{tableblue}\textbf{6.40e-2} & 6.30e-1 & 7.60e-1                              & 7.94e-1                              & 7.48e-1                              & \cellcolor{tablelightblue}{5.90e-1} & 6.16e-1                              & 5.93e-1                              & 9.31e-1                              & 1.04e+0                              \\ \hline
                                            & $\text{2d-VC}^\rddagger$ & \cellcolor{tableblue}\textbf{3.11e-2} & 1.01e+0 & 2.35e-1                              & \cellcolor{tablelightblue}{2.12e-1} & 2.14e-1                              & 4.75e-1                              & 2.12e+0                              & 6.42e-1                              & 8.49e-1                              & 9.52e-1                              \\
                                            & 2d-MS & \cellcolor{tableblue}\textbf{2.84e-2} & 6.21e-2 & 2.42e-1                              & 8.79e-2                              & \cellcolor{tablelightblue}{4.40e-2} & 2.18e-1                              & 1.13e-1                              & 7.40e-2                              & 9.85e-1                              & 8.20e-2                              \\
                                            & 2d-CG & \cellcolor{tableblue}\textbf{1.50e-2} & 3.64e-2 & 1.45e-1                              & 1.25e-1                              & 1.16e-1                              & 7.12e-2                              & 9.38e-2                              & \cellcolor{tablelightblue}{2.39e-2} & 4.61e-1                              & 9.16e-2                              \\
\multirow{-4}{*}{Heat}                      & $\text{2d-LT}^\rddagger$ & \cellcolor{tableblue}\textbf{2.11e-1} & \cellcolor{tablelightblue}9.99e-1 & 9.99e-1                              & 9.99e-1                              & 1.00e+0                              & 1.00e+0                              & 1.00e+0                              & 9.99e-1                              & 9.99e-1                              & 1.01e+0                              \\ \hline
                                            & 2d-C  & \cellcolor{tableblue}\textbf{1.28e-2} & 4.70e-2 & 1.45e-1                              & NA                & 1.98e-1                              & 7.27e-1                              & 7.70e-2                              & \cellcolor{tablelightblue}{3.60e-2} & 3.79e-2                              & 8.45e-2                              \\
                                            & 2d-CG & \cellcolor{tableblue}\textbf{6.62e-2} & 1.19e-1 & 3.26e-1                              & 3.32e-1                              & 2.93e-1                              & 4.31e-1                              & 1.54e-1                              & \cellcolor{tablelightblue}{8.24e-2} & 1.74e-1                              & 8.27e+0                              \\
\multirow{-3}{*}{NS}                        & 2d-LT & \cellcolor{tableblue}\textbf{9.09e-1} & 9.96e-1 & 1.00e+0                              & 1.00e+0                              & 9.99e-1                              & 1.00e+0                              & \cellcolor{tablelightblue}{9.95e-1} & 9.98e-1                              & 9.99e-1                              & 1.00e+0                              \\ \hline
                                            & 1d-C  & \cellcolor{tableblue}\textbf{1.28e-2} & 5.88e-1 & 2.85e-1                              & 3.61e-1                              & \cellcolor{tablelightblue}{9.79e-2} & 1.21e-1                              & 5.56e-1                              & 4.54e-1                              & 6.77e-1                              & 5.91e-1                              \\
                                            & 2d-CG & \cellcolor{tableblue}\textbf{5.85e-1} & 1.84e+0 & 1.66e+0                              & 1.48e+0                              & 2.16e+0                              & 1.09e+0                              & 8.14e-1                              & 8.19e-1                              & \cellcolor{tablelightblue}{7.94e-1} & 1.06e+0                              \\
\multirow{-3}{*}{Wave}                      & $\text{2d-MS}^\rddagger$ & \cellcolor{tableblue}\textbf{5.71e-2} & 1.34e+0 & 1.02e+0                              & 1.02e+0                              & 1.04e+0                              & 1.01e+0 & 1.02e+0                              & 1.06e+0                              & 1.06e+0                              & 1.03e+0                              \\ \hline
                                            & GS    & \cellcolor{tableblue}\textbf{1.44e-2} & 3.19e-1 & 1.58e-1                              & 9.37e-2                              & 2.16e-1                              & 9.37e-2                              & 2.48e-1                              & 9.47e-2                              & 9.46e-2                              & \cellcolor{tablelightblue}{7.99e-2} \\
\multirow{-2}{*}{Chaotics}                  & KS    & \cellcolor{tableblue}\textbf{9.52e-1} & 1.01e+0 & 9.86e-1                              & \cellcolor{tablelightblue}{9.57e-1}                             & 9.64e-1                              & 9.61e-1                              & 9.94e-1                              & 1.01e+0                              & 1.00e+0                              & 1.02e+0                              \\ \hline
\end{tabular}}
    \begin{tablenotes}
      \small
      \item Abbreviations: ``Optim'' for optimizer, ``MAdam'' for MultiAdam, and ``Loss Fn'' for ``Loss Function''.
    \end{tablenotes}
    \end{threeparttable}
\end{table*}

\subsection{Relationship Between Condition Number and Error \& Convergence}\label{sec:exp:cond}
In this section, we empirically validate the theoretical findings in Section~\ref{sec:analyze}, especially the role of condition number in affecting the prediction accuracy and convergence of PINNs. Details of PDEs and implementation can be found in Appendix~\ref{app:exp:cond}. All experimental results are the average of 5 trials.

We begin by introducing two practical techniques to estimate the condition number when the ground-truth solution is provided:
\begin{enumerate}
    \item Training a neural network to find the suprema in Eq.~\eqref{eq:cond} with a small fixed $\epsilon$; 
    \item  Leveraging the finite difference method (FDM) to discretize the PDEs and subsequently approximating the condition number using the matrix norm as discussed in Eq.~\eqref{eq:original_cond}.
\end{enumerate}
To substantiate the reliability of these estimation techniques, we reconsider the 1D Poisson equation presented in Section~\ref{sec:model}. Since $\| u \|$ and $\| f \|$ can be computed straightforwardly, our focus pivots to approximating $\| \mathcal{F}^{-1} \|$. Figure~\ref{fig:theory_numerical_cond} captures our estimations across varied $P$ values, showcasing the close alignment with our theorem.

Transitioning to more intricate scenarios, we consider 3 practical problems: wave, Helmholtz, and Burgers' equation. System parameters within each problem are different: frequency $C$ in Wave, source term parameter $A$ in Helmholtz, and viscosity $\nu$ in Burgers. We vary the system parameter and monitor the subsequent influence on the condition number and error.

Figure~\ref{fig:cond_convergence} unveils that a \emph{strong}, but \emph{simple} linear correlation emerges between normalized condition numbers and their corresponding errors, suggesting that the condition number could be highly related to PINNs' performance. This relationship, however, varies across different equations depending on the specific normalization technique used. For instance, in the wave equation, $\log(\text{L2RE})$ exhibits a linear relationship with $\log(\mathrm{cond}(\mathcal{P}))$, while in Helmholtz, $\log(\text{L2RE})$ corresponds to $\sqrt{\mathrm{cond}(\mathcal{P})}$. A detailed interpretation of these patterns, through the lens of physics, is discussed in Appendix~\ref{app:exp:connection}. Lastly, Figure~\ref{fig:cond_errorhistory} underscores the condition number's profound impact on convergence dynamics, particularly evident in the wave equation, affirming the validity of our theoretical frameworks.


\begin{figure*}[tb]
     \centering
     \begin{subfigure}[b]{0.32\textwidth}
         \centering
         \includegraphics[height=0.72\textwidth]{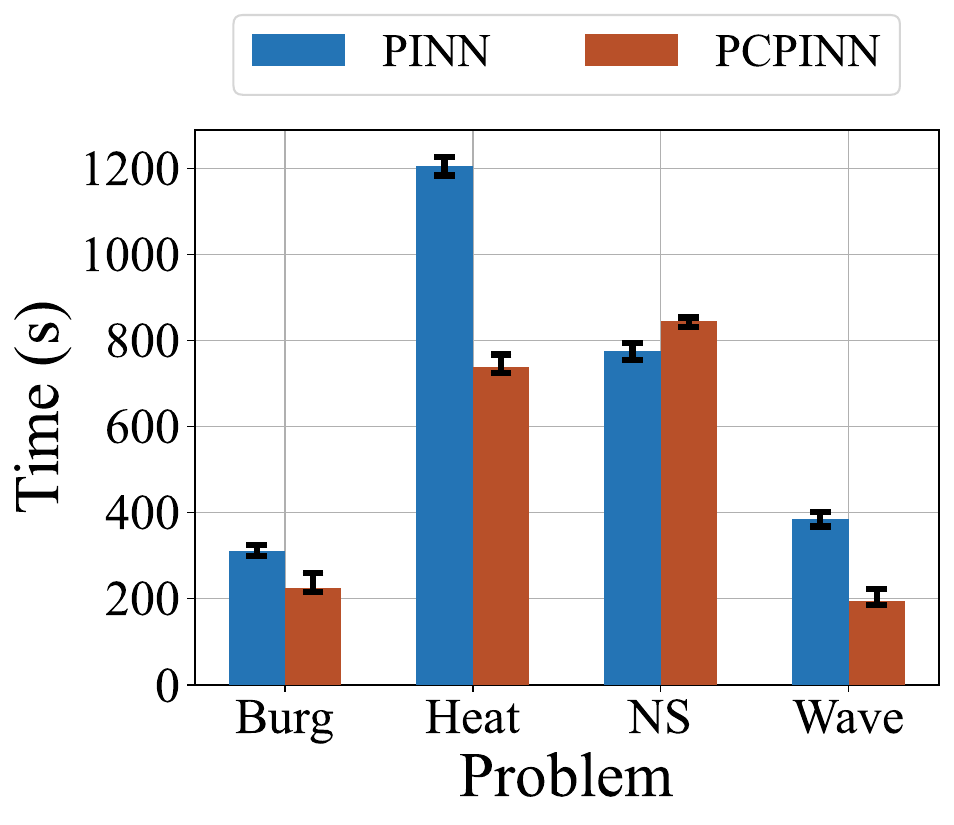}
         \caption{Time: PINN vs. Ours}
         \label{fig:time:1}
     \end{subfigure}
     \begin{subfigure}[b]{0.34\textwidth}
         \centering
         \includegraphics[height=0.68\textwidth]{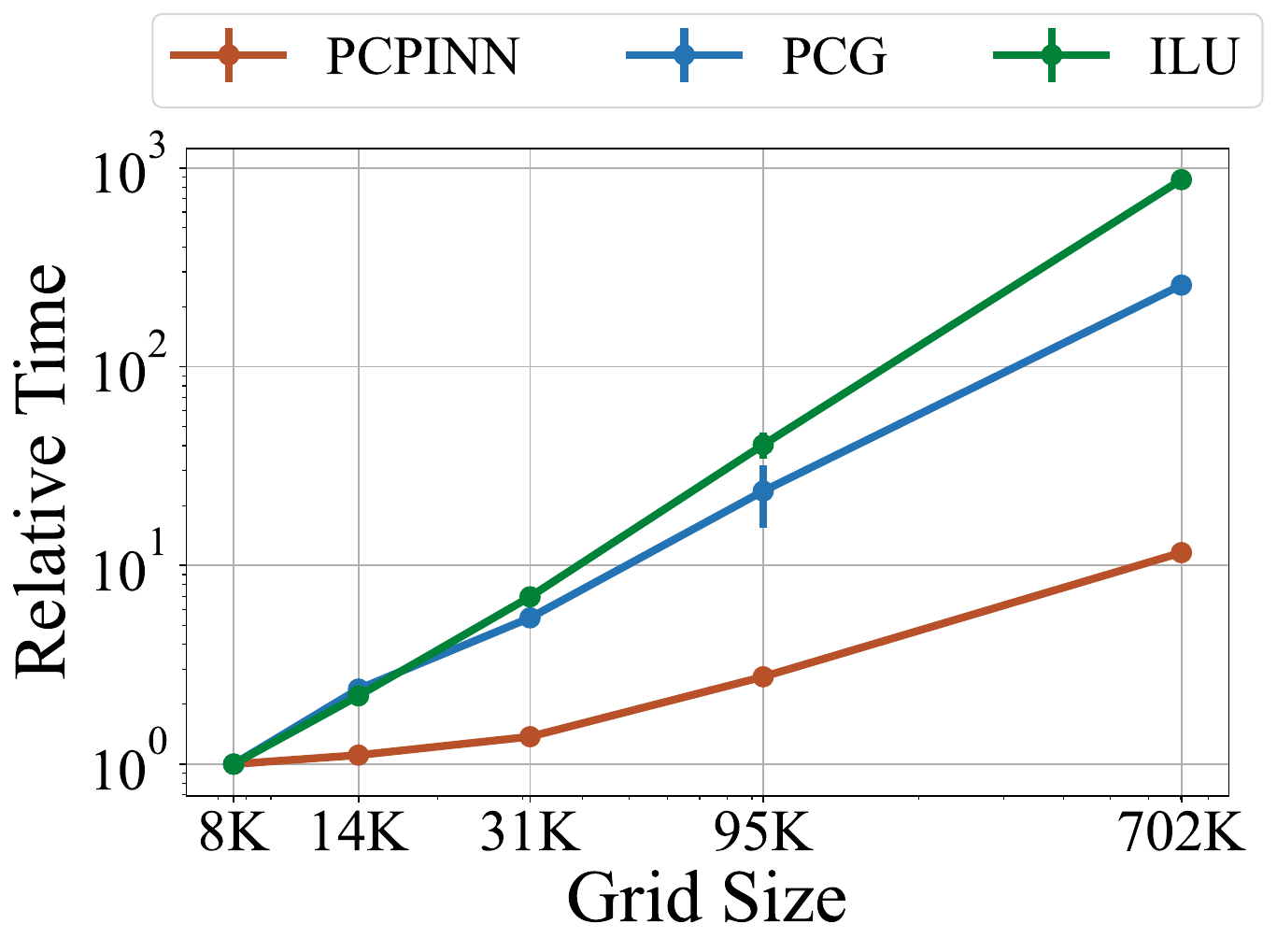}
         \caption{Scaling Law in Time: mean $\pm$ std}
         \label{fig:time:2}
     \end{subfigure}
     \begin{subfigure}[b]{0.31\textwidth}
         \centering
         \includegraphics[height=0.72\textwidth]{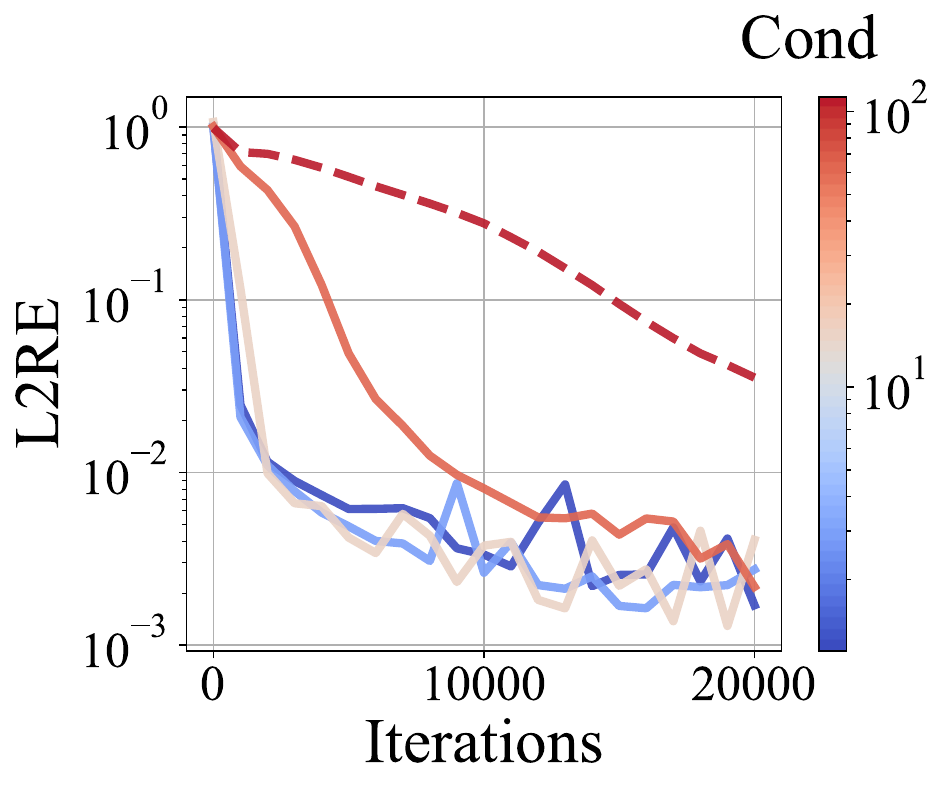}
         \caption{Effect of preconditioner precision}
         \label{fig:time:3}
     \end{subfigure}
\caption{\textbf{(a):} Computation time of PCPINN (ours) and vanilla PINN in selected problems, with error bars showing the $[\mathrm{min}, \mathrm{max}]$ in 5 trials. \textbf{(b):} Scaling law of computational time relative to an 8K grid size, contrasting our PCPINN with the preconditioned conjugate gradient method (PCG) and the preconditioning (ILU). \textbf{(c):} Convergence dynamics under varying preconditioner precision, with the dashed line for no preconditioner and the color bar for condition numbers: $\frac{\| \mP^{-1}\vb \|}{\| \vu \|} \| \mA^{-1}\mP \|$ under different preconditioner precisions.}
\label{fig:time}
\end{figure*}

\subsection{Benchmark of Forward Problems}\label{sec:exp:forward}
We consider the comprehensive PINN benchmark, PINNacle \citep{hao2023pinnacle}, encompassing 20 forward PDE problems and 10+ state-of-the-art PINN baselines. These problems, highlighted in Table~\ref{tab:overview}, include challenges from multi-scale properties to intricate geometries and diverse domains from fluids to chaos, underscoring the benchmark's difficulty and diversity. Further details on the benchmark can be found in \citep{hao2023pinnacle}.


\paragraph{Results and Performance.}
From the set of 20 problems, we have tested our method on 18, excluding 2 high-dimensional PDEs due to our method's mesh-based inherency. The experimental results are derived from 5 trials, with baseline results sourced directly from the PINNacle paper. In most cases, as detailed in Table~\ref{tb:l2re}, our method has achieved superior performance, showcasing a remarkable error drop (by an order of magnitude) for \textbf{7} problems. In \textbf{2} of these, ours uniquely achieved acceptable approximation, with competitors yielding errors at nearly $100\%$. Our success is attributed to the employed preconditioner, mitigating intrinsic pathologies and enhancing PINN performance. For the supplementary results and experimental details, including PDEs, baselines, and implementation specifics, please refer to Appendix~\ref{app:exp:forward:res} and Appendix~\ref{app:exp:forward}.



\paragraph{Convergence Analysis.}
Using the 1D wave equation for illustration, our method's convergence dynamics surpass those of traditional baselines. As depicted in Figure~\ref{fig:intro:1}, we achieve \emph{superexponential} convergence, while baselines show a slower, oscillating trajectory. Notably, their oscillations look smaller than real because of the logarithm-scale vertical axis. This clear difference is further emphasized in Figure~\ref{fig:intro:2}, where our method swiftly identifies the correct minimum, attributed to our preconditioner's ability to reshape the optimization landscape, facilitating rapid convergence with minimal oscillations.


\paragraph{Computation Time Analysis.}
We compare the computation time of our method to that of the vanilla PINN across diverse problems including Wave1d-C, Burgers1d-C, Heat2d-VC, and NS2d-C. As shown in Figure~\ref{fig:time:1}, our method is efficient, sometimes even outpacing the baseline. This efficiency is probably due to our rapid preconditioner calculation (basically less than 3s) and avoidance of time-intensive automatic derivation. Furthermore, we assessed the scalability of our method, the conjugate gradient method (used by the FEM solver), and the ILU for large-scale problems like Poisson3d-CG. While the neural network currently lags behind traditional methods in speed, its growth rate is remarkably slower by nearly two orders of magnitude. As Figure~\ref{fig:time:2} suggests, we anticipate superior scaling in even larger problems, thanks to the neural network's capacity to operate on low-dimensional manifolds, effectively mitigating the curse of dimensionality.



\paragraph{Effect of Preconditioner Precision.}
In our approach, a critical factor is the precision of the preconditioner (i.e., the deviation between $\mP$ and $\mA$), which is controlled by the drop tolerance in ILU. We have conducted ablation studies on this parameter across four Poisson equation problems. Figure~\ref{fig:time:3} depicts the convergence trajectories of our approach under condition numbers after preconditioning with varying precision in Poisson2d-C. The outcomes indicate a gradual performance decline of our method with decreasing precision of the preconditioner. Absent a preconditioner, our method reverts to a PINN with a discrete loss function, consequently failing to converge. This underscores the indispensable role of the preconditioner in enhancing the performance of PINNs. Comprehensive experimental details are available in Appendix~\ref{app:exp:forward:abla}.







\section{Conclusion and Limitation}
In this work, we have spotlighted the central role of the condition number in characterizing the training pathologies inherent to PINNs. By weaving together insights from traditional numerical analysis with modern machine learning techniques, we have theoretically demonstrated a direct correlation between a reduced condition number and improved PINNs' prediction accuracy and convergence. Our proposed algorithm, tested on a comprehensive benchmark, achieves significant improvements and overcomes challenges previously considered intractable. However, our preconditioning method relies on meshing, which is not feasible for high-dimensional problems. In future work, we will attempt to use neural networks to learn a preconditioner to overcome the curse of dimensionality.



\section*{Broder Impact}
This paper presents work whose goal is to advance the field of Physics-Informed Machine Learning. There are many potential societal consequences of our work, none of which we feel must be specifically highlighted here.

\bibliography{main}

\begin{thebibliography}{41}
\providecommand{\natexlab}[1]{#1}
\providecommand{\url}[1]{\texttt{#1}}
\expandafter\ifx\csname urlstyle\endcsname\relax
  \providecommand{\doi}[1]{doi: #1}\else
  \providecommand{\doi}{doi: \begingroup \urlstyle{rm}\Url}\fi

\bibitem[Aln{\ae}s et~al.(2015)Aln{\ae}s, Blechta, Hake, Johansson, Kehlet, Logg, Richardson, Ring, Rognes, and Wells]{alnaes2015fenics}
Aln{\ae}s, M., Blechta, J., Hake, J., Johansson, A., Kehlet, B., Logg, A., Richardson, C., Ring, J., Rognes, M.~E., and Wells, G.~N.
\newblock The fenics project version 1.5.
\newblock \emph{Archive of numerical software}, 3\penalty0 (100), 2015.

\bibitem[Beerens \& Higham(2023)Beerens and Higham]{beerens2023adversarial}
Beerens, L. and Higham, D.~J.
\newblock Adversarial ink: Componentwise backward error attacks on deep learning.
\newblock \emph{arXiv preprint arXiv:2306.02918}, 2023.

\bibitem[Berg \& Nystr{\"o}m(2018)Berg and Nystr{\"o}m]{berg2018unified}
Berg, J. and Nystr{\"o}m, K.
\newblock A unified deep artificial neural network approach to partial differential equations in complex geometries.
\newblock \emph{Neurocomputing}, 317:\penalty0 28--41, 2018.

\bibitem[Cai et~al.(2021)Cai, Mao, Wang, Yin, and Karniadakis]{cai2021physics}
Cai, S., Mao, Z., Wang, Z., Yin, M., and Karniadakis, G.~E.
\newblock Physics-informed neural networks (pinns) for fluid mechanics: A review.
\newblock \emph{Acta Mechanica Sinica}, 37\penalty0 (12):\penalty0 1727--1738, 2021.

\bibitem[Chen et~al.(2020)Chen, Lu, Karniadakis, and Dal~Negro]{chen2020physics}
Chen, Y., Lu, L., Karniadakis, G.~E., and Dal~Negro, L.
\newblock Physics-informed neural networks for inverse problems in nano-optics and metamaterials.
\newblock \emph{Optics express}, 28\penalty0 (8):\penalty0 11618--11633, 2020.

\bibitem[{COMSOL AB}(2022)]{comsol}
{COMSOL AB}.
\newblock Comsol multiphysics® v. 6.1, 2022.
\newblock URL \url{https://www.comsol.com}.

\bibitem[De~Ryck \& Mishra(2022)De~Ryck and Mishra]{de2022error}
De~Ryck, T. and Mishra, S.
\newblock Error analysis for physics-informed neural networks (pinns) approximating kolmogorov pdes.
\newblock \emph{Advances in Computational Mathematics}, 48\penalty0 (6):\penalty0 1--40, 2022.

\bibitem[De~Ryck et~al.(2022)De~Ryck, Jagtap, and Mishra]{de2022error2}
De~Ryck, T., Jagtap, A.~D., and Mishra, S.
\newblock Error estimates for physics informed neural networks approximating the navier-stokes equations.
\newblock \emph{arXiv preprint arXiv:2203.09346}, 2022.

\bibitem[Geuzaine \& Remacle(2009)Geuzaine and Remacle]{geuzaine2009gmsh}
Geuzaine, C. and Remacle, J.-F.
\newblock Gmsh: A 3-d finite element mesh generator with built-in pre-and post-processing facilities.
\newblock \emph{International journal for numerical methods in engineering}, 79\penalty0 (11):\penalty0 1309--1331, 2009.

\bibitem[Glorot \& Bengio(2010)Glorot and Bengio]{glorot2010understanding}
Glorot, X. and Bengio, Y.
\newblock Understanding the difficulty of training deep feedforward neural networks.
\newblock In \emph{Proceedings of the thirteenth international conference on artificial intelligence and statistics}, pp.\  249--256. JMLR Workshop and Conference Proceedings, 2010.

\bibitem[Guo \& Haghighat(2022)Guo and Haghighat]{guo2022energy}
Guo, M. and Haghighat, E.
\newblock Energy-based error bound of physics-informed neural network solutions in elasticity.
\newblock \emph{Journal of Engineering Mechanics}, 148\penalty0 (8):\penalty0 04022038, 2022.

\bibitem[Hao et~al.(2022)Hao, Liu, Zhang, Ying, Feng, Su, and Zhu]{hao2022physics}
Hao, Z., Liu, S., Zhang, Y., Ying, C., Feng, Y., Su, H., and Zhu, J.
\newblock Physics-informed machine learning: A survey on problems, methods and applications.
\newblock \emph{arXiv preprint arXiv:2211.08064}, 2022.

\bibitem[Hao et~al.(2023)Hao, Yao, Su, Su, Wang, Lu, Xia, Zhang, Liu, Lu, et~al.]{hao2023pinnacle}
Hao, Z., Yao, J., Su, C., Su, H., Wang, Z., Lu, F., Xia, Z., Zhang, Y., Liu, S., Lu, L., et~al.
\newblock Pinnacle: A comprehensive benchmark of physics-informed neural networks for solving pdes.
\newblock \emph{arXiv preprint arXiv:2306.08827}, 2023.

\bibitem[Hilditch(2013)]{hilditch2013introduction}
Hilditch, D.
\newblock An introduction to well-posedness and free-evolution.
\newblock \emph{International Journal of Modern Physics A}, 28\penalty0 (22n23):\penalty0 1340015, 2013.

\bibitem[Huang \& Wang(2022)Huang and Wang]{huang2022applications}
Huang, B. and Wang, J.
\newblock Applications of physics-informed neural networks in power systems-a review.
\newblock \emph{IEEE Transactions on Power Systems}, 2022.

\bibitem[Jacot et~al.(2018)Jacot, Gabriel, and Hongler]{jacot2018ntk}
Jacot, A., Gabriel, F., and Hongler, C.
\newblock Neural tangent kernel: Convergence and generalization in neural networks.
\newblock \emph{Advances in neural information processing systems}, 31, 2018.

\bibitem[Jagtap et~al.(2022)Jagtap, Mao, Adams, and Karniadakis]{jagtap2022physics}
Jagtap, A.~D., Mao, Z., Adams, N., and Karniadakis, G.~E.
\newblock Physics-informed neural networks for inverse problems in supersonic flows.
\newblock \emph{Journal of Computational Physics}, 466:\penalty0 111402, 2022.

\bibitem[Kingma \& Ba(2014)Kingma and Ba]{kingma2014adam}
Kingma, D.~P. and Ba, J.
\newblock Adam: A method for stochastic optimization.
\newblock \emph{arXiv preprint arXiv:1412.6980}, 2014.

\bibitem[Krishnapriyan et~al.(2021)Krishnapriyan, Gholami, Zhe, Kirby, and Mahoney]{krishnapriyan2021characterizing}
Krishnapriyan, A., Gholami, A., Zhe, S., Kirby, R., and Mahoney, M.~W.
\newblock Characterizing possible failure modes in physics-informed neural networks.
\newblock \emph{Advances in Neural Information Processing Systems}, 34:\penalty0 26548--26560, 2021.

\bibitem[Liu et~al.(2022)Liu, Zhongkai, Ying, Su, Zhu, and Cheng]{liu2022unified}
Liu, S., Zhongkai, H., Ying, C., Su, H., Zhu, J., and Cheng, Z.
\newblock A unified hard-constraint framework for solving geometrically complex pdes.
\newblock \emph{Advances in Neural Information Processing Systems}, 35:\penalty0 20287--20299, 2022.

\bibitem[Liu \& Wang(2021)Liu and Wang]{liu2021physics}
Liu, X.-Y. and Wang, J.-X.
\newblock Physics-informed dyna-style model-based deep reinforcement learning for dynamic control.
\newblock \emph{Proceedings of the Royal Society A}, 477\penalty0 (2255):\penalty0 20210618, 2021.

\bibitem[Lu et~al.(2021{\natexlab{a}})Lu, Meng, Mao, and Karniadakis]{lu2021deepxde}
Lu, L., Meng, X., Mao, Z., and Karniadakis, G.~E.
\newblock Deepxde: A deep learning library for solving differential equations.
\newblock \emph{SIAM review}, 63\penalty0 (1):\penalty0 208--228, 2021{\natexlab{a}}.

\bibitem[Lu et~al.(2021{\natexlab{b}})Lu, Pestourie, Yao, Wang, Verdugo, and Johnson]{lu2021physics}
Lu, L., Pestourie, R., Yao, W., Wang, Z., Verdugo, F., and Johnson, S.~G.
\newblock Physics-informed neural networks with hard constraints for inverse design.
\newblock \emph{SIAM Journal on Scientific Computing}, 43\penalty0 (6):\penalty0 B1105--B1132, 2021{\natexlab{b}}.

\bibitem[Martin \& Schaub(2022)Martin and Schaub]{martin2022reinforcement}
Martin, J. and Schaub, H.
\newblock Reinforcement learning and orbit-discovery enhanced by small-body physics-informed neural network gravity models.
\newblock In \emph{AIAA SCITECH 2022 Forum}, pp.\  2272, 2022.

\bibitem[Mishra \& Molinaro(2022)Mishra and Molinaro]{mishra2022estimates}
Mishra, S. and Molinaro, R.
\newblock Estimates on the generalization error of physics-informed neural networks for approximating a class of inverse problems for pdes.
\newblock \emph{IMA Journal of Numerical Analysis}, 42\penalty0 (2):\penalty0 981--1022, 2022.

\bibitem[Pang et~al.(2019)Pang, Lu, and Karniadakis]{pang2019fpinns}
Pang, G., Lu, L., and Karniadakis, G.~E.
\newblock fpinns: Fractional physics-informed neural networks.
\newblock \emph{SIAM Journal on Scientific Computing}, 41\penalty0 (4):\penalty0 A2603--A2626, 2019.

\bibitem[Paszke et~al.(2019)Paszke, Gross, Massa, Lerer, Bradbury, Chanan, Killeen, Lin, Gimelshein, Antiga, et~al.]{paszke2019pytorch}
Paszke, A., Gross, S., Massa, F., Lerer, A., Bradbury, J., Chanan, G., Killeen, T., Lin, Z., Gimelshein, N., Antiga, L., et~al.
\newblock Pytorch: An imperative style, high-performance deep learning library.
\newblock \emph{Advances in neural information processing systems}, 32, 2019.

\bibitem[Rahaman et~al.(2019)Rahaman, Baratin, Arpit, Draxler, Lin, Hamprecht, Bengio, and Courville]{pmlr-v97-rahaman19a}
Rahaman, N., Baratin, A., Arpit, D., Draxler, F., Lin, M., Hamprecht, F., Bengio, Y., and Courville, A.
\newblock On the spectral bias of neural networks.
\newblock In Chaudhuri, K. and Salakhutdinov, R. (eds.), \emph{Proceedings of the 36th International Conference on Machine Learning}, volume~97 of \emph{Proceedings of Machine Learning Research}, pp.\  5301--5310. PMLR, 09--15 Jun 2019.
\newblock URL \url{https://proceedings.mlr.press/v97/rahaman19a.html}.

\bibitem[Raissi et~al.(2019)Raissi, Perdikaris, and Karniadakis]{raissi2019physics}
Raissi, M., Perdikaris, P., and Karniadakis, G.~E.
\newblock Physics-informed neural networks: A deep learning framework for solving forward and inverse problems involving nonlinear partial differential equations.
\newblock \emph{Journal of Computational physics}, 378:\penalty0 686--707, 2019.

\bibitem[Shabat et~al.(2018)Shabat, Shmueli, Aizenbud, and Averbuch]{shabat2018randomized}
Shabat, G., Shmueli, Y., Aizenbud, Y., and Averbuch, A.
\newblock Randomized lu decomposition.
\newblock \emph{Applied and Computational Harmonic Analysis}, 44\penalty0 (2):\penalty0 246--272, 2018.

\bibitem[Sheng \& Yang(2021)Sheng and Yang]{sheng2021pfnn}
Sheng, H. and Yang, C.
\newblock Pfnn: A penalty-free neural network method for solving a class of second-order boundary-value problems on complex geometries.
\newblock \emph{Journal of Computational Physics}, 428:\penalty0 110085, 2021.

\bibitem[Sheng \& Yang(2022)Sheng and Yang]{sheng2022pfnn}
Sheng, H. and Yang, C.
\newblock Pfnn-2: A domain decomposed penalty-free neural network method for solving partial differential equations.
\newblock \emph{arXiv preprint arXiv:2205.00593}, 2022.

\bibitem[S{\"u}li \& Mayers(2003)S{\"u}li and Mayers]{suli2003introduction}
S{\"u}li, E. and Mayers, D.~F.
\newblock \emph{An introduction to numerical analysis}.
\newblock Cambridge university press, 2003.

\bibitem[Tancik et~al.(2020)Tancik, Srinivasan, Mildenhall, Fridovich-Keil, Raghavan, Singhal, Ramamoorthi, Barron, and Ng]{tancik2020fourier}
Tancik, M., Srinivasan, P., Mildenhall, B., Fridovich-Keil, S., Raghavan, N., Singhal, U., Ramamoorthi, R., Barron, J., and Ng, R.
\newblock Fourier features let networks learn high frequency functions in low dimensional domains.
\newblock \emph{Advances in Neural Information Processing Systems}, 33:\penalty0 7537--7547, 2020.

\bibitem[Wang et~al.(2021)Wang, Teng, and Perdikaris]{wang2021understanding}
Wang, S., Teng, Y., and Perdikaris, P.
\newblock Understanding and mitigating gradient flow pathologies in physics-informed neural networks.
\newblock \emph{SIAM Journal on Scientific Computing}, 43\penalty0 (5):\penalty0 A3055--A3081, 2021.

\bibitem[Wang et~al.(2022{\natexlab{a}})Wang, Sankaran, and Perdikaris]{wang2022respecting}
Wang, S., Sankaran, S., and Perdikaris, P.
\newblock Respecting causality is all you need for training physics-informed neural networks.
\newblock \emph{arXiv preprint arXiv:2203.07404}, 2022{\natexlab{a}}.

\bibitem[Wang et~al.(2022{\natexlab{b}})Wang, Yu, and Perdikaris]{wang2022and}
Wang, S., Yu, X., and Perdikaris, P.
\newblock When and why pinns fail to train: A neural tangent kernel perspective.
\newblock \emph{Journal of Computational Physics}, 449:\penalty0 110768, 2022{\natexlab{b}}.

\bibitem[Wang et~al.(2022{\natexlab{c}})Wang, Yu, and Perdikaris]{wang2022ntk}
Wang, S., Yu, X., and Perdikaris, P.
\newblock When and why pinns fail to train: A neural tangent kernel perspective.
\newblock \emph{Journal of Computational Physics}, 449:\penalty0 110768, 2022{\natexlab{c}}.

\bibitem[Xu et~al.(2019)Xu, Zhang, Luo, Xiao, and Ma]{xu2019frequency}
Xu, Z.-Q.~J., Zhang, Y., Luo, T., Xiao, Y., and Ma, Z.
\newblock Frequency principle: Fourier analysis sheds light on deep neural networks.
\newblock \emph{arXiv preprint arXiv:1901.06523}, 2019.

\bibitem[Yang et~al.(2021)Yang, Meng, and Karniadakis]{yang2021b}
Yang, L., Meng, X., and Karniadakis, G.~E.
\newblock B-pinns: Bayesian physics-informed neural networks for forward and inverse pde problems with noisy data.
\newblock \emph{Journal of Computational Physics}, 425:\penalty0 109913, 2021.

\bibitem[Zhu et~al.(2021)Zhu, Liu, and Yan]{zhu2021machine}
Zhu, Q., Liu, Z., and Yan, J.
\newblock Machine learning for metal additive manufacturing: predicting temperature and melt pool fluid dynamics using physics-informed neural networks.
\newblock \emph{Computational Mechanics}, 67:\penalty0 619--635, 2021.

\end{thebibliography}
\bibliographystyle{icml2024}

\newpage
\appendix
\onecolumn

\section{Supplements for Section~\ref{sec:analyze}}
The following are general assumptions across our theories:

\begin{assumption}\label{ass:1}
The problem domain $\Omega$ is an open, bounded, and nonempty subset of $\mathbb{R}^d$, where $d\in \mathbb{N}^+$ is the spatial(-temporal) dimensionality. And 
\end{assumption}

\begin{assumption}\label{ass:2}
The boundary value problem (BVP) considered in Eq.~\eqref{eq:bvp} is well-posed, which means the solution exists and is unique, and $\mathcal{F}^{-1}$ is well-defined.
\end{assumption}


\begin{assumption}\label{ass:4}
$\| u \| \neq 0$ and $\| f \| \neq 0$.
\end{assumption}

\begin{remark}
This assumption assures that the \emph{relative} conditional number is well-defined. If it is not satisfied, we could define the \emph{absolute} conditional number by removing the zero terms. 
\end{remark}



\begin{assumption}\label{ass:7}
For any continuous function $v$ defined on $\Omega$ (i.e., $v\in C(\Omega)$), it holds that $\inf_{\vth\in\Theta} \| u_{\vth} - v \| = 0$.
\end{assumption}

\begin{remark}
We assume that the neural network has sufficient approximation capability and ignore the corresponding error.
\end{remark}

\subsection{Proof for Theorem~\ref{theo:cond:bound}}\label{app:cond:bound}



Under Assumption~\ref{ass:1} -- \ref{ass:7}, the proof of Theorem~\ref{theo:cond:bound} is given as follows.

\begin{proof}
According to the local Lipschitz continuity of $\mathcal{F}^{-1}$, there exists $r>0$ such that:
\begin{equation}
    \left\| \mathcal{F}^{-1}[w_1] - \mathcal{F}^{-1}[w_2] \right\| \le K \| w_1 - w_2 \|,
\end{equation}
holds for any $w_1, w_2 \in W$ which satisfy that $\|w_1 - f\| < r$ and $\|w_2 - f\| < r$.

Taking an $\epsilon < r$, we can derive that:
\begin{equation}
\begin{aligned}
    &\sup_{0<\| \delta f \| \le \epsilon} \frac{\| \delta u \| \big/ \| u \|}{\| \delta f \| \big/ \| f \|}\\
    &=\frac{\| f \|}{\| u \|}\sup_{0<\| \mathcal{F}[u_{\vth}] - f \| \le \epsilon} \frac{\| u_{\vth} - u \|}{\| \mathcal{F}[u_{\vth}] - f\|}\\
    &= \frac{\| f \|}{\| u \|}\sup_{0<\| h \| \le \epsilon} \frac{\| \mathcal{F}^{-1}[f+h] - \mathcal{F}^{-1}[f] \|}{\| h\|} &&\text{(let $\mathcal{F}[u_{\vth}] - f = h$)}\\
    &\le \frac{\| f \|}{\| u \|}\sup_{0<\| h \| \le \epsilon} \frac{K\left\| h \right\| }{\| h \|} \\
    &= \frac{\| f \|}{\| u \|}K.
\end{aligned}
\end{equation}

Finally, let $\epsilon \rightarrow 0^+$, we can prove the theorem:
\begin{equation}
    \mathrm{cond}(\mathcal{P}) = \lim_{\epsilon\to 0^+} \sup_{0<\| \delta f \| \le \epsilon} \frac{\| \delta u \| \big/ \| u \|}{\| \delta f \| \big/ \| f \|} \le
    \frac{\| f \|}{\| u \|} K.
\end{equation}
    
\end{proof}

\subsection{The Existence of Condition Number in Special Cases}\label{app:prop}

\begin{proposition}\label{theo:prop}
Considering a well-posed $\mathcal{P}: \{ \mathcal{F}[u] = f \text{ in } \Omega,  u = g \text{ in } \partial\Omega\}$, we assert that:
\begin{enumerate}
    \item If $\mathcal{F}$ is linear (i.e., a linear PDE) and $g = 0$ (homogeneous BC), then $\mathcal{F}^{-1}$ is a bounded linear operator and $\mathrm{cond}(\mathcal{P}) = \frac{\| f \|}{\| u \|} \| \mathcal{F}^{-1} \| < \infty$.
    \item Define $\mathcal{P}_1: \{ \mathcal{F}[u] = 0 \text{ in } \Omega,  u = g \text{ in } \partial\Omega\}$. If $\mathcal{F}$ is linear and $\mathcal{P}_1$ is well-posed, then $\mathrm{cond}(\mathcal{P}) < \infty$.
    \item If $\mathcal{F}^{-1}$ is Fréchet differentiable at $f$, then $\mathrm{cond}(\mathcal{P}) = \frac{\| f \|}{\| u \|} \| D\mathcal{F}^{-1}[f] \| < \infty$, where $D\mathcal{F}^{-1}[f]\colon W\rightarrow V$ is a bounded linear operator, the Fréchet derivative of $\mathcal{F}^{-1}$ at $f$.
\end{enumerate}
\end{proposition}

We divide the Proposition~\ref{theo:prop} into the following theorems and prove them one by one.


\begin{theorem}
If $\mathcal{F}$ is linear and $g=0$, then $\mathcal{F}^{-1}$ is a bounded linear operator and:
\begin{equation}
    \mathrm{cond}(\mathcal{P}) = \frac{\| f \|}{\| u \|} \left\|\mathcal{F}^{-1}\right\| < \infty.
\end{equation}
\end{theorem}

\begin{proof}
    Firstly, it is easy to show the linearity. Considering $k_1, k_2\in \mathbb{K}, w_1, w_2 \in S$, there exists $u_1, u_2 \in V$ such that $\mathcal{F}[u_1] = w_1 \land 
 u_1|_{\partial\Omega} = 0$ and $\mathcal{F}[u_2] = w_2 \land 
 u_2|_{\partial\Omega} = 0$. Then, we have:
    \begin{equation}
        \mathcal{F}^{-1}[k_1 w_1 + k_2 w_2] = k_1 u_1 + k_2 u_2 = k_1 \mathcal{F}^{-1}[w_1] + k_2 \mathcal{F}^{-1}[w_2],
    \end{equation}
    where the first equation holds because $\mathcal{F}[k_1 u_1 + k_2 u_2] = k_1\mathcal{F}[u_1] + k_2\mathcal{F}[u_2] = k_1 w_1 + k_2 w_2$ and $k_1 u_1 + k_2 u_2 = 0 \text{ in } \partial\Omega$.

    Secondly, according to the well-posedness, $\mathcal{F}^{-1}$ is continuous and thus bounded.

    Finally, we have:
    \begin{equation}
\begin{aligned}
    &\sup_{0<\| \delta f \| \le \epsilon} \frac{\| \delta u \| \big/ \| u \|}{\| \delta f \| \big/ \| f \|}\\
    &=\frac{\| f \|}{\| u \|}\sup_{0<\| \mathcal{F}[u_{\vth}] - f \| \le \epsilon} \frac{\| u_{\vth} - u \|}{\| \mathcal{F}[u_{\vth}] - f\|}\\
    &= \frac{\| f \|}{\| u \|}\sup_{0<\| h \| \le \epsilon} \frac{\| \mathcal{F}^{-1}[f+h] - \mathcal{F}^{-1}[f] \|}{\| h\|} &&\text{(let $\mathcal{F}[u_{\vth}] - f = h$)}\\
    &= \frac{\| f \|}{\| u \|}\sup_{0<\| h \| \le \epsilon} \frac{\left\| \mathcal{F}^{-1}[h] \right\| }{\| h \|} \\
    &= \frac{\| f \|}{\| u \|}\left\| \mathcal{F}^{-1} \right\|.
\end{aligned}
\end{equation}
    Therefore, let $\epsilon \rightarrow 0^+$, $\mathrm{cond}(\mathcal{P}) = \frac{\| f \|}{\| u \|} \left\|\mathcal{F}^{-1}\right\| < \infty$.
    
\end{proof}

\begin{theorem}
Define $\mathcal{P}_1: \{ \mathcal{F}[u] = 0 \text{ in } \Omega,  u = g \text{ in } \partial\Omega\}$. If $\mathcal{F}$ is linear and $\mathcal{P}_1$ is well-posed, then:
\begin{equation}
    \mathrm{cond}(\mathcal{P}) < \infty.
\end{equation}
\end{theorem}
\begin{proof}
Since $\mathcal{P}_1$ is well-posed, there exists a unique solution $u_1 \in V$ to it. We define $\mathcal{G}: S \rightarrow V$ as $\mathcal{G}[w] = \mathcal{F}^{-1}[w] - u_1$. Then we show that $\mathcal{G}$ is linear. Consider $k_1, k_2\in \mathbb{K}, w_1, w_2 \in S$,
\begin{equation}
\begin{aligned}
    \mathcal{G}[k_1 w_1 + k_2 w_2] &= \mathcal{F}^{-1}[k_1 w_1 + k_2 w_2] - u_1,\\
    k_1\mathcal{G}[w_1] + k_2\mathcal{G}[w_2] &= k_1 \left( \mathcal{F}^{-1}[w_1] - u_1\right) + k_2 \left( \mathcal{F}^{-1}[w_2] - u_1\right).
\end{aligned}
\end{equation}
We have to show that:
\begin{equation}\label{eq:to:prove:1}
    \begin{aligned}
    &\mathcal{F}^{-1}[k_1 w_1 + k_2 w_2] - u_1 &&= k_1 \left( \mathcal{F}^{-1}[w_1] - u_1\right) + k_2 \left( \mathcal{F}^{-1}[w_2] - u_1\right)\\
    \Longleftrightarrow\quad &\mathcal{F}^{-1}[k_1 w_1 + k_2 w_2] &&= k_1 \left( \mathcal{F}^{-1}[w_1] - u_1\right) + k_2 \left( \mathcal{F}^{-1}[w_2] - u_1\right) + u_1.
\end{aligned}
\end{equation}
Apply $\mathcal{F}$ on both sides:
\begin{equation}
\begin{aligned}
    k_1 w_1 + k_2 w_2 &= \mathcal{F}\left(\mathcal{F}^{-1}[k_1 w_1 + k_2 w_2]\right) \\
    &= \mathcal{F} \left(k_1 \left( \mathcal{F}^{-1}[w_1] - u_1\right) + k_2 \left( \mathcal{F}^{-1}[w_2] - u_1\right) + u_1 \right)\\
    &= k_1 w_1 + k_2 w_2.
\end{aligned}
\end{equation}
And consider the value on the boundary:
\begin{equation}
\begin{aligned}
    g &= \left(\mathcal{F}^{-1}[k_1 w_1 + k_2 w_2]\right)\Big|_{\partial\Omega} \\
    &= \left(k_1 \left( \mathcal{F}^{-1}[w_1] - u_1\right) + k_2 \left( \mathcal{F}^{-1}[w_2] - u_1\right) + u_1 \right)\Big|_{\partial\Omega}\\
    &= k_1 (g-g)+ k_2 (g-g) + g = g.
\end{aligned}
\end{equation}
Then, according to the well-defineness of $\mathcal{F}^{-1}$, we can prove that Eq.~\eqref{eq:to:prove:1} holds and thus $\mathcal{G}$ is linear. Besides, since $\mathcal{F}^{-1}$ is continuous, $\mathcal{G}$ is a bounded linear operator.

Finally, we have:
    \begin{equation}
\begin{aligned}
    &\sup_{0<\| \delta f \| \le \epsilon} \frac{\| \delta u \| \big/ \| u \|}{\| \delta f \| \big/ \| f \|}\\
    &=\frac{\| f \|}{\| u \|}\sup_{0<\| \mathcal{F}[u_{\vth}] - f \| \le \epsilon} \frac{\| u_{\vth} - u \|}{\| \mathcal{F}[u_{\vth}] - f\|}\\
    &= \frac{\| f \|}{\| u \|}\sup_{0<\| h \| \le \epsilon} \frac{\| \mathcal{F}^{-1}[f+h] - \mathcal{F}^{-1}[f] \|}{\| h\|} &&\text{(let $\mathcal{F}[u_{\vth}] - f = h$)}\\
    &= \frac{\| f \|}{\| u \|}\sup_{0<\| h \| \le \epsilon} \frac{\left\| \mathcal{G}[f+h] - \mathcal{G}[f] \right\| }{\| h \|} \\
    &= \frac{\| f \|}{\| u \|}\sup_{0<\| h \| \le \epsilon} \frac{\left\| \mathcal{G}[h] \right\| }{\| h \|} \\
    &= \frac{\| f \|}{\| u \|}\left\| \mathcal{G} \right\|.
\end{aligned}
\end{equation}
    Therefore, let $\epsilon \rightarrow 0^+$, $\mathrm{cond}(\mathcal{P}) = \frac{\| f \|}{\| u \|} \left\|\mathcal{G}\right\| < \infty$.





\end{proof}

\begin{theorem}\label{theo:cond:derivative}
If $\mathcal{F}^{-1}$ is Fréchet differentiable at $f$, we have that:
\begin{equation}
    \mathrm{cond}(\mathcal{P}) = \frac{\| f \|}{\| u \|} \left\| D\mathcal{F}^{-1}[f] \right\| < \infty,
\end{equation}
where $D\mathcal{F}^{-1}[f]\colon S\rightarrow V$ is a bounded linear operator, the Fréchet derivative of $\mathcal{F}^{-1}$ at $f$.
\end{theorem}

\begin{proof}
Since $\mathcal{F}^{-1}$ is Fréchet differentiable at $f$, it is true that:
\begin{equation}
\begin{aligned}
    &\lim_{\epsilon\to 0^+} \sup_{0<\| h \| \le \epsilon} \frac{\left\| \mathcal{F}^{-1}[f+h] - \mathcal{F}^{-1}[f] - D\mathcal{F}^{-1}[f][h] \right\| }{\| h \|}\\
    &= \lim_{\| h \|\to 0^+} \frac{\left\| \mathcal{F}^{-1}[f+h] - \mathcal{F}^{-1}[f] - D\mathcal{F}^{-1}[f][h] \right\| }{\| h \|} = 0.
\end{aligned}
\end{equation}

We can find that $W \neq \{ 0 \}$ since $u\in V$, $\mathcal{F}[u] = f \in W$, and $\| f \| \neq 0$. Therefore, we have that:
\begin{equation}
\begin{aligned}
    &\lim_{\epsilon\to 0^+} \sup_{0<\| h \| \le \epsilon} \frac{\left\| D\mathcal{F}^{-1}[f][h] \right\| }{\| h \|}\\
    &= \lim_{\epsilon\to 0^+} \sup_{0<\| h \| \le \epsilon} \left\| D\mathcal{F}^{-1}[f]\left[\frac{h}{\| h \|} \right] \right\|  = \left\| D\mathcal{F}^{-1}[f] \right\|,
\end{aligned}
\end{equation}
which holds due to the fact that $D\mathcal{F}^{-1}[f]$ is a bounded linear operator.

Then, we have that:
\begin{equation}
\begin{aligned}
    &\sup_{0<\| \delta f \| \le \epsilon} \frac{\| \delta u \| \big/ \| u \|}{\| \delta f \| \big/ \| f \|}\\
    &=\frac{\| f \|}{\| u \|}\sup_{0<\| \mathcal{F}[u_{\vth}] - f \| \le \epsilon} \frac{\| u_{\vth} - u \|}{\| \mathcal{F}[u_{\vth}] - f\|}\\
    &= \frac{\| f \|}{\| u \|}\sup_{0<\| h \| \le \epsilon} \frac{\| \mathcal{F}^{-1}[f+h] - \mathcal{F}^{-1}[f] \|}{\| h\|} &&\text{(let $\mathcal{F}[u_{\vth}] - f = h$)}\\
    &\le \frac{\| f \|}{\| u \|}\sup_{0<\| h \| \le \epsilon} \frac{\left\| \mathcal{F}^{-1}[f+h] - \mathcal{F}^{-1}[f] - D\mathcal{F}^{-1}[f][h] \right\| }{\| h \|} \\
    &\quad + \frac{\| f \|}{\| u \|} \sup_{0<\| h \| \le \epsilon} \frac{\left\| D\mathcal{F}^{-1}[f][h] \right\| }{\| h \|} \to 0 + \frac{\| f \|}{\| u \|} \left\| D\mathcal{F}^{-1}[f] \right\|,
\end{aligned}
\end{equation}
when $\epsilon\to 0^+$.

As for the left-hand side, it follows that:
\begin{equation}
\begin{aligned}
    &\frac{\| f \|}{\| u \|}\sup_{0<\| h \| \le \epsilon} \frac{\| \mathcal{F}^{-1}[f+h] - \mathcal{F}^{-1}[f] \|}{\| h\|}\\
    &\ge \frac{\| f \|}{\| u \|}\sup_{0<\| h \| \le \epsilon} \bigg( 
    \frac{\left\| D\mathcal{F}^{-1}[f][h] \right\| }{\| h \|}\\
    &\quad -\frac{\left\| \mathcal{F}^{-1}[f+h] - \mathcal{F}^{-1}[f] - D\mathcal{F}^{-1}[f][h] \right\| }{\| h \|} \bigg)\\
    &\ge \frac{\| f \|}{\| u \|}\sup_{0<\| h \| \le \epsilon} \bigg( 
    \frac{\left\| D\mathcal{F}^{-1}[f][h] \right\| }{\| h \|}\\
    &\quad -\sup_{0<\| h \| \le \epsilon}\frac{\left\| \mathcal{F}^{-1}[f+h] - \mathcal{F}^{-1}[f] - D\mathcal{F}^{-1}[f][h] \right\| }{\| h \|} \bigg)\\
    &= \frac{\| f \|}{\| u \|}\sup_{0<\| h \| \le \epsilon}  
    \frac{\left\| D\mathcal{F}^{-1}[f][h] \right\| }{\| h \|}\\
    &\quad -\frac{\| f \|}{\| u \|}\sup_{0<\| h \| \le \epsilon}\frac{\left\| \mathcal{F}^{-1}[f+h] - \mathcal{F}^{-1}[f] - D\mathcal{F}^{-1}[f][h] \right\| }{\| h \|}\\
    &\to \frac{\| f \|}{\| u \|} \left\| D\mathcal{F}^{-1}[f] \right\| - 0,
\end{aligned}
\end{equation}
when $\epsilon\to 0^+$.

According to the squeeze theorem, we have proven the theorem:
\begin{equation}
    \mathrm{cond}(\mathcal{P}) = \lim_{\epsilon\to 0^+} \sup_{0<\| \delta f \| \le \epsilon} \frac{\| \delta u \| \big/ \| u \|}{\| \delta f \| \big/ \| f \|} =
    \frac{\| f \|}{\| u \|} \left\| D\mathcal{F}^{-1}[f] \right\| < \infty.
\end{equation}


\end{proof}

\subsection{Proof for Theorem~\ref{theo:poisson_condnumber}}\label{app:poisson_condnumber}
Firstly, we define the inner product in $L^2((0,2\pi/P))$ as:
\begin{equation}
    \langle f, g\rangle = \frac P{2\pi} \int _0^{2\pi} f(x) g(x) \text dx.
\end{equation}
With the inner product defined above, $L^2((0,2\pi/P))$ forms a Hilbert space. As $f \in L^2$, we can have a Fourier series representation of $f$:
\begin{equation}
    f = 2c + \sum_{k\ge 1} a_k \sin (kPx) + \sum_{k\ge 1} b_k \cos(kPx).
\end{equation}
It is then easy to obtain $u = \mathcal F^{-1}[f]$ from the series:
\begin{equation}
    u = cx(x-2\pi/P) - \sum_{k\ge 1} \frac{a_k}{k^2P^2}\sin (kPx) - \sum_{k\ge 1} \frac{b_k}{k^2P^2}(\cos(kPx) - 1).
\end{equation}

By definition, $\|\mathcal F^{-1}\|$ can be rewrite as $\|\mathcal F^{-1}\| = \sup_{\|f\| = 1} \|\mathcal F^{-1}[f]\|$. Therefore, the original problem is equivalent to the following constrained optimizing problem:
\begin{equation}
\begin{aligned}
    \max &\ \ \|u\|^2\\
    s.t. &\ \ \|f\|^2 = 1\\
    \text{where} &\ \ \|f\|^2 = 4c^2 + \frac12\sum_{k\ge 1} a_k^2 + \frac12\sum_{k\ge 1} b_k^2\\
    &\ \ \|u\|^2 = \frac1{P^4}(\frac{8\pi^4}{15}c^2 - \frac{4\pi^2}3 c \sum_{k\ge 1}\frac{b_k}{k^2} - 4c\sum_{k\ge 1} \frac{b_k}{k^4}+\frac12\sum_{k\ge 1}\frac{a_k^2}{k^4}+\frac12\sum_{k\ge 1}\frac{b_k^2}{k^4}+(\sum_{k\ge 1}\frac{b_k}{k^2})^2).
\end{aligned}
\end{equation}

We then prove the following lemma.

\begin{lemma} \label{lemm:ak_is_zero}
When $\|u\|^2$ reaches its maximum, we have $a_k = 0, \forall k\ge 1$.
\end{lemma}
\begin{proof}
Firstly, it is obvious that $a_k = 0,\forall k\ge 2$. This is because the only term for $a_k$ is $\sum_{k\ge 1}\frac{a_k^2}{k^4}$. Thus, when $\exists k\ge 2, a_k\ne 0$, then it is better to move the value from $a_k$ to $a_1$.

Now we suppose $a_1\ne 0$. Since $\|f\|^2 = 4c^2 + \frac12\sum_{k\ge 1}a_k^2 + \frac12\sum_{k\ge 1}b_k^2 = 1$, we can replace $a_1^2$ by $2 - \sum_{k\ge 1}b_k^2 - 8c^2$. So we get the following problem:

\begin{equation}
\begin{aligned}
    \max &\ \ \|u\|^2=P^{-4}((\frac{8\pi^4}{15}-4)c^2 - \frac{4\pi^2}3 c \sum_{k\ge 1}\frac{b_k}{k^2} - 4c\sum_{k\ge 1} \frac{b_k}{k^4}+1 - \frac12\sum_{k\ge 1}b_k^2+\frac12\sum_{k\ge 1}\frac{b_k^2}{k^4}+(\sum_{k\ge 1}\frac{b_k}{k^2})^2)\\
    s.t. &\ \ 1 - \frac12\sum_{k\ge 1}b_k^2 - 4c^2> 0.
\end{aligned}
\end{equation}

To simplify the expression, we define $B = \sum_{k\ge 1} \frac{b_k}{k^2}$. When $\|u\|^2$ reaches its maximum, it must satisfy $\frac{\partial}{\partial b_j} \|u\|^2 = 0$:
\begin{equation}
    \frac{\partial}{\partial b_j} \|u\|^2 = P^{-4}(-\frac{4\pi^2}3c\frac1{j^2}-4c\frac1{j^4}-b_k+\frac{b_k}{j^4}+2B\frac1{j^2})=0.
\end{equation}

When $j=1$, we get $B = 2c(1+\frac{\pi^2}3)$. When $j\ge 2$, we can solve $b_j$ from the equation that $b_j = \frac{\frac{4\pi^2}3cj^2+4c-2Bj^2}{1-j^4} = \frac{4c}{1+j^2}$. Therefore, we can solve $b_1 = B - \sum_{k\ge 2} \frac{b_k}{k^2} = 2c(1+\pi\coth(\pi))$.

Now we define $d_k = b_k / c$, which are constants satisfying $d_1 = 2(1+\pi\coth(\pi))$ and $d_j = \frac{4}{1+j^2},\forall j\ge 2$. Then $\|u\|^2$ can be reformulized as:
\begin{equation}
\begin{aligned}
    \|u\|^2&=P^{-4}(1+c^2(\frac{8\pi^4}{15}-4 - \frac{4\pi^2}3\sum_{k\ge 1}\frac{d_k}{k^2} - 4\sum_{k\ge 1} \frac{d_k}{k^4} - \frac12\sum_{k\ge 1}d_k^2+\frac12\sum_{k\ge 1}\frac{d_k^2}{k^4}+(\sum_{k\ge 1}\frac{d_k}{k^2})^2))\\
    &=P^{-4}(1+c^2S).
\end{aligned}
\end{equation}
Where $S>0$. From the constraint that $1 - \frac12\sum_{k\ge 1}b_k^2 - 4c^2 = 1 - c^2(\frac12\sum_{k\ge 1}d_k^2+4)>0$, we can get the feasible interval of $c$: $c\in (-\sqrt{1/ (\frac12\sum_{k\ge 1}d_k^2+4)}, \sqrt{1/ (\frac12\sum_{k\ge 1}d_k^2+4)})$. In this way, $\|u\|^2$ has no maximum, leading to a contradiction. Therefore, we proved that $a_1$ should be zero.
\end{proof}

Finally, we provide a proof for Theorem~\ref{theo:poisson_condnumber}.
\begin{proof}
Given the conclusion in the Lemma \ref{lemm:ak_is_zero}, we will focus on $b_k$ and $c$ only. Now assume $c\ne 0$ and replace $b_k$ by $d_k = b_k / c$. 
\begin{equation}
\begin{aligned}
    \|f\|^2 &= c^2(4+ \frac12\sum_{k\ge 1}d_k^2) = 1,\\
    \|u\|^2 &= P^{-4}c^2(\frac{8\pi^4}{15} - \frac{4\pi^2}3 \sum_{k\ge 1}\frac{d_k}{k^2} - 4\sum_{k\ge 1} \frac{d_k}{k^4}+\frac12\sum_{k\ge 1}\frac{d_k^2}{k^4}+(\sum_{k\ge 1}\frac{d_k}{k^2})^2).
\end{aligned}
\end{equation}
By doing this, we can remove the constraint $\|f\|^2 = 1$ by replacing $c^2 = 2/(8+\sum_{k\ge 1}e_k^2 + \sum_{k\ge 1}d_k^2)$. Now our objective is simply maximizing:
\begin{equation}
    \|u\|^2 = \frac{\frac{8\pi^4}{15} - \frac{4\pi^2}3 \sum_{k\ge 1}\frac{d_k}{k^2} - 4\sum_{k\ge 1} \frac{d_k}{k^4}+\frac12\sum_{k\ge 1}\frac{d_k^2}{k^4}+(\sum_{k\ge 1}\frac{d_k}{k^2})^2}{P^4(8 + \sum_{k\ge 1}d_k^2)}.
\end{equation}

To simplify the long expression, we define $B=\sum_{k\ge 1} \frac{d_k}{k^2}$, $C = \sum_{k\ge 1} d_k^2$, $D=\sum_{k\ge 1} \frac{d_k}{k^4}$ and $E=\sum_{k\ge 1}\frac{d_k^2}{k^4}$ in the following proof.

When $\|u\|^2$ reaches its maximum, it must satisfy $\frac{\partial}{\partial d_j}\|u\|^2 = 0$. Thus we can get the following equation:
\begin{equation}
    \frac{\partial}{\partial d_j}\|u\|^2= \frac{(8 + C)(-\frac{4\pi^2}{3j^2}-\frac{4}{j^4}+\frac{d_j}{j^4}+2B\frac{1}{j^2})-2d_j(\frac{8\pi^4}{15} - \frac{4\pi^2}3 B - 4D+\frac12E+B^2)}{P^4(8 + C)^2} = 0.
\end{equation}
From the equation we can solve for $d_k$:
\begin{equation}
    d_k = \frac{((2B-\frac{4\pi^2}3)k^2-4)(8+C)}{(\frac{16\pi^4}{15}-\frac{8\pi^2}3B-8D+E+2B^2)k^4-8-C}.
\end{equation}

Now we learn that $d_k$ can be determined by $B, C, D, E$. We denote $d_k = g_k(B,C,D,E)$ and we can now solve $B,C,D,E$ from the 4 equations below:
\begin{equation}
\begin{aligned}
    B &=\sum_{k\ge 1} \frac{g_k(B,C,D,E)}{k^2},\\
    C &= \sum_{k\ge 1} g_k^2(B,C,D,E),\\
    D &=\sum_{k\ge 1} \frac{g_k(B,C,D,E)}{k^4},\\
    E &=\sum_{k\ge 1}\frac{g_k^2(B,C,D,E)}{k^4}.
\end{aligned}
\end{equation}

Where we get $B = \frac{2\pi^2}3-8, C = \pi^2-8, D=\frac{2(-720+60\pi^2+\pi^4)}{45}, E=\frac{8(-2160+210\pi^2+\pi^4)}{45}$. 

Thus, we get $d_k = -\frac{4}{4k^2-1}$ and $\|u\|^2 = 16P^{-4}$ for maximum value. So $\|\mathcal F^{-1}\|=\|u\|=4P^{-2}$

\end{proof}

\subsection{Proof for Corollary~\ref{theo:error_control}}\label{app:error_control}
\begin{proof}
Since $\mathrm{cond}(\mathcal{P}) < \infty$, we arbitrarily take $M>0$, then there exists $\xi > 0$ such that:
\begin{equation}
    \left| \sup_{0<\| \delta f \| \le \epsilon} \frac{\| \delta u \| \big/ \| u \|}{\| \delta f \| \big/ \| f \|} - \mathrm{cond}(\mathcal{P}) \right| < M,
\end{equation}
which holds for any $\epsilon \in (0, \xi)$.

Thus, we can defined $\alpha\colon (0, \xi) \rightarrow \mathbb{R}$ as:
\begin{equation}
    \alpha(x) =  \sup_{0<\| \delta f \| \le x} \frac{\| \delta u \| \big/ \| u \|}{\| \delta f \| \big/ \| f \|} - \mathrm{cond}(\mathcal{P}),
\end{equation}
which satisfies that $\lim_{x\to 0^+} \alpha(x) = 0$.

It follows that:
\begin{equation}
    \sup_{0<\| \delta f \| \le \epsilon} \frac{\| \delta u \| \big/ \| u \|}{\| \delta f \| \big/ \| f \|} = \mathrm{cond}(\mathcal{P}) + \alpha(\epsilon), \quad \forall \epsilon \in (0, \xi),
\end{equation}
which is equivalent to the statement that for any $\epsilon \in (0, \xi)$, when $0 < \sqrt{\mathcal{L}(\vth)} \le \epsilon$:
\begin{equation}\label{eq:app_a2}
    \frac{\| u_{\vth} - u \|}{\| u \|} \le 
    \left(\mathrm{cond}(\mathcal{P}) + \alpha(\epsilon) \right) \frac{\sqrt{\mathcal{L}(\vth)}}{\| f \|}, \quad \forall \vth \in \Theta.
\end{equation}
If $\sqrt{\mathcal{L}(\vth)} = 0$, then $u_{\vth} = u$ since the BVP is well-posed, and thus Eq.~\eqref{eq:app_a2} still holds.
\end{proof}

\subsection{Proof for Theorem~\ref{theo:ntk}}\label{app:ntk}
Let $f_{\vth} = \mathcal{F}[u_{\vth}]$. Substituting the expression for $c(t)$, we have that:
\begin{equation}
\begin{aligned}
    c(t) &= \frac{1}{N} \sum_{i=1}^N \left\| \frac{\partial \mathcal{F}[u_{\vth(t)}]}{\partial \vth}(\vx^{(i)}) \right\|^2\\
    &= \frac{1}{N} \sum_{i=1}^N \left\| \left( \frac{\partial \mathcal{F}[u_{\vth(t)}]}{\partial u} \circ \frac{\partial u_{\vth(t)}}{\partial \vth} \right)(\vx^{(i)}) \right\|^2 \\
    &\approx \frac{1}{|\Omega|}\left\|  \frac{\partial \mathcal{F}[u_{\vth(t)}]}{\partial u} \circ \frac{\partial u_{\vth(t)}}{\partial \vth} \right\|^2 &&\text{($L^2$ function norm)}\\
     &=\frac{1}{|\Omega|}\left\|  \left( D\mathcal{F}^{-1}[f_{\vth(t)}] \right)^{-1} \circ \frac{\partial u_{\vth(t)}}{\partial \vth} \right\|^2 \\
    &\ge \frac{1/|\Omega|}{\| D\mathcal{F}^{-1}[f_{\vth(t)}] \|^2}    \left\| \frac{\partial u_{\vth(t)}}{\partial \vth} \right\|^2 &&\text{(operator norm of $D\mathcal{F}^{-1}[f_{\vth(t)}]$)}\\
    &= \frac{\|f\|^2/(\|u\|^2 | \Omega |)}{(\mathrm{cond}(\mathcal{P}))^2 + \alpha(\| f_{\vth(t)} - f\|^2) } 
    \left\|   \frac{\partial u_{\vth(t)}}{\partial \vth} \right\|^2\\
    &= \frac{\|f\|^2/(\|u\|^2 | \Omega |)}{(\mathrm{cond}(\mathcal{P}))^2 + \alpha(\mathcal{L}(\vth(t))) } 
    \left\|   \frac{\partial u_{\vth(t)}}{\partial \vth} \right\|^2,
\end{aligned}
\end{equation}
where $D\mathcal{F}^{-1}[w]\colon W\rightarrow V$ is the Fréchet derivative of $\mathcal{F}^{-1}$ at $w$.


\section{Supplements for Section~\ref{sec:algorithm}}\label{app:algo}

\subsection{Detailed Derivation for Eq.~\eqref{eq:original_cond}}\label{app:original_cond}

\begin{lemma}\label{lemma:22}
    Supposing that $\mA \in \mathbb{R}^{N\times N}$ is invertible, we have:
    \begin{equation}
        \lim_{\epsilon \rightarrow 0^+} \sup_{\substack{0<\| \vv \| \le \epsilon \\ \vv \in \mathbb{R}^N}} \frac{\| \mA \vv \|}{\| \vv \|} = \| \mA \|.
    \end{equation}
\end{lemma}
\begin{proof}
For any $\epsilon > 0$, we firstly prove that:
\begin{equation}\label{eq:set_eq}
    \left\{ \frac{\| \mA \vv \|}{\| \vv \|} \colon 0<\| \vv \| \le \epsilon \land \vv \in \mathbb{R}^N \right\} = \left\{ \frac{\| \mA \vv \|}{\| \vv \|} \colon \| \vv \| \neq 0 \land \vv \in \mathbb{R}^N \right\}.
\end{equation}
We only need to prove that:
\begin{equation}
    \left\{ \frac{\| \mA \vv \|}{\| \vv \|} \colon 0<\| \vv \| \le \epsilon \land \vv \in \mathbb{R}^N \right\} \supseteq \left\{ \frac{\| \mA \vv \|}{\| \vv \|} \colon \| \vv \| \neq 0 \land \vv \in \mathbb{R}^N \right\},
\end{equation}
because the other direction is obvious. For any $a\in \left\{ \| \mA \vv \| / \| \vv \| \colon \| \vv \| \neq 0 \land \vv \in \mathbb{R}^N \right\}$, there exists $\vv$ with $\| \vv \| \neq 0$ such that $a = \| \mA \vv \| / \| \vv \| $. We consider $\vv' = \epsilon \vv / \|\vv \| $. It is clear that $\| \vv' \| = \epsilon$ and that:
\begin{equation}
    \frac{\| \mA \vv' \|}{\| \vv' \|} = \frac{\epsilon / \|\vv \| \| \mA \vv \|}{\epsilon / \|\vv \| \| \vv \|} = \frac{\| \mA \vv \|}{\| \vv \|} = a.
\end{equation}
Then, we have that $a\in \left\{ \| \mA \vv \| / \| \vv \| \colon 0<\| \vv \| \le \epsilon \land \vv \in \mathbb{R}^N \right\}$. Therefore, Eq.~\eqref{eq:set_eq} holds and thus:
\begin{equation}
    \sup\left\{ \frac{\| \mA \vv \|}{\| \vv \|} \colon 0<\| \vv \| \le \epsilon \land \vv \in \mathbb{R}^N \right\} = \sup\left\{ \frac{\| \mA \vv \|}{\| \vv \|} \colon \| \vv \| \neq 0 \land \vv \in \mathbb{R}^N \right\} = \| \mA \|.
\end{equation}
Let $\epsilon\rightarrow 0^+$, we finally prove that this lemma.

\end{proof}

We now start our derviation. Let $\vu_{\vth}$ denote the predictions of the neural network at the mesh locations: $\vu_{\vth}=( u_{\vth}(\vx^{(i)}) )_{i=1}^{N}$. From Definition~\ref{def:cond}, we have:
\begin{equation}\label{eq:22}
\begin{aligned}
    \mathrm{cond}(\mathcal{P}) &= \lim_{\epsilon\to 0^+} \sup_{{\substack{0<\| \delta f \| \le \epsilon\\ \vth\in \Theta}}} \frac{\| \delta u \| \big/ \| u \|}{\| \delta f \| \big/ \| f \|} \\
    &= \frac{\| f \|}{ \| u \|} \lim_{\epsilon\to 0^+} \sup_{{\substack{0<\| \mathcal{F}[u_{\vth}] - f \| \le \epsilon\\ \vth\in \Theta}}} \frac{\| u_{\vth} - u \|}{\| \mathcal{F}[u_{\vth}] - f \|}\\
    &\approx \frac{\| \vb \|}{ \| \vu \|} \lim_{\epsilon\to 0^+} \sup_{{\substack{0<\| \mA\vu_{\vth} - \vb \| \le \epsilon\\ \vth\in \Theta}}} \frac{\| \vu_{\vth} - \vu \|}{\| \mA\vu_{\vth} - \vb \|}\\
    &= \frac{\| \vb \|}{ \| \vu \|} \lim_{\epsilon\to 0^+} \sup_{{\substack{0<\| \mA(\vu_{\vth} - \vu) \| \le \epsilon\\ \vth\in \Theta}}} \frac{\| \vu_{\vth} - \vu \|}{\| \mA(\vu_{\vth} - \vu) \|},
\end{aligned}
\end{equation}
where the approximate equality holds because we discretize the BVP. Because of the assumption that the neural network has sufficient approximation capability (see Assumption~\ref{ass:7}) and the fact that $\| \mA \vv \| \le \| \mA \| \| \vv \|, \forall \vv \in \mathbb{R}^N$, Eq.~\eqref{eq:22} can be further rewritten as:
\begin{equation}\label{eq:23}
    \frac{\| \vb \|}{ \| \vu \|} \lim_{\epsilon\to 0^+} \sup_{{\substack{0<\| \vv \| \le \epsilon\\ \vv \in \mathbb{R}^N}}} \frac{\| \vv \|}{\| \mA\vv \|} = \frac{\| \vb \|}{ \| \vu \|} \| \mA^{-1} \|,
\end{equation}
where the equality holds according to Lemma~\ref{lemma:22}.

When we apply the precondition number $\mP$ satisfying that $\mP \approx \mA$ ($\mP^{-1} \approx \mA^{-1}$, also), the linear system transfers from $\mA \vu = \vb$ to $\mP^{-1}\mA \vu = \mP^{-1}\vb$. Equivalently, we have $\mA \rightarrow \mP^{-1}\mA$ and $\vb \rightarrow \mP^{-1}\vb$. Then, Eq.~\eqref{eq:23} becomes:
\begin{equation}
    \frac{\| \vb \|}{ \| \vu \|} \| \mA^{-1} \| \longrightarrow \frac{\| \mP^{-1}\vb \|}{\| \vu \|} \| \mA^{-1}\mP \| \approx \frac{\| \mA^{-1}\vb \|}{\| \vu \|} \| \mA^{-1}\mA \| = 1.
\end{equation}

\subsection{Enforcing Boundary Conditions via Discretized Losses}\label{app:algo:bc}

In this subsection, we will introduce how to enforce the boundary conditions (BCs) by our discretized loss function.

\paragraph{Dirichlet BCs.} We consider the following 1D Poisson equation:
\begin{equation}
\begin{aligned}
    \Delta u(x) &= 0, &&x\in \Omega=(0, 1),\\
    u(x) &= c, &&x\in \partial\Omega=\{0, 1\},
\end{aligned}
\end{equation}
where $u=u(x)$ is the unknown and $c\in \mathbb{R}$. We discretize the interval $[0, 1]$ into five points $\{ 0, 0.25, 0.5, 0.75, 1\}$ and construct the following discretized equation by the FDM:
\begin{equation}
    \frac{u(x+h) - 2u(x) + u(x-h)}{h^2} = 0, \quad x = \{ 0.25, 0.5, 0.75\},
\end{equation}
where $h=0.25$ and $u(0) = u(1) = c$. This can be reformulated as the following linear system:
\begin{equation}
    \begin{bmatrix}
    -2 & 1 & 0\\
    1 & -2 & 1\\
    0 & 1 & -2
    \end{bmatrix}
    \begin{bmatrix}
    u(0.75) \\
    u(0.5) \\
    u(0.25)
    \end{bmatrix} = 
    \begin{bmatrix}
    -c \\
    0 \\
    -c
    \end{bmatrix}.
\end{equation}
Now, we can see that the BC is enforced by substituting its values into the equation. Similar strategies can also be applied to other numerical schemes such as the FEM.

\paragraph{Neumann BCs and Robin BCs.}
Such types of BCs are typically enforced via the weak form of the PDEs. We consider the following Poisson equation with a Robin BC:
\begin{equation}
\begin{aligned}
    -\Delta u(\vx) &= f(\vx), &&\vx\in\Omega,\\
    \alpha u(\vx) + \beta \frac{\partial u}{\partial n}(\vx) &= g(\vx), &&\vx\in \partial\Omega,
\end{aligned}
\end{equation}
where $\alpha,\beta \in \mathbb{R}$, $\frac{\partial u}{\partial n}(\vx)$ is the normal derivative. The weak form is derived as:
\begin{equation}
    -\int_{\Omega} v\Delta u \diff{\vx} = \int_{\Omega}fv \diff{\vx},
\end{equation}
where $v\in H^1$ is the test function. Then, we perform integration by parts:
\begin{equation}
    \int_{\Omega} \nabla u \cdot \nabla v\diff{\vx} - \int_{\partial\Omega} \frac{\partial u}{\partial n}v\diff{\vx} = \int_{\Omega} fv\diff{\vx}.
\end{equation}
We plug in the Robin BC to obtain:
\begin{equation}
    \int_{\Omega} \nabla u \cdot \nabla v\diff{\vx} + \frac{\alpha}{\beta}\int_{\partial\Omega}uv\diff{\vx} = \int_{\Omega} fv\diff{\vx} + \frac{1}{\beta}\int_{\partial\Omega} gv\diff{\vx}.
\end{equation}
Finally, we assemble the above equation by the FEM and can obtain the loss that incorporates the BC. For other numerical schemes like FDM, we can plug in the finite difference formula of the derivative term to enforce the BC, which is similar to the cases of Dirichlet BCs.

\paragraph{Other BCs.}
For other forms of BCs, enforcement is usually implemented by substitution. For example, when dealing with left-right periodic BCs, we often substitute the values in the left boundary with the values in the right boundary. Or equivalently, we reduce the degrees of freedom of the left and right boundaries by half.

\begin{algorithm}[tb]
   \caption{Preconditoning PINNs for time-dependent problems (sequential)}
   \label{alg:tm}
\begin{algorithmic}[1]
   \STATE {\bfseries Input:} number of iterations $K$, mesh size $N$, learning rate $\eta$, time steps $\{ t_i \}_{i=1}^n$, initial condition $u_0(\vx)$, and initial parameters $\vth^{(0)}$
   \STATE {\bfseries Output:} solutions at each time steps $u_i(\vx), i=1,\dots, n$
   \FOR{$i=1,\dots, n$}
   \STATE Generate a mesh $\{ \vx^{(j)} \}_{j=1}^{N}$ for current time step
   \STATE Evaluate $u_{i-1}(\vx)$ on the mesh to obtain $\vu_{i-1}$
   \STATE Assemble the linear system $\mA'=(\mI + \mA(t_i)), \vb'= (\vb(t_i) + \vu_{i-1})$ according to Eq.~\eqref{eq:discr:tm}
   \STATE Compute the preconditioner for $\mA'$: $\mP = \widehat{\mL}\widehat{\mU}$ via ILU
   \FOR{$k=1,\dots,K$}
        \STATE Evaluate the neural network $u_{\vth^{(k-1)}}$ on mesh points: $\vu_{\vth^{(k-1)}}=( u_{\vth^{(k-1)}}(\vx^{(j)}) )_{j=1}^{N}$
        \STATE Compute the loss function $\mathcal{L}^\dagger(\vth^{(k-1)})$ by:
        
            \begin{equation}
            \mathcal{L}^\dagger(\vth)= \left\| \mP^{-1} (\mA' \vu_{\vth} - \vb')  \right\|^2
            \end{equation}
            
        \STATE Update the parameters via gradient descent: $\vth^{(k)} \leftarrow \vth^{(k-1)} - \eta \nabla_{\vth} \mathcal{L}^\dagger(\vth^{(k-1)})$
   \ENDFOR
   \STATE Let $u_i(\vx) \leftarrow u_{\vth^{(K)}}(\vx)$
   \STATE Let $\vth^{(0)} \leftarrow \vth^{(K)}$ (transfer learning)
   \ENDFOR

   {\bfseries Note:}
   \begin{enumerate}[label=(\alph*)]
       \item If the mesh $\{ \vx^{(j)} \}_{j=1}^{N}$, the matrix $\mA$, and the bias $\vb$ do not vary with time, we can only generate them once at the beginning instead of regeneration at each time step.
       \item We use transfer learning to migrate the neural network from the previous time step to the next time step since the solution varies little for most physical problems (if the number of time steps $n$ is sufficiently large). 
   \end{enumerate}
\end{algorithmic}
\end{algorithm}

\begin{algorithm}[tb]
   \caption{Preconditoning PINNs for time-dependent problems (parallelized)}
   \label{alg:tm:2}
\begin{algorithmic}[1]
   \STATE {\bfseries Input:} number of iterations $K$, mesh size $N$, learning rate $\eta$, time steps for $m$ sub-intervals $S_1 = \{ t_i^1 \}_{i=1}^{n}, \dots, S_m = \{ t_i^m \}_{i=1}^{n}$ (each sub-interval has $n$ steps), initial condition $u_0(\vx)$, and initial parameters $\vth^{(0)}_i, i=1,\dots,n$
   \STATE {\bfseries Output:} solutions at each time steps within each sub-interval $u^s_i(\vx), i=1,\dots, n, s = 1,\dots, m$
   \STATE {\bfseries Initialize:} $u^{1}_{0}(\vx) \leftarrow u_0(\vx)$
   \FOR{$s=1,\dots, m$}
   \STATE Generate a mesh $\{ \vx^{(j)} \}_{j=1}^{N}$ for current time step
   \STATE Evaluate $u^s_0(\vx)$ on the mesh to obtain $\vu_0^s$
   \STATE Assemble the matrix $\mA'_i=(\mI + \mA(t_i^s))$, $i=1,\dots,n$
   \STATE Compute the preconditioner for $\mA'_i$: $\mP_i = \widehat{\mL}_i\widehat{\mU}_i$ via ILU, $i=1,\dots,n$
   \FOR{$k=1,\dots,K$}
        \STATE Evaluate the neural network $u_{\vth^{(k-1)}_i}$ on mesh points: $\vu_{\vth^{(k-1)}_i}=( u_{\vth^{(k-1)}_i}(\vx^{(j)}) )_{j=1}^{N}$, $i=1,\dots,n$
        \STATE Assemble the bias $\vb'_1= (\vb(t_1^s) + \vu_0^s)$ and $\vb'_i= (\vb(t_i^s) + \vu_{\vth^{(k-1)}_{i-1}})$, where $i=2,\dots,n$
        \STATE Compute the loss function $\mathcal{L}^\dagger(\vth^{(k-1)}_1,\dots,\vth^{(k-1)}_n)$ by:
            \begin{equation}
            \mathcal{L}^\dagger(\vth_1,\dots, \vth_n)= \sum_{i=1}^{n}w_i\left\| \mP^{-1}_i (\mA'_i \vu_{\vth_i} - \vb'_i)  \right\|^2,
            \end{equation}
            where $w_i$ is the reweighting parameters of causality \citep{wang2022respecting}, satisfying that $\sum_{i=1}^n w_i = 1$
        \STATE Update the parameters via gradient descent: \\ $\vth^{(k)}_i \leftarrow \vth^{(k-1)}_i - \eta \nabla_{\vth_i} \mathcal{L}^\dagger(\vth^{(k-1)}_1, \dots, \vth^{(k-1)}_n)$, $i=1,\dots,n$
   \ENDFOR
   \STATE Let $u^s_i(\vx) \leftarrow u_{\vth^{(K)}_{i}}, i=1,\dots, n$
   \IF{$s < m$}
    \STATE Let $u^{s+1}_0(\vx) \leftarrow u^{s}_{n}(\vx)$
    \ENDIF
   \STATE Let $\vth^{(0)}_i \leftarrow \vth^{(K)}_i$ (transfer learning), $i=1,\dots,n$
   \ENDFOR

   {\bfseries Note:}
   \begin{enumerate}[label=(\alph*)]
       \item In our approach, we employ multiple neural networks, denoted as $u_{\vth_{i}}, i=1,\dots, n$, to predict the solution at each time step. During implementation, these networks share all their weights except for the final linear layer. This design choice ensures efficient memory usage without compromising the distinctiveness of each network's predictions.
   \end{enumerate}
\end{algorithmic}
\end{algorithm}

\subsection{Handling Time-Dependent \& Nonlinear Problems}\label{app:algo:tmnl}

We now introduce our strategies to handle time-dependent and nonlinear problems.

\paragraph{Time-Dependent Problems.}
For problems with time dependencies, one straightforward approach is to treat time as an additional spatial dimension, resulting in a unified spatial-temporal equation. For instance, supposing that we are dealing with a problem defined in a 2D square $[0,1]^2$ and a time interval $[0,1]$, we can consider it as a problem defined in a 3D cube $[0,1]^3$, where we build the mesh and assemble the equation system. However, this approach can necessitate extremely fine meshing to ensure adequate accuracy, particularly for problems with high temporal frequencies. 

An alternative approach involves discretizing the time dimension into specific time steps and subsequently solving the spatial equation iteratively for each step. For example, we consider the following abstraction of time-dependent PDEs:
\begin{equation}
    \frac{\partial u}{\partial t}(\vx, t) + \mathcal{F}[u](\vx, t) = f(\vx, t), \quad \forall\vx \in \Omega, t \in (0, T],
\end{equation}
with the initial condition of $\vu(\vx, 0) = h(\vx), \forall \vx \in \Omega$ and proper boundary conditions, where $t$ denotes the time coordinate, $T\in \mathbb{R}^+$, and $u$ is the unknown. We now discretize the time interval into $n$ time $t_0, t_1, \dots, t_n$ ($t_0 = 0, t_n = T$). Let $u_i(\vx)$ denote $u(\vx, t_i)$. Starting from $u_0(\vx) = h(\vx)$, we can construct the following iterative systems ($i=1,2,3,\dots,$):
\begin{equation}
     u_i(\vx) + (t_i - t_{i-1})\mathcal{F}[u_{i}](\vx, t_i) = (t_i - t_{i-1})f(\vx, t_i) + u_{i-1}(\vx), \quad \forall\vx \in \Omega.
\end{equation}
Then, we perform discretization in the spatial dimension with a mesh $\{ \vx^{(i)} \}_{i=1}^{N}$:
\begin{equation}\label{eq:discr:tm}
    (\mI + \mA(t_i) ) \vu_{i} = \vb(t_i) + \vu_{i-1},
\end{equation}
where $\mA(t_i), \vb(t_i)$ are matrices at time $t_i$ and $\vu_{i}=( u_{i}(\vx^{(j)}) )_{j=1}^{N}$. It is noted that the specific form of Eq.~\eqref{eq:discr:tm} depends on the numerical schemes employed. For example, when using the FEM, Eq.~\eqref{eq:discr:tm} should become:
\begin{equation}
    (\mK + \mA(t_i) ) \vu_{i} = \vb(t_i) + \mK \vu_{i-1},
\end{equation}
where $\mK$ is the mass matrix which simply integrates the trial and test functions.

Now, we can iteratively solve Eq.~\eqref{eq:discr:tm} with a PINN to obtain the solution at each time step. Specifically, we can sequentially solve each time step at one time, as described by Algorithm~\ref{alg:tm}, or divide the time interval into several sub-intervals and train in parallel within sub-intervals (see Algorithm~\ref{alg:tm:2}).

\begin{algorithm}[tb]
   \caption{Preconditoning PINNs for non-linear problems}
   \label{alg:nonlinear}
\begin{algorithmic}[1]
   \STATE {\bfseries Input:} number of iterations $K$, number of newton iteration $T$, mesh size $N$, learning rate $\eta$, initial guess $u_0(\vx)$, and initial parameters $\vth^{(0)}$
   \STATE {\bfseries Output:} solution $u_T(\vx)$
   \STATE Generate a mesh $\{ \vx^{(j)} \}_{j=1}^{N}$ for the problem domain $\Omega$
   \STATE Assemble the nonlinear system $\mF$
   \FOR{$i=1,\dots, T$}
   \STATE Evaluate $u_{i-1}(\vx)$ on the mesh to obtain $\vu_{i-1}$
   \STATE Compute the Jacobian matrix $J_{\mF}(\vu_{i-1})$
   \STATE Compute the preconditioner for $J_{\mF}(\vu_{i-1})$: $\mP = \widehat{\mL}\widehat{\mU}$ via ILU
   \FOR{$k=1,\dots,K$}
        \STATE Evaluate the neural network $u_{\vth^{(k-1)}}$ on mesh points: $\vu_{\vth^{(k-1)}}=( u_{\vth^{(k-1)}}(\vx^{(j)}) )_{j=1}^{N}$
        \STATE Compute the loss function $\mathcal{L}^\dagger(\vth^{(k-1)})$ by:
            \begin{equation}
            \mathcal{L}^\dagger(\vth)= \left\| \mP^{-1} (J_{\mF}(\vu_{i-1}) \vu_{\vth} - J_{\mF}(\vu_{i-1}) \vu_{i-1} + \mF(\vu_{i-1}))  \right\|^2
            \end{equation}
            
        \STATE Update the parameters via gradient descent: $\vth^{(k)} \leftarrow \vth^{(k-1)} - \eta \nabla_{\vth} \mathcal{L}^\dagger(\vth^{(k-1)})$
   \ENDFOR
   \STATE Let $u_i(\vx) \leftarrow u_{\vth^{(K)}}(\vx)$
   \STATE Let $\vth^{(0)} \leftarrow \vth^{(K)}$ (transfer learning)
   \ENDFOR

   {\bfseries Note:}
   \begin{enumerate}[label=(\alph*)]
       \item Here, we only present the vanilla Newton method, while a lot of advanced techniques could be applied, which include line search, relaxation, specific stopping criteria, and so on.
   \end{enumerate}
\end{algorithmic}
\end{algorithm}

\paragraph{Nonlinear Problems.}
 In the context of nonlinear problems, a strategy is to transfer the nonlinear components to the right-hand side and only precondition the linear portion. For example, we consider the following equation:
 \begin{equation}
     \Delta u(\vx) + \sin{u}(\vx) = f(\vx), \quad \forall \vx\in\Omega.
 \end{equation}
 We can simply move the nonlinear term $\sin{u}(\vx)$ to the right-hand-side and assemble:
 \begin{equation}
     \mA \vu = \vb - \sin{\vu}.
 \end{equation}
 Then, we can compute the preconditioner for the linear part $\mA$ and the loss function becomes $\mathcal{L}^\dagger(\vth) = \| \mP^{-1} (\mA \vu_{\vth} - \vb + \sin{\vu_{\vth}}) \|^2$. Nonetheless, this might lead to convergence issues in cases of highly nonlinearity. 
 
 To address this, we employ the Newton-Raphson method, allowing us to linearize the problem and then solve the associated linear tangent equation during each Newton iteration. Specifically, assembling a nonlinear problem results in a system of nonlinear equations:
 \begin{equation}
     \mF(\vu) = \bm{0},\quad \mF(\vu) = (F_1(\vu), \dots, F_m(\vu)),
 \end{equation}
 where $m$ is the number of nonlinear equations. The Newton-Raphson method solves the above equation with the following iterations ($i=1,2,3\dots,$):
 \begin{equation}
     \vu_{i} = \vu_{i-1} - J_{\mF}(\vu_{i-1})^{-1} \mF(\vu_{i-1}),
 \end{equation}
 where $J_{\mF}(\vu_{i-1})^{-1}$ the Jacobian matrix of $\mF$ at $\vu_{i-1}$. Now, we can use the neural network to solve the linear equation $J_{\mF}(\vu_{i-1}) \vu_{i} = J_{\mF}(\vu_{i-1}) \vu_{i-1} -\mF(\vu_{i-1})$ for $\vu_{i}$ and proceed the iteration. We provide a detailed description in Algorithm~\ref{alg:nonlinear}.

\section{Supplements for Section~\ref{sec:exp:cond}}\label{app:exp:cond}

\subsection{Environment and Global Settings}\label{app:exp:cond:env}

\paragraph{Environment.} We employ PyTorch \citep{paszke2019pytorch} as our deep-learning backend and base our physics-informed learning experiment on DeepXDE \citep{lu2021deepxde}. All models are trained on an NVIDIA TITAN Xp 12GB GPU in the operating system of Ubuntu 18.04.5 LTS. When analytical solutions are not available, we utilize the Finite Difference Method (FDM) to produce ground truth solutions for the PDEs.

\paragraph{Global Settings.} Unless otherwise stated, all the neural networks used are MLP of 5 hidden layers with 100 neurons in each layer. Besides, $\tanh$ is used for the activation function and Glorot normal \citep{glorot2010understanding} is used for trainable parameter initialization. The networks are all trained with an Adam optimizer \citep{kingma2014adam} (where the learning rate is $10^{-3}$ and $\beta_1 = \beta_2 = 0.99$) for 20000 iterations.

\subsection{Details of Wave, Burgers', and Helmholtz Equations}
\label{app:exp:cond:pde}

The specific definitions of the PDEs are shown below.

\paragraph{Wave Equation.} The governing PDE is:
\begin{equation}
    u_{tt} - C^2 u_{xx} = \left(\frac{\pi}{8}\right)^2(C^2 - 1)  \sin\left(\frac{\pi}8x\right)\cos\left(\frac{\pi}8t\right),
\end{equation}
with the boundary condition:
\begin{equation}
    u(0, t) = u(8, t) = 0,
\end{equation}
and initial condition:
\begin{align}
\begin{aligned}
    u(x,0) &= \sin \left(\frac{\pi}8x\right) + \frac12\sin\left(\frac{\pi}2x\right),\\
    u_t(x,0) &= 0,\\
\end{aligned}
\end{align}
defined on the domain $\Omega \times T = [0,8]\times [0,8]$, where $u = u(x,t)$ is the unknown.

The reference solution is:
\begin{equation}
    u(x, t) = \sin\left(\frac{\pi}8x\right)\cos\left(\frac{\pi}8t\right) + \frac12\sin\left(\frac\pi2x\right)\cos\left(\frac{C\pi}2t\right).
\end{equation}

In the experiment, we uniformly sample the value of parameter $C$ with a step of $0.1$ within the range $[1.1, 5]$.

\paragraph{Helmholtz Equation.} The governing PDE is:
\begin{equation}
    \Delta u + u = (1 - 2\pi^2A^2)\sin(A\pi x_1) \sin (A \pi x_2),
\end{equation}
with the boundary condition:
\begin{equation}
    u(x_1,0) = u(x_1, 1) = u(0, x_2) = u(1, x_2) = 0,
\end{equation}
defined on $\Omega = [0,1]^2$, where $u=u(\vx)=u(x_1,x_2)$ is the unknown.

The reference solution is:
\begin{equation}
    u(x, y) = \sin (A\pi x_1)\sin(A\pi x_2).
\end{equation}

In the experiment, we vary $A$ as integers between $1$ and $20$. 

\paragraph{Burgers' Equation.} The governing PDE on domain $\Omega \times T = [-1, 1] \times [0,1]$ is:
\begin{equation}
u_t + uu_x - \nu  u_{xx} = \sin(\pi x),
\end{equation}
with the boundary condition:
\begin{equation}
    u(-1, t) = u(1, t)= 0,
\end{equation}
and initial condition:
\begin{equation}
u(0, x) = -\sin (\pi x),
\end{equation}
where $u=u(x,t)$ is the unknown.

In the experiment, we uniformly sample 21 values of $\nu$ on a logarithmic scale (base 10) ranging from $10^{-2}$ to $1$. The reference solution is generated by the FDW with a mesh of $501 \times 21$, where the nonlinear algebra equation is solved by 10-step Newton iterations.

\subsection{Experimental Details}\label{app:exp:cond:detail}

\paragraph{Implementation Details.}
Firstly, we introduce how we numerically estimate the condition number:
\begin{enumerate}
    \item \textbf{FDM Approach:} We assemble the matrix $\mA$ with a specified uniform mesh. For linear PDEs, according to Eq.~\eqref{eq:original_cond}, we have that $\mathrm{cond}(\mathcal{P}) \approx \frac{\| \vb \|}{\| \vu \|} \| \mA^{-1} \|$. Therefore, we could approximate the condition number by calculating the norm of $\mA^{-1}$. For nonlinear PDEs, in light of Proposition~\ref{theo:prop}, we have $\mathrm{cond}(\mathcal{P}) = \frac{\| f \|}{\| u \|} \| D\mathcal{F}^{-1}[f] \|$ by assuming its Fréchet differentiablity. Then, we could approximate the condition number by the norm of the inverse of the Jacobian matrix of the discretized nonlinear equations.
    \item \textbf{Neural Network Approach:} According to the definition of the condition number, we can directly train a neural network to maximize:
    \begin{equation}
        \frac{\| \delta u \| \big/ \| u \|}{\| \delta f \| \big/ \| f \|}.
    \end{equation}
    where $\| \delta f \|$ are confined to a small value. For linear PDEs, we can simplify the problem to be computing this equation: $\|\mathcal F^{-1}\| = \sup_{\|f\|=1} \frac{\|\mathcal{F}^{-1}[f]\|}{\|f\|} = \sup_{\|f\|=1} \frac{\|u_{\vth}\|}{\|f\|}$. Since the operator is linear, we can further remove the constraint $\|f\|=1$ and optimize $\frac{\|u_{\vth} \|}{\|f\|}=\frac{\|u_{\vth} \|}{\|\mathcal F(u_{\vth})\|}$ over the parameter space to find the maximum, which will be minimizing its reciprocal or its opposite.
\end{enumerate}




\paragraph{Hyper-parameters.}
Secondly, we introduce the hyper-parameters of computing solution or the condition number for each problem:
\begin{itemize}
    \item \textbf{1D Poisson Equation:} We employ a mesh of the size $100$ for FDM. The hard-constraint ansatz for the PINN is: $x(2\pi / P - x) / (\pi / P)^2 u_{\vth}$. We use $2048$ collocation points and $128$ boundary points to train the PINN for $5000$ epochs to compute the condition number.
    \item \textbf{Wave Equation:} We employ a mesh of the size $50\times 50$ for FDM. The hard-constraint ansatz for the PINN is: $u_0 + x(8-x)/16 \cdot (t(12-t))^2 / 256 \cdot u_{\vth}$, where $t$ is time and $u_0$ is the initial condition. We use $8192$ collocation points and $2048$ boundary points to train the PINN with the learning rate of $10^{-4}$. 
    \item \textbf{Helmholtz Equation:} We employ a mesh of the size $50\times 50$ for FDM. The hard-constraint ansatz for the PINN is: $\alpha u_{\vth} + (1-\alpha)\sin(A \pi  x) \sin(A \pi  y)$, where $\alpha = 16x (1-x) y (1-y)$. We use $8192$ collocation points and $2048$ boundary points to train the PINN. 
    \item \textbf{Burgers' Equation:} We employ a mesh of the size $500\times 20$ for FDM. The hard-constraint ansatz for the PINN is: $\alpha (1-\beta) u_{\vth} - \beta \sin(\pi  x)$, where $\alpha = (1+x)(1-x), \beta= \exp{(-t)}$. We use $8192$ collocation points and $2048$ boundary points to train the PINN.
\end{itemize}


\paragraph{Nomralization of the Condition Number.}
For Burgers equation and Wave equation, we set:
\begin{equation}
    \mathrm{normalized}\ \mathrm{cond}(\mathcal{P}) = \mathrm{MinMax}(\log(\mathrm{cond}(\mathcal{P}) + c))
\end{equation}
where $c=0$ for Wave equation.
For the Helmholtz equation, we select 
\begin{equation}
    \mathrm{normalized}\ \mathrm{cond}(\mathcal{P}) = \mathrm{MinMax}(\sqrt{\mathrm{cond}(\mathcal{P})})
\end{equation}
as the normalizer. Here, $\mathrm{MinMax}(\cdot)$ denotes a min-max normalization for the given sequence to ensure the final values living in $[0, 1]$.

\subsection{Physical Interpretation for Correlation Between PINN Error and Condition Number}\label{app:exp:connection}

Figure \ref{fig:cond_convergence} unveils a robust linear association between the normalized condition number and the log-scaled L2 relative error (L2RE). This correlation can be expressed as:
\begin{equation*}\label{eq:app:phy}
    \log(\mathrm{L2RE}) \appropto \mathrm{normalized}\ \mathrm{cond}(\mathcal{P}),
\end{equation*}
where, for simplicity, we omit the bias term (similarly in subsequent derivations).

To demystify this pronounced correlation, we first investigate the spectral behaviors of PINNs in approximating functions. When a neural network mimics the solutions of PDEs, it might exhibit a spectral bias. This implies that networks are more adept at capturing low-frequency components than their high-frequency counterparts \citep{pmlr-v97-rahaman19a}. Recent studies have empirically demonstrated an exponential preference of neural networks towards frequency \cite{xu2019frequency}. This leads to the inference that the error could be exponentially influenced by the system's frequency. Hence, it is plausible to represent this relationship as:
\begin{equation*}
    \log(\mathrm{L2RE}) \appropto \mathrm{Frequency}.
\end{equation*}

In what follows, we explore how $\mathrm{Frequency}$ correlates with $\mathrm{cond}(\mathcal{P})$. Using $\mathrm{Frequency}$ as a bridge, we will model the relationship between $\log(\mathrm{L2RE})$ and $\mathrm{cond}(\mathcal{P})$.
\begin{itemize}
    \item \textbf{Helmholtz Equation:} Here, $\mathcal{F}^{-1}$ remains constant with the parameter $A$. This implies that $\mathrm{cond}(\mathcal{P}) \propto \frac{\| f \|}{\| u \|} = |1 - 2\pi^2 A^2|$. Given that $A$ determines the solution's frequency, we infer that $\sqrt{\mathrm{cond}(\mathcal{P})} \appropto \mathrm{Frequency}$. This leads to the conclusion that  $\log(\mathrm{L2RE}) \appropto \sqrt{\mathrm{cond}(\mathcal{P})}$, aligning with our experimental findings.
    \item \textbf{Wave \& Burgers’ Equation:} For these equations, the parameters $C$ and $\nu$ influence the frequency of both the solution and the operator $\mathcal{F}$. Given their similar roles, we use the wave equation to elucidate the relationship between the condition number and the parameter. This relationship is found to be \emph{at least exponential}. Based on Proposition~\ref{theo:prop}, we define $\mathcal{P}_1$ as:
    \begin{equation}
        u_{tt} - C^2 u_{xx} = 0,
    \end{equation}
    maintaining the initial and boundary conditions. Assuming $\mathcal{P}_1$ is well-posed, we introduce $\mathcal{G}[w] = \mathcal{F}^{-1}[w] - u_1$ for every $w$ in $S$, where $u_1$ is the solution to $\mathcal{P}_1$. Chossing a particular $f_0(x,t) = C^4 ( -e^{C^2 t} (1 + Kx) + e^{Cx} (1+C^2t) )$ with $K = \frac{e^{8C} - 1}{8}$, we derive $\mathcal{G}[f_0](x,t) = (e^{C^2 t} - 1 - C^2 t)(e^{Cx} - 1 - Kx)$. Consequently, we obtain:
    \begin{equation}
        \mathrm{cond}(\mathcal{P}) = \frac{\| f \|}{\| u \|} \left\|\mathcal{G}\right\| \ge \frac{\| f \|}{\| u \|} \frac{\| \mathcal{G}[f_0] \|}{\| f_0 \|} \appropto \frac{e^{kC}}{C^n},
    \end{equation}
    where $k, n$ are constants independent of $C$. In summary, we deduce $\log(\mathrm{cond}(\mathcal{P})) \appropto \mathrm{Frequency}$, leading to $\log(\mathrm{L2RE}) \appropto \log(\mathrm{cond}(\mathcal{P}))$.
\end{itemize}







\section{Supplements for Section~\ref{sec:exp:forward}}\label{app:exp:forward}

\subsection{Environment and Global Settings}\label{app:exp:env}

\paragraph{Environment.} The environment settings are basically consistent with that in Appendix~\ref{app:exp:cond:env}, except that:
\begin{itemize}
    \item The model in NS2d-CG is trained on an Tesla V100-PCIE 16GB GPU. If you want to in a GPU with lower memory, you can specify \texttt{Use Sparse Solver = True} in the configuration to save memory.
    \item The reference data are generated by the work of \cite{hao2023pinnacle}.
    \item We employ the finite element method (FEM) for discretization, utilizing FEniCS \citep{alnaes2015fenics} as the platform.
\end{itemize}

\paragraph{Global Settings.} Unless otherwise stated, we adopt the following settings:
\begin{itemize}
    \item For 2D problems (including the time dimension), we employ the MLP of 3 hidden layers with 64 neurons in each layer. For 3D problems (including the time dimension), we employ the MLP of 5 hidden layers with 128 neurons in each layer. Besides, \texttt{SiLU} is used for the activation function. The initialization method is the default one in PyTorch. And we employ 10-dimensional Fourier features, as detailed in \citep{tancik2020fourier}, uniformly sampled on a logarithmic scale (base 2) spanning $2\pi \times [2^{-5}, 2^{5}]$. 
    \item The networks are all trained with an Adam optimizer \citep{kingma2014adam} (where the learning rate is $10^{-3}$ and $\beta_1 = 0.9, \beta_2 = 0.99$) for 20000 iterations.
    \item The results of baselines are from the paper \citep{hao2023pinnacle}, except the computation time results, which are re-evaluated in the same environment as our method. 
\end{itemize}

\paragraph{Baselines Introduction.}
We redirect readers to the Section 3.3.1 in the paper \citep{hao2023pinnacle}.


\subsection{PDE Problems' Introduction and Implementation Details}\label{app:exp:detail}

In this section, we briefly describe PDE problems considered in PINNacle \cite{hao2023pinnacle} used in our experiment, as well as the implementation and hyper-parameters for our method. We refer to the original paper \citep{hao2023pinnacle} for the problem details such as initial conditions and boundary conditions.

\paragraph{Burgers1d-C.}
The equation is given by:
\begin{equation}
    \frac{\partial u}{\partial t} + u u_x = \nu u_{x x},
\end{equation}
define on $\Omega\times T =[- 1, 1] \times [0, 1]$, where $u=u(x,t)$ is the unknown, $\Omega$ is the spatial domain whereas $T$ is the temporal domain (the same below). In this and subsequent PDE problems, initial conditions and boundary conditions are omitted for clarity unless specified otherwise. Let $\Omega' = \Omega\times T, x' = (x,t)$. The weak form is expressed as:
\begin{equation}
     \int_{\Omega'} \frac{\partial u}{\partial t}\cdot v \diff{x'}  + \int_{\Omega'} (uu_x)\cdot v \diff{x'} + \nu \int_{\Omega'} u_x\cdot v_x \diff{x'}= 0,
\end{equation}
where $v$ is the test function. We employ the FEniCS to discretize the problem with a mesh of size $500\times 20$. Given that the matrix size remains within the memory constraints, we utilize a dense matrix implementation for faster matrix computations. The drop tolerance of the ILU is $10^{-4}$. We solve the problem with $10$-step Newton iterations (see Algorithm~\ref{alg:nonlinear}) and train the neural model for $2000$ iterations in each Newton step.

\paragraph{Burgers2d-C.}
The equation is given by:
\begin{equation}
    \frac{\partial\bm{u}}{\partial t}+\vu \cdot \nabla \vu-\nu \Delta \vu=0,
\end{equation}
defined on $\Omega\times T =[0, 4]^2 \times [0, 1]$, where $\vu=(u_1(\vx,t), u_2(\vx,t))$ is the unknown. We solve this problem by an (implicit) time-stepping scheme (see Algorithm~\ref{alg:tm:2}). The number of sub-time intervals is $50$, with each interval having $10$ steps. The weak form is expressed as:
\begin{equation}
     \int_{\Omega} \vu_1 \cdot \vv \diff{\vx}  + \delta t \nu \int_{\Omega} \nabla \vu_1 \cdot \nabla \vv \diff{\vx} + \delta t \int_{\Omega} (\vu_1 \cdot \nabla \vu_1 ) \cdot \vv \diff{\vx} = \int_{\Omega} \vu_0 \cdot \vv \diff{\vx},
\end{equation}
where $\vu_0 = \vu_0(\vx)$ is the solution at the previous time step, $\vu_1 = \vu_1(\vx)$ is the solution at current time step, $\vv = \vv(\vx)$ is the test function, and $\delta t = 1/500$ is the time step length. We employ the FEniCS to discretize the problem with an external mesh including $12657$ nodes generated by COMSOL Multiphysics (commercial software for FEM \citep{comsol}). It is noted that we do not employ a Newton method to solve the discretized nonlinear equations since the time overhead is too high. Instead, we only precondition the linear portion (see Appendix~\ref{app:algo:tmnl}) and let the neural model find the correct solution by gradient descent. Besides, we utilize a sparse matrix implementation since the matrix size exceeds the memory constraint. The drop tolerance of the ILU is $10^{-1}$. We train the model for $2000$ iterations in each sub-time interval while $40000$ iterations in the first interval (i.e., cold-start training). Finally, in this problem, we employ an MLP of $5$ layers with $128$ neurons in each layer as our neural model.

\paragraph{Poisson2d-C.}
The equation is given by:
\begin{equation}
    -\Delta u=0,
\end{equation}
defined on a 2D irregular domain $\Omega$, a rectangular domain $[-0.5, 0.5]^2$ with four circular voids of the same size, where $u=u(\vx)$ is the unknown. The weak form is expressed as:
\begin{equation}
    \int_{\Omega} \nabla u \cdot \nabla v \diff{\vx} = 0,
\end{equation}
where $v$ is the test function. We employ the FEniCS to discretize the problem with an external mesh including $10602$ nodes generated by the Gmsh \citep{geuzaine2009gmsh}. Given that the matrix size remains within the memory constraints, we utilize a dense matrix implementation for faster matrix computations. The drop tolerance of the ILU is $10^{-3}$.

\paragraph{Poisson2d-CG.}
The equation is given by:
\begin{equation}
    -\Delta u + k^2 u=f,
\end{equation}
defined on a 2D irregular domain $\Omega$, a rectangular domain $[-1, 1]^2$ with four circular voids of different sizes, where $u=u(\vx)$ is the unknown, $k=8$, and $f=f(\vx)$ is given. The weak form is expressed as:
\begin{equation}
    \int_{\Omega} \nabla u \cdot \nabla v \diff{\vx} + k^2 \int_{\Omega} u \cdot v \diff{\vx} = \int_{\Omega} f \cdot v \diff{\vx},
\end{equation}
where $v$ is the test function. We employ the FEniCS to discretize the problem with an external mesh including $9382$ nodes generated by the Gmsh. Given that the matrix size remains within the memory constraints, we utilize a dense matrix implementation for faster matrix computations. The drop tolerance of the ILU is $10^{-3}$.

\paragraph{Poisson3d-CG.}
The equation is given by:
\begin{equation}
    -\mu_i \Delta u + k_i^2 u=f \quad\text{in } \Omega_i,\quad i=1,2,
\end{equation}
defined on a 3D irregular domain $\Omega$, a cubic domain $[0, 1]^3$ with four spherical voids of different sizes, where $u=u(\vx)$ is the unknown, $\Omega_1 = \Omega \cap \{ \vx = (x_1, x_2, x_3) \mid x_3 < 0.5 \}$, $\Omega_2 = \Omega \cap \{ \vx = (x_1, x_2, x_3) \mid x_3 \ge 0.5 \}$, $\mu_1 = \mu_2 =1$, $k_1 = 8, k_2 =10$, and $f=f(\vx)$ is given. The weak form is expressed as:
\begin{equation}
    \mu_1 \int_{\Omega_1} \nabla u \cdot \nabla v \diff{\vx} + k_1^2 \int_{\Omega_1} u \cdot v \diff{\vx} + \mu_2 \int_{\Omega_2} \nabla u \cdot \nabla v \diff{\vx} + k_2^2 \int_{\Omega_2} u \cdot v \diff{\vx}= \int_{\Omega} f \cdot v \diff{\vx},
\end{equation}
where $v$ is the test function. We employ the FEniCS to discretize the problem with an external mesh including $13680$ nodes generated by the Gmsh. Given that the matrix size remains within the memory constraints, we utilize a dense matrix implementation for faster matrix computations. The drop tolerance of the ILU is $10^{-3}$.

\paragraph{Poisson2d-MS.}
The equation is given by:
\begin{equation}
\begin{aligned}
    -\nabla( a \nabla u) &= f  &&\text{in } \Omega,\\
    \frac{\partial u}{\partial n} + u &= 0  &&\text{in } \partial\Omega,
\end{aligned}
\end{equation}
defined on $\Omega = [-10, 10]^2$, where $u=u(\vx)$ is the unknown and $a=a(\vx)$ denotes a predefined function. Notably, $\Omega$ is partitioned into a $5\times 5$ grid of uniform cells. Within each cell, $a$ takes a piecewise linear form, introducing discontinuities at the cell boundaries. We define the weak form to be:
\begin{equation}
    \int_{\Omega} a(\nabla u \cdot \nabla v) \diff{\vx} + \int_{\partial\Omega} a (u \cdot v) \diff{\vx}= \int_{\Omega} f \cdot v \diff{\vx},
\end{equation}
where $v$ is the test function. We employ the FEniCS to discretize the problem with a mesh of size $100\times 100$. Given that the matrix size remains within the memory constraints, we utilize a dense matrix implementation for faster matrix computations. The drop tolerance of the ILU is $10^{-3}$. Finally, in this problem, we employ a Fourier MLP of $5$ layers with $128$ neurons in each layer as our neural model, where the Fourier features have a dimension of 128 and are sampled in $\mathcal{N}(0, \pi)$.

\paragraph{Heat2d-VC.}
The equation is given by:
\begin{equation}
    \frac{\partial u}{\partial t}-\nabla( a \nabla u) = f,
\end{equation}
define on $\Omega\times T =[0, 1]^2 \times [0, 5]$, where $u=u(\vx,t)$ is the unknown and $a=a(\vx)$ denotes a predefined function with multi-scale frequencies. Let $\Omega' = \Omega\times T, \vx' = (\vx,t)$. We define the weak form to be:
\begin{equation}
    \int_{\Omega'} \frac{\partial u}{\partial t} \cdot v \diff{\vx'} + \int_{\Omega'} a(\nabla u \cdot \nabla v) \diff{\vx'} = \int_{\Omega'} f \cdot v \diff{\vx'},
\end{equation}
where $v$ is the test function. We employ the FEniCS to discretize the problem with a mesh of size $20 \times 100\times 100$. Besides, we utilize a sparse matrix implementation since the matrix size exceeds the memory constraint. The drop tolerance of the ILU is $10^{-1}$. Finally, in this problem, we employ a Fourier MLP of $5$ layers with $128$ neurons in each layer as our neural model, where the Fourier features have a dimension of 128 and are sampled in $\mathcal{N}(0, \pi)$.

\paragraph{Heat2d-MS.}
The equation is given by:
\begin{equation}
    \frac{\partial u}{\partial t}-\nabla \cdot \left(\left(\frac{1}{(500 \pi)^2}, \frac{1}{(\pi)^2}\right) \odot \nabla u \right) = 0,
\end{equation}
define on $\Omega\times T =[0, 1]^2 \times [0, 5]$, where $u=u(\vx,t)$ is the unknown and $\odot$ denotes an element-wise multiplication. Let $\Omega' = \Omega\times T, \vx' = (\vx,t)$. We define the weak form to be:
\begin{equation}
    \int_{\Omega'} \frac{\partial u}{\partial t} \cdot v \diff{\vx'} + \int_{\Omega'} \left(\left(\frac{1}{(500 \pi)^2}, \frac{1}{(\pi)^2}\right) \odot \nabla u \right)\cdot \nabla v \diff{\vx'} = 0,
\end{equation}
where $v$ is the test function. We employ the FEniCS to discretize the problem with a mesh of size $500 \times 20\times 20$. Besides, we utilize a sparse matrix implementation since the matrix size exceeds the memory constraint. The drop tolerance of the ILU is $10^{-1}$. Finally, in this problem, we employ an MLP of $5$ layers with $128$ neurons in each layer as our neural model. The model is trained for $50000$ iterations.

\paragraph{Heat2d-CG.}
The equation is given by:
\begin{equation}
\begin{aligned}
    \frac{\partial u}{\partial t}-\Delta u &= 0 &&\text{in } \Omega\times T,\\
    \frac{\partial u}{\partial n} &= 5 - u &&\text{in } \partial\Omega_{\mathrm{large}}\times T,\\
    \frac{\partial u}{\partial n} &= 1 - u &&\text{in } \partial\Omega_{\mathrm{small}}\times T,\\
    \frac{\partial u}{\partial n} &= 0.1 - u &&\text{in } \partial\Omega_{\mathrm{outer}}\times T,
\end{aligned}
\end{equation}
define on $\Omega\times T$, where $T = [0, 3]$, $\Omega$ is a rectangular domain $[-8, 8]\times [-12, 12]$ with eleven large circular voids and six small circular voids, and $u=u(\vx,t)$ is the unknown. Here, $\partial\Omega_{\mathrm{large}}$ denotes the inner large circular boundary, $\partial\Omega_{\mathrm{small}}$ the inner small circular boundary, $\partial\Omega_{\mathrm{outer}}$ the outer rectangular boundary, and $\partial\Omega_{\mathrm{large}} \cup \partial\Omega_{\mathrm{small}} \cup \partial\Omega_{\mathrm{outer}} = \partial \Omega$. We let:
\begin{equation}
\begin{aligned}
    \Omega' &= \Omega\times T,\\
    \partial\Omega_{\mathrm{large}}' &= \partial\Omega_{\mathrm{large}}\times T,\\
    \partial\Omega_{\mathrm{small}}' &= \partial\Omega_{\mathrm{small}}\times T,\\
    \partial\Omega_{\mathrm{outer}}' &= \partial\Omega_{\mathrm{outer}}\times T,
\end{aligned}
\end{equation}
and $\vx' = (\vx,t)$.
We define the weak form to be:
\begin{equation}
\begin{aligned}
    \int_{\Omega'} \frac{\partial u}{\partial t} \cdot v \diff{\vx'} + \int_{\Omega'} \nabla u \cdot \nabla v \diff{\vx'}
    - \int_{\partial\Omega_{\mathrm{large}}'} (5-u) \cdot  v \diff{\vx'} &\\
    - \int_{\partial\Omega_{\mathrm{small}}'} (1-u) \cdot  v \diff{\vx'} - \int_{\partial\Omega_{\mathrm{outer}}'} (0.1-u) \cdot  v \diff{\vx'} &= 0,
\end{aligned}
\end{equation}
where $v$ is the test function. We employ the FEniCS to discretize the problem with an external mesh including $255946$ nodes generated by the Gmsh. Besides, we utilize a sparse matrix implementation since the matrix size exceeds the memory constraint. The drop tolerance of the ILU is $10^{-1}$.

\paragraph{Heat2d-LT.}
The equation is given by:
\begin{equation}
    \frac{\partial u}{\partial t} = 0.001\Delta u + 5\sin{(u^2)}f,
\end{equation}
define on $\Omega\times T =[0, 1]^2 \times [0, 100]$, where $u=u(\vx,t)$ is the unknown and $f=f(\vx, t)$ is given. We solve this problem by an (implicit) time-stepping scheme (see Algorithm~\ref{alg:tm:2}). The number of sub-time intervals is $2000$, with each interval having $1$ step. We define the weak form to be:
\begin{equation}
    \int_{\Omega} u_1 \cdot v \diff{\vx} + 0.001\delta t\int_{\Omega} \nabla u_1 \cdot \nabla v \diff{\vx} - \delta t \int_{\Omega} \left( 5\sin{(u_1^2)}f \right) \cdot v \diff{\vx'} = \int_{\Omega} u_0 \cdot v \diff{\vx},
\end{equation}
where $u_0 = u_0(\vx)$ is the solution at the previous time step, $u_1 = u_1(\vx)$ is the solution at current time step, $v = v(\vx)$ is the test function, and $\delta t = 1/2000$ is the time step length. We employ the FEniCS to discretize the problem with a mesh of size $20 \times 20$. It is noted that we do not employ a Newton method to solve the discretized nonlinear equations since the time overhead is too high. Instead, we only precondition the linear portion (see Appendix~\ref{app:algo:tmnl}) and let the neural model find the correct solution by gradient descent. Given that the matrix size remains within the memory constraints, we utilize a dense matrix implementation for faster matrix computations. The drop tolerance of the ILU is $10^{-4}$. We train the model for $1000$ iterations in each sub-time interval while $100000$ iterations in the first interval (i.e., cold-start training). Finally, in this problem, we employ an MLP of $5$ layers with $128$ neurons in each layer as our neural model.

\paragraph{NS2d-C.}
The equation is given by:
\begin{equation}
\begin{aligned}
     \bm{u} \cdot \nabla \boldsymbol{u}+\nabla p-\frac{1}{R e} \Delta \boldsymbol{u}&=0,\\
\nabla \cdot \boldsymbol{u}&=0,
\end{aligned}
\end{equation}
defined on $\Omega =[0, 1]^2$, where $\vu=(u_1(\vx), u_2(\vx))$ and $p$ are the unknown velocity and pressure, respectively, and $Re$ is the Reynolds number. The weak form is expressed as:
\begin{equation}
    \frac{1}{R e} \int_{\Omega} \nabla \vu \cdot \nabla \vv \diff{\vx} + \int_{\Omega} (\vu \cdot \nabla \vu ) \cdot \vv \diff{\vx} - \int_{\Omega} p \nabla \vv \diff{\vx} - \int_{\Omega} q \nabla \vu \diff{\vx} = 0,
\end{equation}
where $\vv = \vv(\vx)$ and $q = q(\vx)$ are, respectively, the test functions corresponding to $\vu$ and $p$. We employ the FEniCS to discretize the problem with a mesh of size $50 \times 50$. Given that the matrix size remains within the memory constraints, we utilize a dense matrix implementation for faster matrix computations. The drop tolerance of the ILU is $10^{-4}$. We solve the problem with $20$-step Newton iterations (see Algorithm~\ref{alg:nonlinear}) and train the neural model for $1000$ iterations in each Newton step.

\paragraph{NS2d-CG.}
The equation is given by:
\begin{equation}
\begin{aligned}
     \bm{u} \cdot \nabla \boldsymbol{u}+\nabla p-\frac{1}{R e} \Delta \boldsymbol{u}&=0,\\
\nabla \cdot \boldsymbol{u}&=0,
\end{aligned}
\end{equation}
defined on $\Omega =[0, 4]\times [0,2] \setminus ([0,2] \times [1,2])$, where $\vu=(u_1(\vx), u_2(\vx))$ and $p$ are the unknown velocity and pressure, respectively, and $Re$ is the Reynolds number. The weak form is expressed as:
\begin{equation}
    \frac{1}{R e} \int_{\Omega} \nabla \vu \cdot \nabla \vv \diff{\vx} + \int_{\Omega} (\vu \cdot \nabla \vu ) \cdot \vv \diff{\vx} - \int_{\Omega} p \nabla \vv \diff{\vx} - \int_{\Omega} q \nabla \vu \diff{\vx} = 0,
\end{equation}
where $\vv = \vv(\vx)$ and $q = q(\vx)$ are, respectively, the test functions corresponding to $\vu$ and $p$. We employ the FEniCS to discretize the problem with an external mesh including $2907$ nodes generated by the Gmsh. Given that the matrix size remains within the memory constraints, we utilize a dense matrix implementation for faster matrix computations. The drop tolerance of the ILU is $10^{-4}$. We solve the problem with $20$-step Newton iterations (see Algorithm~\ref{alg:nonlinear}) and train the neural model for $1000$ iterations in each Newton step.

\paragraph{NS2d-LT.}
The equation is given by:
\begin{equation}
\begin{aligned}
    \frac{\partial\bm{u}}{\partial t} + \bm{u} \cdot \nabla \boldsymbol{u}+\nabla p-\frac{1}{R e} \Delta \boldsymbol{u}&=f,\\
\nabla \cdot \boldsymbol{u}&=0, 
\end{aligned}
\end{equation}
defined on $\Omega\times T =([0, 2]\times[0,1]) \times [0, 5]$, where $\vu=(u_1(\vx), u_2(\vx))$ and $p$ are the unknown velocity and pressure, respectively, $Re$ is the Reynolds number, and $f = f(\vx, t)$ is predefined. We solve this problem by an (implicit) time-stepping scheme (see Algorithm~\ref{alg:tm:2}). The number of sub-time intervals is $50$, with each interval having $1$ step. The weak form is expressed as:
\begin{equation}
\begin{aligned}
     \int_{\Omega} \vu_1 \cdot \vv \diff{\vx}  + \delta t \frac{1}{R e} \int_{\Omega} \nabla \vu_1 \cdot \nabla \vv \diff{\vx} + \delta t \int_{\Omega} (\vu_1 \cdot \nabla \vu_1 ) \cdot \vv \diff{\vx} &\\
     - \delta t \int_{\Omega} p_1 \nabla \vv \diff{\vx} - \delta t \int_{\Omega} q \nabla \vu_1 \diff{\vx} &= \int_{\Omega} \vu_0 \cdot \vv \diff{\vx},
\end{aligned}
\end{equation}
where $\vu_0 = \vu_0(\vx)$ is the velocity at the previous time step, $\vu_1 = \vu_1(\vx)$ and $p_1 = p_1(\vx)$ are the velocity and pressure at current time step, $\vv = \vv(\vx), q=q(\vx)$ are the test functions corresponding to velocity and pressure, and $\delta t = 1/50$ is the time step length. We employ the FEniCS to discretize the problem with a mesh of size $60 \times 30$. It is noted that we do not employ a Newton method to solve the discretized nonlinear equations since the time overhead is too high. Instead, we only precondition the linear portion (see Appendix~\ref{app:algo:tmnl}) and let the neural model find the correct solution by gradient descent. Given that the matrix size remains within the memory constraints, we utilize a dense matrix implementation for faster matrix computations. The drop tolerance of the ILU is $10^{-4}$. We train the model for $1000$ iterations in each sub-time interval while $100000$ iterations in the first interval (i.e., cold-start training).

\paragraph{Wave1d-C.}
The equation is given by:
\begin{equation}
    \frac{\partial^2 u}{\partial t^2}-4\frac{\partial^2 u}{\partial x^2}= 0,
\end{equation}
defined on $\Omega\times T =[0,1] \times [0, 1]$, where $u=u(x, t)$ is the unknown. Let $\Omega' = \Omega\times T, x' = (x,t)$. The weak form is expressed as:
\begin{equation}
    -\int_{\Omega'} \frac{\partial u}{\partial t} \cdot \frac{\partial v}{\partial t} \diff{x'} + 4\int_{\Omega'} \frac{\partial u}{\partial x} \cdot \frac{\partial v}{\partial x} \diff{x'}= 0,
\end{equation}
where $v$ is the test function. We employ the FEniCS to discretize the problem with a mesh of size $100 \times 100$. Given that the matrix size remains within the memory constraints, we utilize a dense matrix implementation for faster matrix computations. The drop tolerance of the ILU is $10^{-3}$.

\paragraph{Wave2d-CG.}
The equation is given by:
\begin{equation}
    \frac{1}{c}\frac{\partial^2 u}{\partial t^2} - \Delta u = 0,
\end{equation}
define on $\Omega\times T =[-1, 1]^2 \times [0, 5]$, where $u=u(\vx,t)$ is the unknown and $c=c(\vx)$ is a parameter function with high frequencies, generated by the Gaussian random field. We solve this problem by an (implicit) time-stepping scheme (see Algorithm~\ref{alg:tm:2}). The number of sub-time intervals is $50$, with each interval having $5$ steps. We define the weak form to be:
\begin{equation}
    \int_{\Omega} u_1 \cdot v \diff{\vx} + \delta t^2\int_{\Omega} c\left(\nabla u_1 \cdot \nabla v \right) \diff{\vx} = \int_{\Omega} (2u_0 - u_{-1}) \cdot v \diff{\vx},
\end{equation}
where $u_{-1} = u_{-1}(\vx)$ is the solution at the time step before the previous time step, $u_0 = u_0(\vx)$ is the solution at the previous time step, $u_1 = u_1(\vx)$ is the solution at current time step, $v = v(\vx)$ is the test function, and $\delta t = 1/250$ is the time step length. We employ the FEniCS to discretize the problem with a mesh of size $40 \times 40$. Given that the matrix size remains within the memory constraints, we utilize a dense matrix implementation for faster matrix computations. The drop tolerance of the ILU is $10^{-4}$. We train the model for $1000$ iterations in each sub-time interval while $500000$ iterations in the first interval (i.e., cold-start training). 

\paragraph{Wave2d-MS.}
The equation is given by:
\begin{equation}
    \frac{\partial^2 u}{\partial t^2}+\nabla \cdot \left(\left(1, a^2\right) \odot \nabla u \right)= 0,
\end{equation}
defined on $\Omega\times T =[0,1]^2 \times [0, 100]$, where $u=u(\vx, t)$ is the unknown and $a$ is a given parameter. Let $\Omega' = \Omega\times T, \vx' = (\vx,t)$. The weak form is expressed as:
\begin{equation}
    \int_{\Omega'} \frac{\partial u}{\partial t} \cdot \frac{\partial v}{\partial t} \diff{\vx'} + \int_{\Omega'} \left(\left(1, a^2 \right) \odot \nabla u \right)\cdot \nabla v \diff{\vx'} = 0,
\end{equation}
where $v$ is the test function. We employ the FEniCS to discretize the problem with a mesh of size $10\times 10\times 1000$. Besides, we utilize a sparse matrix implementation since the matrix size exceeds the memory constraint. The drop tolerance of the ILU is $10^{-1}$. Finally, in this problem, we employ a Fourier MLP of $5$ layers with $128$ neurons in each layer as our neural model, where the Fourier features have a dimension of 128 and are sampled in $\mathcal{N}(0, \pi)$.

\paragraph{GS.}
The equation is given by:
\begin{equation}
\begin{aligned}
\frac{\partial u_1}{\partial t} & =  \varepsilon_1 \Delta u_1 + b (1 - u_1) - u_1 u_2^2, \\
\frac{\partial u_2}{\partial t} & =  \varepsilon_2 \Delta u_2 - d u_2 + u_1 u_2^2,
\end{aligned}
\end{equation}
defined on $\Omega\times T =[-1, 1]^2 \times [0, 200]$, where $\vu=(u_1(\vx,t), u_2(\vx,t))$ is the unknown and $b,d,\epsilon_1, \epsilon_2$ are given. We solve this problem by an (implicit) time-stepping scheme (see Algorithm~\ref{alg:tm:2}). The number of sub-time intervals is $200$, with each interval having $1$ step. The weak form is expressed as:
\begin{equation}
\begin{aligned}
     \int_{\Omega} \vu_1 \cdot \vv \diff{\vx}  + \delta t \int_{\Omega} \left( \epsilon_1 \nabla u_{1,1} \cdot \nabla v_1 + \epsilon_2 \nabla u_{1,2} \cdot \nabla v_2 \right) \diff{\vx}& \\
     + \delta t \int_{\Omega} \left( (u_{1,1}  u_{1,2}^2) \cdot v_1 -( u_{1,1} u_{1,2}^2)\cdot v_2)\right) \diff{\vx} & \\
     + \delta t \int_{\Omega} \left( -b (1 - u_{1,1}) \cdot v_1 + d u_{1,2} \cdot v_2)\right) \diff{\vx}
      &= \int_{\Omega} \vu_0 \cdot \vv \diff{\vx},
\end{aligned}
\end{equation}
where $\vu_0 = \vu_0(\vx)$ is the solution at the previous time step, $\vu_1 = \vu_1(\vx) = (u_{1,1}(\vx), u_{1,2}(\vx))$ is the solution at current time step, $\vv = \vv(\vx)$ is the test function, and $\delta t = 1/200$ is the time step length. We employ the FEniCS to discretize the problem with a mesh of size $128\times 128$. It is noted that we do not employ a Newton method to solve the discretized nonlinear equations since the time overhead is too high. Instead, we only precondition the linear portion (see Appendix~\ref{app:algo:tmnl}) and let the neural model find the correct solution by gradient descent. Besides, we utilize a sparse matrix implementation since the matrix size exceeds the memory constraint. The drop tolerance of the ILU is $10^{-1}$. We train the model for $1000$ iterations in each sub-time interval while $20000$ iterations in the first interval (i.e., cold-start training). Finally, in this problem, we employ an MLP of $5$ layers with $128$ neurons in each layer as our neural model.

\paragraph{KS.}
The equation is given by:
\begin{equation}
    \frac{\partial u}{\partial t} + \alpha u \frac{\partial u}{\partial x} + \beta \frac{\partial^2 u}{\partial x^2} + \gamma \frac{\partial^4 u}{\partial x^4} = 0,
\end{equation}
define on $\Omega\times T =[0, 2\pi] \times [0, 1]$, where $u=u(x,t)$ is the unknown and $\alpha,\beta,\gamma$ are multi-scale co-efficients. We solve this problem by an (implicit) time-stepping scheme (see Algorithm~\ref{alg:tm:2}). The number of sub-time intervals is $1$, with each interval having $250$ steps. We define the weak form to be:
\begin{equation}
    \int_{\Omega} u_1 v \diff{x} + \alpha \delta t\int_{\Omega} u_1\frac{\partial u_1}{\partial x} v \diff{x} - \beta \delta t\int_{\Omega} \frac{\partial u_1}{\partial x} \frac{\partial v}{\partial x} \diff{x} - \gamma \delta t\int_{\Omega} \frac{\partial^3 u_1}{\partial x^3} \frac{\partial v}{\partial x} \diff{x} = \int_{\Omega} u_0 v \diff{x},
\end{equation}
where $u_0 = u_0(\vx)$ is the solution at the previous time step, $u_1 = u_1(\vx)$ is the solution at current time step, $v = v(\vx)$ is the test function, and $\delta t = 1/250$ is the time step length. We employ the FEniCS to discretize the problem with a mesh of size $500$. It is noted that we do not employ a Newton method to solve the discretized nonlinear equations since the time overhead is too high. Instead, we only precondition the linear portion (see Appendix~\ref{app:algo:tmnl}) and let the neural model find the correct solution by gradient descent. Given that the matrix size remains within the memory constraints, we utilize a dense matrix implementation for faster matrix computations. The drop tolerance of the ILU is $10^{-4}$. We train the model for $15000$ iterations in each sub-time interval. Finally, in this problem, we employ an MLP of $5$ layers with $128$ neurons in each layer as our neural model.

\paragraph{Poisson Inverse Problem (PInv).}
The equation is given by:
\begin{equation}
    -\nabla( a \nabla u) = f,
\end{equation}
define on $\Omega=[0, 1]^2$, where $u=u(\vx)$ is the unknown solution, $a=a(\vx)$ denotes the unknown parameter function, and $f=f(\vx)$ is predefined. Given $2500$ uniformly distributed samples $\{ u(\vx^{(i)}) \}$ with Gaussian noise of $\mathcal{N}(0, 0.1)$, our target is to reconstruct the unknown solution $u$ and infer the unknown parameter function $a$. We define the weak form to be:
\begin{equation}
    \int_{\Omega} a(\nabla u \cdot \nabla v) \diff{\vx} = \int_{\Omega} f \cdot v \diff{\vx},
\end{equation}
where $v$ is the test function.  We employ the FEniCS to discretize the problem with a mesh of size $100\times 100$. Besides, we utilize a sparse matrix implementation. For fast speed, we employ the Jacobi preconditioner since the preconditioner needs updating every iteration. Finally, in this problem, we employ an MLP of $3$ layers with $64$ neurons in each layer for $u$ and an MLP of $5$ layers with $128$ neurons in each layer for $a$. The models are trained for $11000$ iterations, where $10000$ iterations are warm-up iterations. In warm-up iterations, only data loss is involved while physics loss is included in the rest of iterations. 

\paragraph{Heat Inverse Problem (HInv).}
The equation is given by:
\begin{equation}
    \frac{\partial u}{\partial t}-\nabla( a \nabla u) = f,
\end{equation}
define on $\Omega\times T =[-1, 1]^2 \times [0, 1]$, where $u=u(\vx, t)$ is the unknown solution, $a=a(\vx)$ denotes the unknown parameter function, and $f=f(\vx, t)$ is predefined. Given $2500$ uniformly distributed samples $\{ u(\vx^{(i)}, t^{(i)}) \}$ with Gaussian noise of $\mathcal{N}(0, 0.1)$, our target is to reconstruct the unknown solution $u$ and infer the unknown parameter function $a$. Let $\Omega' = \Omega\times T, \vx' = (\vx,t)$. We define the weak form to be:
\begin{equation}
    \int_{\Omega'} \frac{\partial u}{\partial t} \cdot v \diff{\vx'} + \int_{\Omega'} a(\nabla u \cdot \nabla v) \diff{\vx'} = \int_{\Omega'} f \cdot v \diff{\vx'},
\end{equation}
where $v$ is the test function.  We employ the FEniCS to discretize the problem with a mesh of size $40\times 40 \times 10$. Besides, we utilize a sparse matrix implementation. For fast speed, we employ the Jacobi preconditioner since the preconditioner needs updating every iteration. Finally, in this problem, we employ an MLP of $3$ layers with $64$ neurons in each layer for $u$ and an MLP of $3$ layers with $64$ neurons in each layer for $a$. The models are trained for $5000$ iterations, where $4000$ iterations are warm-up iterations. In warm-up iterations, only data loss is involved while physics loss is included in the rest of iterations.

\subsection{Experimental Results of Varying Preconditioner Precision}
\label{app:exp:forward:abla}
We provide the comprehensive results of the four Poisson problems in this subsection. Table~\ref{tab:abla} presents the convergence results of L2RE as well as some metrics to measure the precision of the preconditioner for different cases. For example, ``$\mP^{-1}f$ Error'' measures the L2RE between the $\mP^{-1}f$ and the $\mA^{-1}f$. Besides, Figure~\ref{fig:exp_relationship_2} shows the convergence history of different cases. We can find that although preconditioning (ILU) cannot ensure that the condition number decreases, it can often promote convergence.

\begin{table}[t]
\centering
\renewcommand{\arraystretch}{1.2}
\caption{Comprehensive results of varying preconditioner precisions.}
\vspace{0.1in}
\resizebox{0.8\linewidth}{!}{
\begin{scriptsize}
\begin{tabular}{ll|llll|l}
\hline
\multicolumn{2}{c|}{Poisson}      & \multicolumn{4}{c|}{Drop Tolerance}         & \multirow{2}{*}{No Preconditioner} \\ \cline{3-6}
\multicolumn{2}{c|}{}             & \textbf{1.00e-4} & \textbf{1.00e-3} & \textbf{1.00e-2} & \textbf{1.00e-1} &                                  \\ \hline
\multirow{3}{*}{2d-C}  & L2RE     & 1.70e-3 & 2.74e-3 & 4.07e-3 & 2.18e-3 & 3.54e-2                          \\
                       & Cond     & 1.10e+0 & 2.82e+0 & 1.52e+1 & 6.03e+1 & 1.13e+2                          \\
                       & $\mP^{-1}f$ Error & 2.04e-2 & 2.08e-1 & 5.51e-1 & 7.67e-1 &   --                              \\ \hline
\multirow{3}{*}{2d-CG} & L2RE     & 5.38e-3 & 7.87e-3 & 4.27e-3 & 4.36e-3 & 3.86e-3                          \\
                       & Cond     & 1.01e+0 & 1.19e+0 & 2.55e+0 & 7.22e+0 & 1.27e+1                          \\
                       & $\mP^{-1}f$ Error & 2.84e-3 & 4.05e-2 & 3.50e-1 & 7.00e-1 &  --                                \\ \hline
\multirow{3}{*}{3d-CG} & L2RE     & 4.18e-2 & 4.11e-2 & 4.11e-2 & 4.23e-2 & 4.19e-2                          \\
                       & Cond     & 6.77e+0 & 1.17e+0 & 1.38e+0 & 1.77e+0 & 2.20e+0                          \\
                       & $\mP^{-1}f$ Error & 4.63e-1 & 2.05e-1 & 5.84e-1 & 8.73e-1 & --                                 \\ \hline
\multirow{3}{*}{2d-MS} & L2RE     & 6.48e-2 & 6.38e-2 & 6.37e-1 & 7.06e-1 & 8.55e-1                          \\
                       & Cond     & 3.23e+0 & 3.25e+1 & 2.47e+2 & 3.42e+2 & 3.39e+0                          \\
                       & $\mP^{-1}f$ Error & 3.74e-1 & 6.42e-1 & 8.13e-1 & 9.58e-1 &  --                                \\ \hline
\end{tabular}
\end{scriptsize}
}
\label{tab:abla}
\end{table}

\begin{figure}[bt]
     \centering
     \begin{subfigure}[b]{0.32\textwidth}
         \centering
         \includegraphics[height=0.7\textwidth]{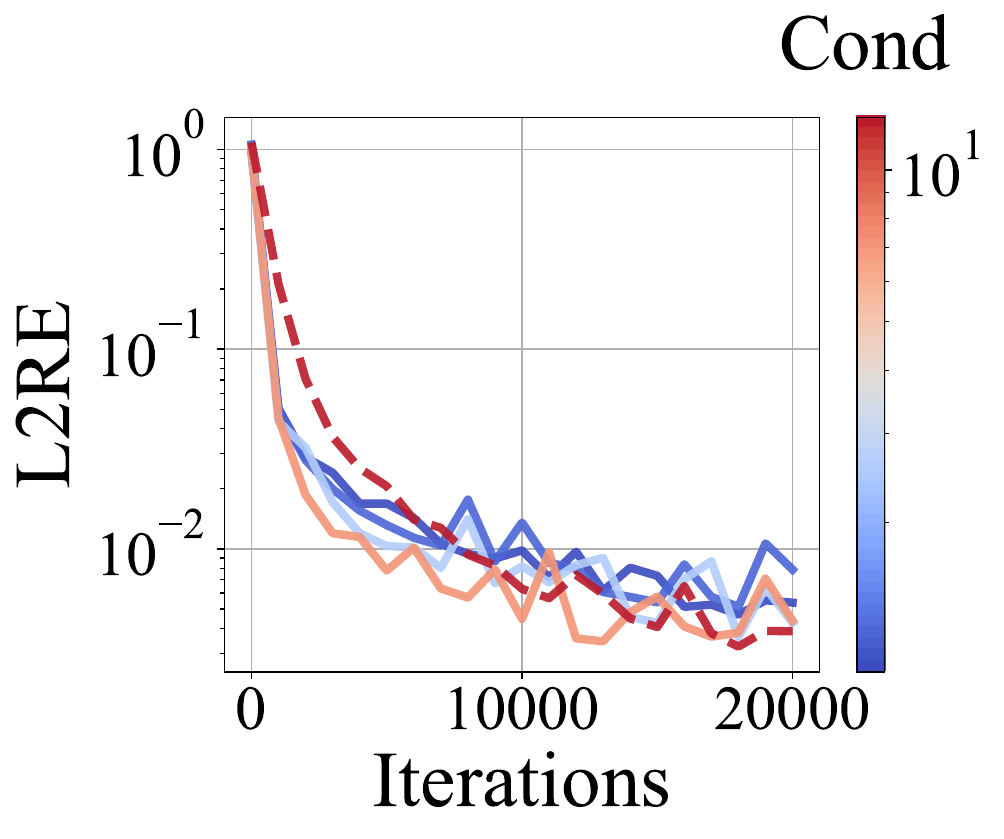}
         \caption{Poisson2d-CG}
     \end{subfigure}
     \begin{subfigure}[b]{0.32\textwidth}
         \centering
         \includegraphics[height=0.7\textwidth]{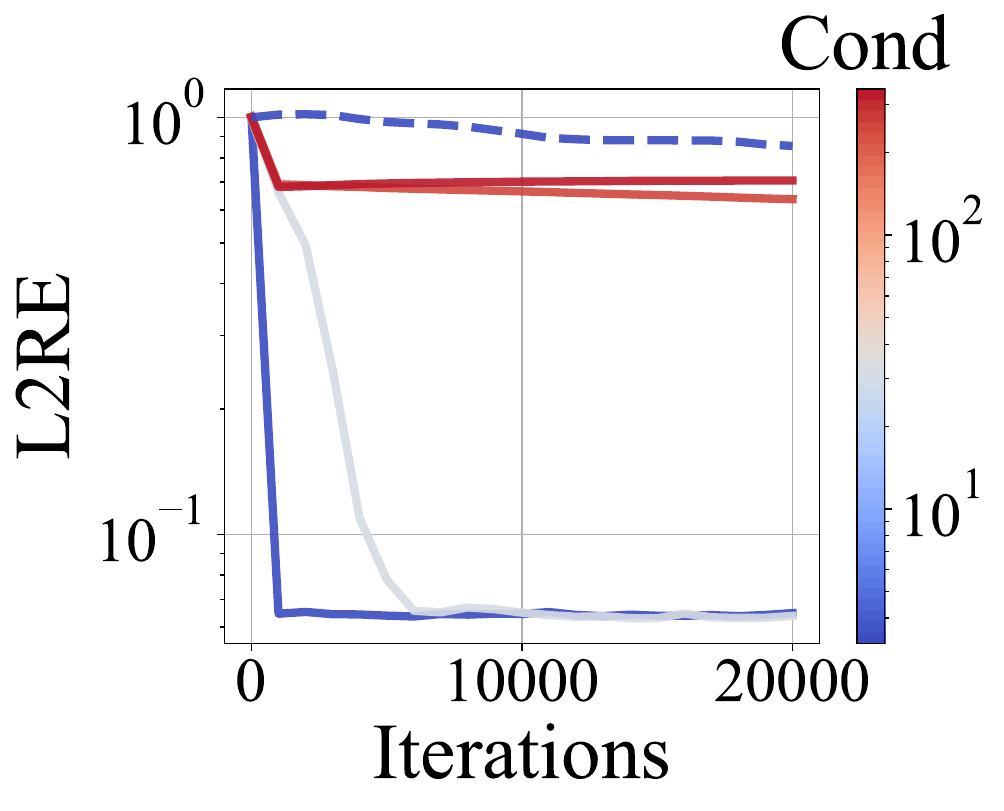}
         \caption{Poisson2d-MS}
     \end{subfigure}
     \begin{subfigure}[b]{0.32\textwidth}
         \centering
         \includegraphics[height=0.7\textwidth]{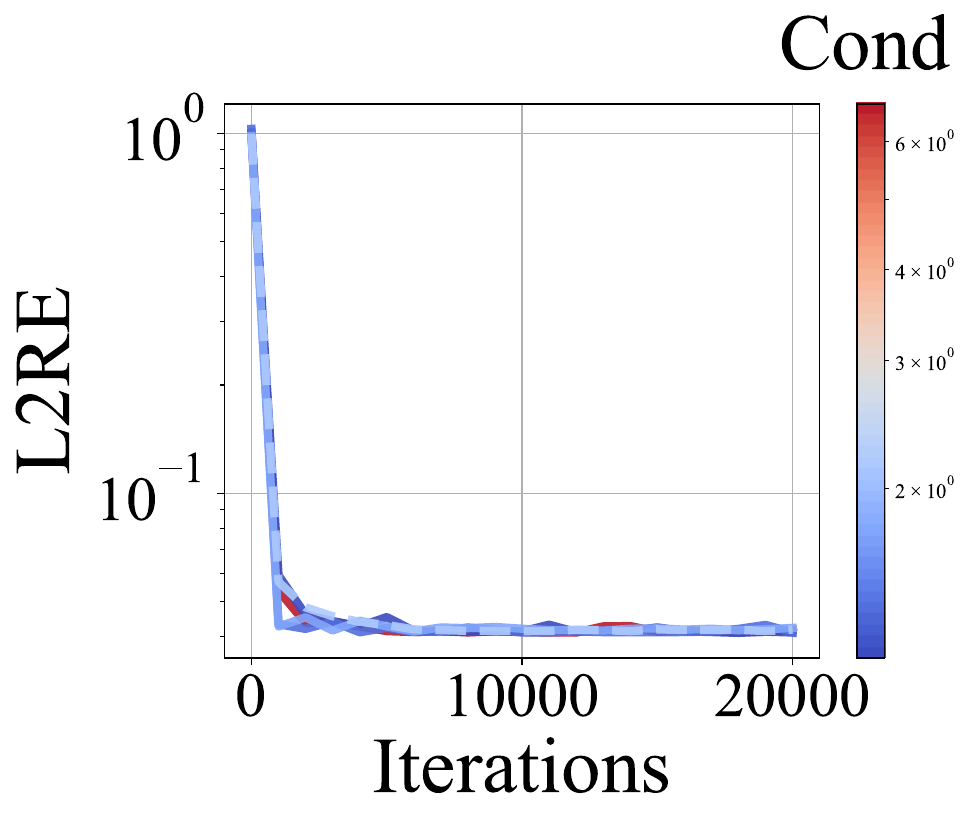}
         \caption{Poisson3d-CG}
     \end{subfigure}
\caption{The training L2 relative error (L2RE) in ablation study. The dashed line marks the trajectory corresponding to the one without the preconditioner.}
\label{fig:exp_relationship_2}
\end{figure}

\subsection{Ablation Study}
\label{app:exp:real:abla}
We perform extensive ablation studies for the forward benchmark problems.

\paragraph{More Random Trials.}
In Table~\ref{tab:4}, we have re-evaluated all experiments of the forward problems using 10 random trials. To succinctly demonstrate the consistency and reliability of our findings, we compared the outcomes of the 5-trial (our choice for main results) and 10-trial experiments. Our findings show that the results from the 10-trial evaluations align closely with those from the original 5-trial tests, indicating that our initial conclusions are consistent and reliable. Moreover, the comparison with the state-of-the-art (SOTA) baseline methods remains unchanged, affirming the robustness of our approach.

\paragraph{Different Preconditioning Methods.}
In Table~\ref{tab:5}, we have tested other matrix preconditioning methods on two selected problems, Poisson2d-MS and Wave2d-MS over three random trials. The results indicate that the ILU preconditioning method, which we employ in our approach, demonstrates greater stability and effectiveness in comparison to the Row Balancing and Diagonal methods. This evidence supports our choice of ILU as a superior option for the problems we address. 

\paragraph{Initialization Methods and Network Hyperparameters.}
In Table~\ref{tab:6}, \ref{tab:7}, \ref{tab:8}, \ref{tab:9}, and \ref{tab:10},  we have conducted additional studies on the impact of various initialization schemes and hyperparameters. These additional analyses strengthen our confidence in the robustness and reliability of our proposed method. The sensitivity to initialization schemes and hyperparameters is minimal, indicating that our approach is adaptable and stable across different settings. This aspect is critical for the practical application of our method in diverse problem contexts.

\begin{table}[t]
\centering
\renewcommand{\arraystretch}{1.2}
\caption{Results for 10 random trials.}
\label{tab:4}
\vspace{0.1in}
\resizebox{0.8\linewidth}{!}{
\begin{scriptsize}
\begin{tabular}{l|lll}
\hline
L2RE (mean ± std) & 5 Random Samples  & 10 Random Samples & Best Baseline     \\ \hline
Burgers1d-C       & 1.42e-2 ± 1.62e-4 & 1.41e-2 ± 2.16e-4 & 1.43e-2 ± 1.44e-3 \\
Burgers2d-C       & 5.23e-1 ± 7.52e-2 & 4.90e-1 ± 2.94e-2 & 2.60e-1 ± 5.78e-3 \\
Poisson2d-C       & 3.98e-3 ± 3.70e-3 & 1.84e-3 ± 9.18e-4 & 1.23e-2 ± 7.37e-3 \\
Poisson2d-CG      & 5.07e-3 ± 1.93e-3 & 5.04e-3 ± 1.53e-3 & 1.43e-2 ± 4.31e-3 \\
Poisson3d-CG      & 4.16e-2 ± 7.53e-4 & 4.13e-2 ± 5.08e-4 & 1.02e-1 ± 3.16e-2 \\
Poisson2d-MS      & 6.40e-2 ± 1.12e-3 & 6.42e-2 ± 7.62e-4 & 5.90e-1 ± 4.06e-2 \\
Heat2d-VC         & 3.11e-2 ± 6.17e-3 & 2.61e-2 ± 3.74e-3 & 2.12e-1 ± 8.61e-4 \\
Heat2d-MS         & 2.84e-2 ± 1.30e-2 & 2.07e-2 ± 6.52e-3 & 4.40e-2 ± 4.81e-3 \\
Heat2d-CG         & 1.50e-2 ± 1.17e-4 & 1.55e-2 ± 5.37e-4 & 2.39e-2 ± 1.39e-3 \\
Heat2d-LT         & 2.11e-1 ± 1.00e-2 & 1.87e-1 ± 8.41e-3 & 9.99e-1 ± 1.05e-5 \\
NS2d-C            & 1.28e-2 ± 2.44e-3 & 1.21e-2 ± 2.53e-3 & 3.60e-2 ± 3.87e-3 \\
NS2d-CG           & 6.62e-2 ± 1.26e-3 & 6.36e-2 ± 2.21e-3 & 8.24e-2 ± 8.21e-3 \\
NS2d-LT           & 9.09e-1 ± 4.00e-4 & 9.09e-1 ± 9.00e-4 & 9.95e-1 ± 7.19e-4 \\
Wave1d-C          & 1.28e-2 ± 1.20e-4 & 1.28e-2 ± 1.55e-4 & 9.79e-2 ± 7.72e-3 \\
Wave2d-CG         & 5.85e-1 ± 9.05e-3 & 5.48e-1 ± 8.69e-3 & 7.94e-1 ± 9.33e-3 \\
Wave2d-MS         & 5.71e-2 ± 5.68e-3 & 6.07e-2 ± 8.20e-3 & 9.82e-1 ± 1.23e-3 \\
GS                & 1.44e-2 ± 2.53e-3 & 1.44e-2 ± 3.10e-3 & 7.99e-2 ± 1.69e-2 \\
KS                & 9.52e-1 ± 2.94e-3 & 9.52e-1 ± 3.03e-3 & 9.57e-1 ± 2.85e-3 \\
\hline
\end{tabular}
\end{scriptsize}
}
\end{table}

\begin{table}[t]
\centering
\renewcommand{\arraystretch}{1.2}
\caption{Different matrix preconditioning methods, 3 random trials.}
\label{tab:5}
\vspace{0.1in}
\resizebox{0.8\linewidth}{!}{
\begin{scriptsize}
\begin{tabular}{l|lll}
\hline
L2RE (mean ± std) & Row Balancing     & Diagonal          & ILU               \\ \hline
Poisson2d-MS      & 6.27e-1 ± 7.23e-2 & 6.27e-1 ± 7.23e-2 & 6.34e-2 ± 1.63e-4 \\
Wave2d-MS         & 6.12e-2 ± 8.16e-4 & 6.12e-2 ± 8.16e-4 & 5.76e-2 ± 1.06e-3 \\
\hline
\end{tabular}
\end{scriptsize}
}
\end{table}

\begin{table}[t]
\centering
\renewcommand{\arraystretch}{1.2}
\caption{Different initialization methods, 3 random trials.}
\vspace{0.1in}
\label{tab:6}
\resizebox{0.8\linewidth}{!}{
\begin{scriptsize}
\begin{tabular}{l|llll}
\hline
L2RE (mean ± std) & Glorot Uniform    & Glorot Normal     & He Normal         & He Uniform        \\ \hline
Poisson2d-MS      & 6.37e-2 ± 4.71e-5 & 6.38e-2 ± 1.63e-4 & 6.38e-2 ± 1.25e-4 & 6.39e-2 ± 1.25e-4 \\
NS2d-C            & 1.35e-2 ± 1.33e-3 & 1.36e-2 ± 2.73e-3 & 1.63e-2 ± 2.15e-3 & 1.78e-2 ± 5.90e-3 \\
Wave2d-MS         & 5.71e-2 ± 1.77e-3 & 6.03e-2 ± 3.04e-3 & 5.58e-2 ± 2.92e-3 & 5.43e-2 ± 5.11e-3 \\
\hline
\end{tabular}
\end{scriptsize}
}
\end{table}

\begin{table}[t]
\centering
\renewcommand{\arraystretch}{1.2}
\caption{Different learning rates (Adam optimizer: $\beta_1=0.9, \beta_2=0.999$), the problem is poisson2d-MS, 3 random trials.}
\label{tab:7}
\vspace{0.1in}
\resizebox{0.8\linewidth}{!}{
\begin{scriptsize}
\begin{tabular}{l|llll}
\hline
Metric (mean ± std) & $\eta=1\times 10^{-4}$ & $\eta=3\times 10^{-4}$ & $\eta=1\times 10^{-3}$ & $\eta=3\times 10^{-3}$ \\ \hline
MAE                 & 8.37e-2 ± 5.89e-4      & 8.40e-2 ± 8.52e-4      & 8.57e-2 ± 3.28e-3      & 8.56e-2 ± 4.66e-3      \\
MSE                 & 2.71e-2 ± 2.36e-4      & 2.72e-2 ± 2.05e-4      & 2.75e-2 ± 1.36e-3      & 2.75e-2 ± 1.11e-3      \\
L1RE                & 4.72e-2 ± 3.40e-4      & 4.74e-2 ± 4.97e-4      & 4.83e-2 ± 1.89e-3      & 4.83e-2 ± 2.65e-3      \\
L2RE                & 6.34e-2 ± 2.83e-4      & 6.36e-2 ± 2.49e-4      & 6.39e-2 ± 1.53e-3      & 6.39e-2 ± 1.28e-3    \\
\hline
\end{tabular}
\end{scriptsize}
}
\end{table}

\begin{table}[t]
\centering
\renewcommand{\arraystretch}{1.2}
\caption{Different Adam betas (Adam optimizer, $\eta=1\times 10^{-3}$), the problem is Poisson2d-MS, 3 random trials.}
\label{tab:8}
\vspace{0.1in}
\resizebox{1\linewidth}{!}{
\begin{scriptsize}
\begin{tabular}{l|lllll}
\hline
Metric (mean ± std) & (0.9,0.9)         & (0.9,0.99)        & (0.9,0.999)       & (0.99,0.99)       & (0.99,0.999)      \\ \hline
MAE                 & 8.45e-2 ± 8.18e-4 & 8.49e-2 ± 1.25e-3 & 8.57e-2 ± 3.28e-3 & 8.34e-2 ± 2.87e-4 & 8.39e-2 ± 3.86e-4 \\
MSE                 & 2.74e-2 ± 4.64e-4 & 2.76e-2 ± 5.25e-4 & 2.75e-2 ± 1.36e-3 & 2.75e-2 ± 8.16e-5 & 2.77e-2 ± 9.43e-5 \\
L1RE                & 4.76e-2 ± 4.50e-4 & 4.79e-2 ± 7.26e-4 & 4.83e-2 ± 1.89e-3 & 4.71e-2 ± 1.63e-4 & 4.73e-2 ± 2.16e-4 \\
L2RE                & 6.37e-2 ± 5.56e-4 & 6.39e-2 ± 6.18e-4 & 6.39e-2 ± 1.53e-3 & 6.39e-2 ± 1.25e-4 & 6.41e-2 ± 9.43e-5 \\
\hline
\end{tabular}
\end{scriptsize}
}
\end{table}

\begin{table}[t]
\centering
\renewcommand{\arraystretch}{1.2}
\caption{Different number of hidden neural neurons in each layer (the number of hidden layers is 5), the problem is Poisson2d-MS, 3 random trails.}
\label{tab:9}
\vspace{0.1in}
\resizebox{1\linewidth}{!}{
\begin{scriptsize}
\begin{tabular}{l|lllll}
\hline
Metric (mean ± std) & 32                & 64                & 128               & 256               & 512               \\ \hline
MAE                 & 8.42e-2 ± 3.77e-4 & 8.38e-2 ± 2.36e-4 & 8.60e-2 ± 3.07e-3 & 8.84e-2 ± 2.05e-3 & 8.49e-2 ± 8.01e-4 \\
MSE                 & 2.72e-2 ± 1.89e-4 & 2.73e-2 ± 2.94e-4 & 2.80e-2 ± 1.01e-3 & 2.90e-2 ± 8.38e-4 & 2.75e-2 ± 1.89e-4 \\
L1RE                & 4.75e-2 ± 2.16e-4 & 4.73e-2 ± 1.41e-4 & 4.85e-2 ± 1.75e-3 & 4.99e-2 ± 1.13e-3 & 4.79e-2 ± 4.50e-4 \\
L2RE                & 6.36e-2 ± 2.36e-4 & 6.36e-2 ± 3.30e-4 & 6.44e-2 ± 1.16e-3 & 6.56e-2 ± 9.63e-4 & 6.38e-2 ± 2.36e-4 \\
\hline
\end{tabular}
\end{scriptsize}
}
\end{table}

\begin{table}[t]
\centering
\renewcommand{\arraystretch}{1.2}
\caption{Different number of hidden layers (the number of hidden neural neurons in each layer is 128), the problem is Poisson2d-MS, 3 random trails.}
\label{tab:10}
\vspace{0.1in}
\resizebox{1\linewidth}{!}{
\begin{scriptsize}
\begin{tabular}{l|lllll}
\hline
Metric (mean ± std) & 3                 & 4                 & 5                 & 6                 & 7                 \\ \hline
MAE                 & 8.39e-2 ± 6.55e-4 & 8.37e-2 ± 8.29e-4 & 8.84e-2 ± 2.05e-3 & 8.21e-2 ± 4.64e-4 & 8.43e-2 ± 4.50e-4 \\
MSE                 & 2.72e-2 ± 1.41e-4 & 2.70e-2 ± 2.87e-4 & 2.90e-2 ± 8.38e-4 & 2.56e-2 ± 2.36e-4 & 2.73e-2 ± 4.71e-5 \\
L1RE                & 4.74e-2 ± 3.68e-4 & 4.72e-2 ± 4.64e-4 & 4.99e-2 ± 1.13e-3 & 4.63e-2 ± 2.49e-4 & 4.75e-2 ± 2.49e-4 \\
L2RE                & 6.35e-2 ± 1.41e-4 & 6.33e-2 ± 2.87e-4 & 6.56e-2 ± 9.63e-4 & 6.17e-2 ± 3.30e-4 & 6.36e-2 ± 9.43e-5 \\
\hline
\end{tabular}
\end{scriptsize}
}
\end{table}

\subsection{Benchmark of Inverse Problems}\label{app:exp:inverse}
Here, we consider two inverse problems, the Poisson Inverse Problem (PInv) and Heat Inverse Problem (HInv), from the benchmark \cite{hao2022physics}. In such problems, our target is to reconstruct the unknown solution from $2500$ noisy samples and infer the unknown parameter function. We compare our method with the SOTA PINN baseline in \citet{hao2022physics} and the traditional adjoint method designed for PDE-constrained optimization. We report the results in Table~\ref{tb:inv}.

From the results, we can conclude that our method achieves state-of-the-art performance in both accuracy and running time. Although the adjoint method converges very fast, it fails to approach the correct solution. This is because the numerical method does not impose any continuous prior on the ansatz and can overfit the noise in the solution samples.

\begin{table}[t]
\centering
\caption{Comparison between our method, SOTA PINN baseline, and the adjoint method over 5 trials. The best results are in \textbf{bold}.}\label{tb:inv}
\vspace{0.1in}
\begin{tabular}{l|lll|lll}
\hline
\multirow{2}{*}{Problem} & \multicolumn{3}{c|}{L2RE (mean ± std)}                                                                  & \multicolumn{3}{c}{Average Running Time (s)}                          \\ \cline{2-7} 
                         & \multicolumn{1}{l|}{Ours}                  & \multicolumn{1}{l|}{SOTA}              & Adjoint           & \multicolumn{1}{l|}{Ours}    & \multicolumn{1}{l|}{SOTA}    & Adjoint \\ \hline
PInv                     & \multicolumn{1}{l|}{\textbf{1.80e-2 ± 9.30e-3}} & \multicolumn{1}{l|}{2.45e-2 ± 1.03e-2} & 7.82e+2 ± 0.00e+0 & \multicolumn{1}{l|}{1.87e+2} & \multicolumn{1}{l|}{4.90e+2} & 1.40e+0 \\ 
HInv                     & \multicolumn{1}{l|}{\textbf{9.04e-3 ± 2.34e-3}}     & \multicolumn{1}{l|}{5.09e-2 ± 4.34e-3} & 1.50e+3 ± 0.00e+0 & \multicolumn{1}{l|}{3.21e+2} & \multicolumn{1}{l|}{3.39e+3} & 1.07e+1 \\ \hline
\end{tabular}
\end{table}

\section{Supplementary Experimental Results}\label{app:exp:forward:res}

In Table~\ref{tab:12}, \ref{tab:13}, and \ref{tab:14}, we display the detailed experiment results in different metrics, including L2RE, L1RE, MSE, and the standard deviation of these metrics over 5 runs.

\clearpage
\eject \pdfpagewidth=45cm \pdfpageheight=22cm
\thispagestyle{empty}

\begin{table}[!ht]
\centering
\begin{minipage}{0.85\pdfpagewidth}
\tiny
\renewcommand{\arraystretch}{1.7}
\caption{Mean (std) of L2RE for main experiments.}
\label{tab:12}
  \resizebox{\linewidth}{!}{
\begin{tabular}{cc|cccccccccccc}
\hline
L2RE & Name & \multicolumn{1}{c|}{\multirow{2}{*}{Ours}} & \multicolumn{2}{c|}{Vanilla} & \multicolumn{3}{c|}{Loss Reweighting/Sampling} & \multicolumn{1}{c|}{Optimizer} & \multicolumn{2}{c|}{Loss functions} & \multicolumn{3}{c}{Architecture} \\ \cline{4-14} 
-- &  & \multicolumn{1}{c|}{} & PINN & \multicolumn{1}{c|}{PINN-w} & LRA & NTK & \multicolumn{1}{c|}{RAR} & \multicolumn{1}{c|}{MultiAdam} & gPINN & \multicolumn{1}{c|}{vPINN} & LAAF & GAAF & FBPINN \\ \hline
\multirow{2}{*}{Burgers} & 1d-C & \textbf{1.42E-2(1.62E-4)} & 1.45E-2(1.59E-3) & 2.63E-2(4.68E-3) & 2.61E-2(1.18E-2) & 1.84E-2(3.66E-3) & 3.32E-2(2.14E-2) & 4.85E-2(1.61E-2) & 2.16E-1(3.34E-2) & 3.47E-1(3.49E-2) & \textbf{1.43E-2(1.44E-3)} & 5.20E-2(2.08E-2) & 2.32E-1(9.14E-2) \\
 & 2d-C & 5.23E-1(7.52E-2) & 3.24E-1(7.54E-4) & 2.70E-1(3.93E-3) & \textbf{2.60E-1(5.78E-3)} & 2.75E-1(4.78E-3) & 3.45E-1(4.56E-5) & 3.33E-1(8.65E-3) & 3.27E-1(1.25E-4) & 6.38E-1(1.47E-2) & 2.77E-1(1.39E-2) & 2.95E-1(1.17E-2) & -- \\ \hline
\multirow{4}{*}{Poisson} & 2d-C & \textbf{3.98E-3(3.70E-3)} & 6.94E-1(8.78E-3) & 3.49E-2(6.91E-3) & 1.17E-1(1.26E-1) & 1.23E-2(7.37E-3) & 6.99E-1(7.46E-3) & 2.63E-2(6.57E-3) & 6.87E-1(1.87E-2) & 4.91E-1(1.55E-2) & 7.68E-1(4.70E-2) & 6.04E-1(7.52E-2) & 4.49E-2(7.91E-3) \\
 & 2d-CG & \textbf{5.07E-3(1.93E-3)} & 6.36E-1(2.57E-3) & 6.08E-2(4.88E-3) & 4.34E-2(7.95E-3) & 1.43E-2(4.31E-3) & 6.48E-1(7.87E-3) & 2.76E-1(1.03E-1) & 7.92E-1(4.56E-3) & 2.86E-1(2.00E-3) & 4.80E-1(1.43E-2) & 8.71E-1(2.67E-1) & 2.90E-2(3.92E-3) \\
 & 3d-CG & \textbf{4.16E-2(7.53E-4)} & 5.60E-1(2.84E-2) & 3.74E-1(3.23E-2)            & 1.02E-1(3.16E-2)                           & 9.47E-1(4.94E-4)                           & 5.76E-1(5.40E-2)         & 3.63E-1(7.81E-2)                           & 4.85E-1(5.70E-2) & 7.38E-1(6.47E-4)                           & 5.79E-1(2.65E-2)                           & 5.02E-1(7.47E-2)                           & 7.39E-1(7.24E-2)                           \\
 & 2d-MS & \textbf{6.40E-2(1.12E-3)} & 6.30E-1(1.07E-2) & 7.60E-1(6.96E-3) & 7.94E-1(6.51E-2) & 7.48E-1(9.94E-3) & 6.44E-1(2.13E-2) & 5.90E-1(4.06E-2) & 6.16E-1(1.74E-2) & 9.72E-1(2.23E-2) & 5.93E-1(1.18E-1) & 9.31E-1(7.12E-2) & 1.04E+0(6.13E-5) \\ \hline
Heat & 2d-VC & \textbf{3.11E-2(6.17E-3)} & 1.01E+0(6.34E-2) & 2.35E-1(1.70E-2) & 2.12E-1(8.61E-4) & 2.14E-1(5.82E-3) & 9.66E-1(1.86E-2) & 4.75E-1(8.44E-2) & 2.12E+0(5.51E-1) & 9.40E-1(1.73E-1) & 6.42E-1(6.32E-2) & 8.49E-1(1.06E-1) & 9.52E-1(2.29E-3) \\
 & 2d-MS & \textbf{2.84E-2(1.30E-2)} & 6.21E-2(1.38E-2) & 2.42E-1(2.67E-2) & 8.79E-2(2.56E-2) & 4.40E-2(4.81E-3) & 7.49E-2(1.05E-2) & 2.18E-1(9.26E-2) & 1.13E-1(3.08E-3) & 9.30E-1(2.06E-2) & 7.40E-2(1.92E-2) & 9.85E-1(1.04E-1) & 8.20E-2(4.87E-3) \\
 & 2d-CG & \textbf{1.50E-2(1.17E-4)}& 3.64E-2(8.82E-3) & 1.45E-1(4.77E-3) & 1.25E-1(4.30E-3) & 1.16E-1(1.21E-2) & 2.72E-2(3.22E-3) & 7.12E-2(1.30E-2) & 9.38E-2(1.45E-2) & 1.67E+0(3.62E-3) & 2.39E-2(1.39E-3) & 4.61E-1(2.63E-1) & 9.16E-2(3.29E-2) \\
 & 2d-LT & \textbf{2.11E-1(1.00E-2)} & 9.99E-1(1.05E-5) & 9.99E-1(8.01E-5) & 9.99E-1(7.37E-5) & 1.00E+0(2.82E-4) & 9.99E-1(1.56E-4) & 1.00E+0(3.85E-5) & 1.00E+0(9.82E-5) & 1.00E+0(0.00E+0) & 9.99E-1(4.49E-4) & 9.99E-1(2.20E-4) & 1.01E+0(1.23E-4) \\ \hline
NS & 2d-C & \textbf{1.28E-2(2.44E-3)} & 4.70E-2(1.12E-3) & 1.45E-1(1.21E-2) & NA & 1.98E-1(2.60E-2) & 4.69E-1(1.16E-2) & 7.27E-1(1.95E-1) & 7.70E-2(2.99E-3) & 2.92E-1(8.24E-2) & 3.60E-2(3.87E-3) & 3.79E-2(4.32E-3) & 8.45E-2(2.26E-2) \\
 & 2d-CG & \textbf{6.62E-2(1.26E-3)} & 1.19E-1(5.46E-3) & 3.26E-1(7.69E-3) & 3.32E-1(7.60E-3) & 2.93E-1(2.02E-2) & 3.34E-1(6.52E-4) & 4.31E-1(6.95E-2) & 1.54E-1(5.89E-3) & 9.94E-1(3.80E-3) & 8.24E-2(8.21E-3) & 1.74E-1(7.00E-2) & 8.27E+0(3.68E-5) \\
 & 2d-LT & \textbf{9.09E-1(4.00E-4)} & 9.96E-1(1.19E-3) & 1.00E+0(3.34E-4) & 1.00E+0(4.05E-4) & 9.99E-1(6.04E-4) & 1.00E+0(3.35E-4) & 1.00E+0(2.19E-4) & 9.95E-1(7.19E-4) & 1.73E+0(1.00E-5) & 9.98E-1(3.42E-3) & 9.99E-1(1.10E-3) & 1.00E+0(2.07E-3) \\ \hline
Wave & 1d-C & \textbf{1.28E-2(1.20E-4)}& 5.88E-1(9.63E-2) & 2.85E-1(8.97E-3) & 3.61E-1(1.95E-2) & 9.79E-2(7.72E-3) & 5.39E-1(1.77E-2) & 1.21E-1(1.76E-2) & 5.56E-1(1.67E-2) & 8.39E-1(5.94E-2) & 4.54E-1(1.08E-2) & 6.77E-1(1.05E-1) & 5.91E-1(4.74E-2) \\
 & 2d-CG & \textbf{5.85E-1(9.05E-3)} & 1.84E+0(3.40E-1) & 1.66E+0(7.39E-2) & 1.48E+0(1.03E-1) & 2.16E+0(1.01E-1) & 1.15E+0(1.06E-1) & 1.09E+0(1.24E-1) & 8.14E-1(1.18E-2) & 7.99E-1(4.31E-2) & 8.19E-1(2.67E-2) & 7.94E-1(9.33E-3) & 1.06E+0(7.54E-2) \\
 & 2d-MS & \textbf{5.71E-2(5.68E-3)} &  1.34E+0(2.34E-1) & 1.02E+0(1.16E-2)            & 1.02E+0(1.36E-2)                           & 1.04E+0(3.11E-2)                           & 1.35E+0(2.43E-1)         & 1.01E+0(5.64E-3)                           & 1.02E+0(4.00E-3) & 9.82E-1(1.23E-3)                           & 1.06E+0(1.71E-2)                           & 1.06E+0(5.35E-2)                           & 1.03E+0(6.68E-3)                           \\ \hline
Chaotic & GS & \textbf{1.44E-2(2.53E-3)}& 3.19E-1(3.18E-1) & 1.58E-1(9.10E-2) & 9.37E-2(4.42E-5) & 2.16E-1(7.73E-2) & 9.46E-2(9.46E-4) & 9.37E-2(1.21E-5) & 2.48E-1(1.10E-1) & 1.16E+0(1.43E-1) & 9.47E-2(7.07E-5) & 9.46E-2(1.15E-4) & 7.99E-2(1.69E-2) \\
 & KS & \textbf{9.52E-1(2.94E-3)} & 1.01E+0(1.28E-3) & 9.86E-1(2.24E-2) & 9.57E-1(2.85E-3) & 9.64E-1(4.94E-3) & 1.01E+0(8.63E-4) & 9.61E-1(4.77E-3) & 9.94E-1(3.83E-3) & 9.72E-1(5.80E-4) & 1.01E+0(2.12E-3) & 1.00E+0(1.24E-2) & 1.02E+0(2.31E-2) \\ \hline
\end{tabular}
}

\label{main-l2re}
\end{minipage}
\end{table}

\clearpage

\eject \pdfpagewidth=8.5in \pdfpageheight=11in

\clearpage
\eject \pdfpagewidth=45cm \pdfpageheight=22cm
\thispagestyle{empty}

\begin{table}[!ht]
\centering
\begin{minipage}{0.85\pdfpagewidth}
\tiny
\renewcommand{\arraystretch}{1.7}
\caption{Mean (std) of L1RE for main experiments.}
\label{tab:13}
  \resizebox{\linewidth}{!}{
\begin{tabular}{cc|cccccccccccc}
\hline
L1RE & Name & \multicolumn{1}{c|}{\multirow{2}{*}{Ours}} & \multicolumn{2}{c|}{Vanilla} & \multicolumn{3}{c|}{Loss Reweighting/Sampling} & \multicolumn{1}{c|}{Optimizer} & \multicolumn{2}{c|}{Loss functions} & \multicolumn{3}{c}{Architecture} \\ \cline{4-14} 
-- &  & \multicolumn{1}{c|}{} & PINN & \multicolumn{1}{c|}{PINN-w} & LRA & NTK & \multicolumn{1}{c|}{RAR} & \multicolumn{1}{c|}{MultiAdam} & gPINN & \multicolumn{1}{c|}{vPINN} & LAAF & GAAF & FBPINN \\ \hline
\multirow{2}{*}{Burgers} & 1d-C & \textbf{9.05E-3(1.45E-4)} & 9.55E-3(6.42E-4) & 1.88E-2(4.05E-3) & 1.35E-2(2.57E-3) & 1.30E-2(1.73E-3) & 1.35E-2(4.66E-3) & 2.64E-2(5.69E-3) & 1.42E-1(1.98E-2) & 4.02E-2(6.41E-3) & 1.40E-2(3.68E-3) & 1.95E-2(8.30E-3) & 3.75E-2(9.70E-3) \\
 & 2d-C & 4.14E-1(2.24E-2) & 2.96E-1(7.40E-4) & 2.43E-1(2.98E-3) & \textbf{2.31E-1(7.16E-3)} & 2.48E-1(5.33E-3) & 3.27E-1(3.73E-5) & 3.12E-1(1.15E-2) & 3.01E-1(3.55E-4) & 6.56E-1(3.01E-2) & 2.57E-1(2.06E-2) & 2.67E-1(1.22E-2) & -- \\ \hline
\multirow{4}{*}{Poisson} & 2d-C & \textbf{4.43E-3(4.69E-3)} & 7.40E-1(5.49E-3) & 3.08E-2(5.13E-3) & 7.82E-2(7.47E-2) & 1.30E-2(8.23E-3) & 7.48E-1(1.01E-2) & 2.47E-2(6.38E-3) & 7.35E-1(2.08E-2) & 4.60E-1(1.39E-2) & 7.67E-1(1.36E-2) & 6.57E-1(3.99E-2) & 5.01E-2(4.71E-3) \\
 & 2d-CG & \textbf{4.76E-3(1.92E-3)} & 5.45E-1(4.71E-3) & 4.54E-2(6.42E-3) & 2.63E-2(5.50E-3) & 1.33E-2(4.96E-3) & 5.60E-1(8.19E-3) & 2.46E-1(1.07E-1) & 7.31E-1(2.77E-3) & 2.45E-1(5.14E-3) & 4.04E-1(1.03E-2) & 7.09E-1(2.12E-1) & 3.21E-2(6.23E-3) \\
 & 3d-CG & \textbf{3.82E-2(1.26E-3)}& 4.51E-1(3.35E-2) & 3.33E-1(2.64E-2) & 7.76E-2(1.63E-2) & 9.93E-1(2.91E-4) & 4.61E-1(4.46E-2) & 3.55E-1(7.75E-2) & 4.57E-1(5.07E-2)) & 7.96E-1(3.57E-4) & 4.60E-1(1.13E-2) & 3.82E-1(4.89E-2) & 6.91E-1(7.52E-2)\\
 & 2d-MS & \textbf{4.84E-2(1.52E-3)} & 7.60E-1(1.06E-2) & 7.49E-1(1.12E-2) & 7.93E-1(7.62E-2) & 7.26E-1(1.46E-2) & 7.84E-1(2.42E-2) & 6.94E-1(5.61E-2) & 7.41E-1(2.01E-2) & 9.61E-1(5.67E-2) & 6.31E-1(5.42E-2) & 9.04E-1(1.01E-1) & 9.94E-1(9.67E-5) \\ \hline
Heat & 2d-VC & \textbf{2.81E-2(6.46E-3)} & 1.12E+0(5.79E-2) & 2.41E-1(1.73E-2) & 2.07E-1(1.04E-3) & 2.03E-1(1.12E-2) & 1.06E+0(5.13E-2) & 5.45E-1(1.07E-1) & 2.41E+0(5.27E-1) & 8.79E-1(2.57E-1) & 7.49E-1(8.54E-2) & 9.91E-1(1.37E-1) & 9.44E-1(1.75E-3) \\
 & 2d-MS & \textbf{3.22E-2(1.42E-2)} & 9.30E-2(2.27E-2) & 2.90E-1(2.43E-2) & 1.13E-1(3.57E-2) & 6.69E-2(8.24E-3) & 1.19E-1(2.16E-2) & 3.00E-1(1.14E-1) & 1.80E-1(1.12E-2) & 9.25E-1(3.90E-2) & 1.14E-1(4.98E-2) & 1.08E+0(2.02E-1) & 5.33E-2(3.92E-3) \\
 & 2d-CG & \textbf{8.42E-3(2.71E-4)} & 3.05E-2(8.47E-3) & 1.37E-1(7.70E-3) & 1.12E-1(2.57E-3) & 1.07E-1(1.44E-2) & 2.21E-2(3.42E-3) & 5.88E-2(1.02E-2) & 8.20E-2(1.32E-2) & 3.09E+0(1.86E-2) & 1.94E-2(1.98E-3) & 3.77E-1(2.17E-1) & 6.77E-1(3.93E-2) \\
 & 2d-LT & \textbf{1.36E-1(4.34E-3)} & 9.98E-1(6.00E-5) & 9.98E-1(1.42E-4) & 9.98E-1(1.47E-4) & 9.99E-1(1.01E-3) & 9.98E-1(2.28E-4) & 9.99E-1(5.69E-5) & 9.98E-1(8.62E-4) & 9.98E-1(0.00E+0) & 9.98E-1(1.27E-4) & 9.98E-1(8.58E-5) & 1.01E+0(7.75E-4) \\ \hline
NS & 2d-C & \textbf{6.90E-3(7.17E-4)} & 5.08E-2(3.06E-3) & 1.84E-1(1.52E-2) & NA & 2.44E-1(3.05E-2) & 5.54E-1(1.24E-2) & 9.86E-1(3.16E-1) & 9.43E-2(3.24E-3) & 1.98E-1(7.81E-2) & 4.42E-2(7.38E-3) & 3.78E-2(8.71E-3) & 1.18E-1(3.10E-2) \\
 & 2d-CG & \textbf{9.62E-2(1.06E-3)} & 1.77E-1(1.00E-2) & 4.22E-1(8.72E-3) & 4.12E-1(6.93E-3) & 3.69E-1(2.46E-2) & 4.65E-1(4.44E-3) & 6.23E-1(8.86E-2) & 2.36E-1(1.15E-2) & 9.95E-1(3.50E-4) & 1.25E-1(1.42E-2) & 2.40E-1(8.01E-2) & 5.92E+0(5.65E-4) \\
 & 2d-LT & \textbf{8.51E-1(8.00E-4)} & 9.88E-1(1.86E-3) & 9.98E-1(4.68E-4) & 9.97E-1(3.64E-4) & 9.95E-1(6.66E-4) & 1.00E+0(2.46E-4) & 9.99E-1(9.27E-4) & 9.90E-1(3.60E-4) & 1.00E+0(1.40E-4) & 9.90E-1(3.78E-3) & 9.96E-1(2.68E-3) & 1.00E+0(1.38E-3) \\ \hline
Wave & 1d-C & \textbf{1.11E-2(2.87E-4)} & 5.87E-1(9.20E-2) & 2.78E-1(8.86E-3) & 3.49E-1(2.02E-2) & 9.42E-2(9.13E-3) & 5.40E-1(1.74E-2) & 1.15E-1(1.91E-2) & 5.60E-1(1.69E-2) & 1.41E+0(1.30E-1) & 4.38E-1(1.40E-2) & 6.82E-1(1.08E-1) & 6.55E-1(4.86E-2) \\
 & 2d-CG & \textbf{4.95E-1(1.23E-2)} & 1.96E+0(3.83E-1) & 1.78E+0(8.89E-2) & 1.58E+0(1.15E-1) & 2.34E+0(1.14E-1) & 1.16E+0(1.16E-1) & 1.09E+0(1.54E-1) & 7.22E-1(1.63E-2) & 1.08E+0(1.25E-1) & 7.45E-1(2.15E-2) & 7.08E-1(9.13E-3) & 1.15E+0(1.03E-1) \\
 & 2d-MS & \textbf{7.46E-2(8.35E-3)} & 2.04E+0(7.38E-1) & 1.10E+0(4.25E-2) & 1.08E+0(6.01E-2) & 1.13E+0(4.91E-2) & 2.08E+0(7.45E-1) & 1.07E+0(1.40E-2) & 1.11E+0(1.91E-2) & 1.05E+0(1.00E-2) & 1.17E+0(4.66E-2) & 1.12E+0(8.62E-2) & 1.29E+0(2.81E-2) \\ \hline
Chaotic & GS & \textbf{4.18E-3(6.93E-4)} & 3.45E-1(4.57E-1) & 1.29E-1(1.54E-1) & 2.01E-2(5.99E-5) & 1.11E-1(4.79E-2) & 2.98E-2(6.44E-3) & 2.00E-2(6.12E-5) & 2.72E-1(1.79E-1) & 1.04E+0(3.04E-1) & 2.07E-2(9.19E-4) & 1.16E-1(1.31E-1) & 5.06E-2(1.87E-2) \\
 & KS & 8.70E-1(8.52E-3) & 9.44E-1(8.57E-4) & 8.95E-1(2.99E-2) & \textbf{8.60E-1(3.48E-3)} & 8.64E-1(3.31E-3) & 9.42E-1(8.75E-4) & 8.73E-1(8.40E-3) & 9.36E-1(6.12E-3) & 8.88E-1(9.92E-3) & 9.39E-1(3.25E-3) & 9.44E-1(9.86E-3) & 9.85E-1(3.35E-2) \\ \hline
\end{tabular}
}
\label{main-l1re}
\end{minipage}
\end{table}

\clearpage
\eject \pdfpagewidth=45cm \pdfpageheight=22cm
\thispagestyle{empty}

\begin{table}[!ht]
\centering
\begin{minipage}{0.85\pdfpagewidth}
\tiny
\renewcommand{\arraystretch}{1.7}
\caption{Mean (std) of MSE for main experiments.}
\label{tab:14}
  \resizebox{\linewidth}{!}{
\begin{tabular}{cc|cccccccccccc}
\hline
MSE & Name & \multicolumn{1}{c|}{\multirow{2}{*}{Ours}} & \multicolumn{2}{c|}{Vanilla} & \multicolumn{3}{c|}{Loss Reweighting/Sampling} & \multicolumn{1}{c|}{Optimizer} & \multicolumn{2}{c|}{Loss functions} & \multicolumn{3}{c}{Architecture} \\ \cline{4-14} 
-- &  & \multicolumn{1}{c|}{} & PINN & \multicolumn{1}{c|}{PINN-w} & LRA & NTK & \multicolumn{1}{c|}{RAR} & \multicolumn{1}{c|}{MultiAdam} & gPINN & \multicolumn{1}{c|}{vPINN} & LAAF & GAAF & FBPINN \\ \hline
\multirow{2}{*}{Burgers} & 1d-C & \textbf{7.52E-5(1.53E-6)} & 7.90E-5(1.78E-5) & 2.64E-4(8.69E-5) & 3.03E-4(2.62E-4) & 1.30E-4(5.19E-5) & 5.78E-4(6.31E-4) & 9.68E-4(5.51E-4) & 1.77E-2(5.58E-3) & 5.13E-3(1.90E-3) & 1.80E-4(1.35E-4) & 3.00E-4(1.56E-4) & 1.53E-2(1.03E-2) \\
 & 2d-C & 2.31E-1(7.11E-2) & 1.69E-1(7.86E-4) & 1.17E-1(3.41E-3) & \textbf{1.09E-1(4.84E-3)} & 1.22E-1(4.22E-3) & 1.92E-1(5.07E-5) & 1.79E-1(9.36E-3) & 1.72E-1(1.31E-4) & 7.08E-1(5.16E-2) & 1.26E-1(1.54E-2) & 1.41E-1(1.12E-2) & -- \\ \hline
\multirow{4}{*}{Poisson} & 2d-C & \textbf{7.22E-6(1.03E-5)} & 1.17E-1(2.98E-3) & 3.09E-4(1.25E-4) & 7.24E-3(9.95E-3) & 5.00E-5(5.33E-5) & 1.19E-1(2.55E-3) & 1.79E-4(8.84E-5) & 1.15E-1(6.22E-3) & 4.86E-2(4.43E-3) & 1.39E-1(5.67E-3) & 9.38E-2(1.91E-2) & 7.89E-4(2.17E-4) \\
 & 2d-CG & \textbf{9.29E-6(7.92E-6)} & 1.28E-1(1.03E-3) & 1.17E-3(1.83E-4) & 6.13E-4(2.31E-4) & 6.99E-5(3.50E-5) & 1.32E-1(3.23E-3) & 2.73E-2(1.92E-2) & 1.98E-1(2.28E-3) & 2.50E-2(3.80E-4) & 7.67E-2(2.73E-3) & 1.77E-1(8.70E-2) & 4.84E-4(9.87E-5) \\
 & 3d-CG & \textbf{1.46E-4(5.29E-6)} & 2.64E-2(2.67E-3) & 1.18E-2(1.97E-3) & 9.51E-4(6.51E-4) & 7.54E-2(7.86E-5) & 2.81E-2(5.15E-3) & 1.16E-2(4.42E-3) & 2.01E-2(4.93E-3) & 4.58E-2(8.04E-5) & 2.82E-2(2.62E-3) & 2.16E-2(5.87E-3) & 4.63E-2(9.28E-3) \\
 & 2d-MS & \textbf{2.75E-2(9.75E-4)} & 2.67E+0(9.04E-2) & 3.90E+0(7.16E-2) & 4.28E+0(6.83E-1) & 3.77E+0(9.98E-2) & 2.80E+0(1.87E-1) & 2.36E+0(3.15E-1) & 2.56E+0(1.43E-1) & 6.09E+0(5.46E-1) & 1.83E+0(3.00E-1) & 5.87E+0(8.72E-1) & 6.68E+0(8.23E-4) \\ \hline
Heat & 2d-VC & \textbf{3.95E-5(1.54E-5)} & 4.00E-2(4.94E-3) & 2.19E-3(3.21E-4) & 1.76E-3(1.43E-5) & 1.79E-3(9.80E-5) & 3.67E-2(1.42E-3) & 9.14E-3(3.13E-3) & 1.89E-1(9.44E-2) & 3.23E-2(2.26E-2) & 1.74E-2(4.35E-3) & 2.93E-2(7.12E-3) & 3.56E-2(1.71E-4) \\
 & 2d-MS & \textbf{2.59E-5(1.80E-5)} & 1.09E-4(4.94E-5) & 1.60E-3(3.35E-4) & 2.25E-4(1.22E-4) & 5.27E-5(1.18E-5) & 1.54E-4(4.17E-5) & 1.51E-3(1.25E-3) & 3.43E-4(1.87E-5) & 2.57E-2(2.22E-3) & 1.57E-4(8.06E-5) & 3.10E-2(1.15E-2) & 2.17E-4(2.47E-5) \\
 & 2d-CG & \textbf{3.34E-4(5.02E-6)} & 2.09E-3(9.69E-4) & 3.15E-2(2.08E-3) & 2.32E-2(1.59E-3) & 2.02E-2(4.15E-3) & 1.12E-3(2.65E-4) & 7.79E-3(2.63E-3) & 1.34E-2(4.13E-3) & 1.16E+1(9.04E-2) & 8.53E-4(9.74E-5) & 3.94E-1(2.71E-1) & 5.61E-1(5.96E-2) \\
 & 2d-LT & \textbf{5.09E-2(4.88E-3)} & 1.14E+0(2.38E-5) & 1.13E+0(1.82E-4) & 1.14E+0(1.67E-4) & 1.14E+0(6.41E-4) & 1.14E+0(3.55E-4) & 1.14E+0(8.74E-5) & 1.14E+0(2.23E-4) & 1.14E+0(0.00E+0) & 1.14E+0(2.20E-4) & 1.14E+0(3.27E-4) & 1.16E+0(2.83E-4) \\ \hline
NS & 2d-C & \textbf{3.22E-6(1.23E-6)}  & 4.19E-5(2.00E-6) & 4.03E-4(6.45E-5) & NA & 7.56E-4(1.90E-4) & 4.18E-3(2.05E-4) & 1.07E-2(5.67E-3) & 1.13E-4(8.77E-6) & 5.30E-4(3.50E-4) & 2.33E-5(4.71E-6) & 2.67E-5(4.71E-6) & 1.37E-4(7.24E-5) \\
 & 2d-CG & \textbf{2.15E-4(8.21E-6)} & 6.94E-4(6.45E-5) & 5.19E-3(2.43E-4) & 5.40E-3(2.49E-4) & 4.22E-3(5.82E-4) & 5.45E-3(2.13E-5) & 9.32E-3(3.09E-3) & 1.16E-3(8.97E-5) & 1.06E+0(1.61E-2) & 3.37E-4(6.60E-5) & 1.72E-3(1.33E-3) & 3.34E+0(2.97E-5) \\
 & 2d-LT & \textbf{4.30E+2(4.00E-1)} & 5.06E+2(1.21E+0) & 5.10E+2(3.40E-1) & 5.10E+2(4.13E-1) & 5.09E+2(6.15E-1) & 5.10E+2(3.42E-1) & 5.10E+2(2.23E-1) & 5.05E+2(7.30E-1) & 5.11E+2(1.76E-2) & 5.06E+2(1.82E+0) & 5.11E+2(2.99E+0) & 5.15E+2(1.77E+0) \\ \hline
Wave & 1d-C & \textbf{5.08E-5(1.16E-6)} & 1.11E-1(3.66E-2) & 2.54E-2(1.61E-3) & 4.08E-2(4.31E-3) & 3.01E-3(4.82E-4) & 9.07E-2(6.02E-3) & 4.68E-3(1.28E-3) & 9.66E-2(5.85E-3) & 6.17E-1(1.19E-1) & 6.03E-2(2.87E-3) & 1.48E-1(4.44E-2) & 1.39E-1(1.97E-2) \\
 & 2d-CG & \textbf{1.59E-2(5.16E-4)} & 1.64E-1(6.13E-2) & 1.28E-1(1.13E-2) & 1.03E-1(1.46E-2) & 2.17E-1(2.05E-2) & 6.25E-2(1.17E-2) & 5.59E-2(1.29E-2) & 3.09E-2(8.98E-4) & 5.24E-2(9.01E-3) & 3.49E-2(3.38E-3) & 2.99E-2(4.68E-4) & 5.78E-2(7.99E-3) \\
 & 2d-MS & \textbf{2.20E+3(4.38E+2)} & 1.30E+5(4.25E+4) & 7.35E+4(1.68E+3) & 7.34E+4(1.97E+3) & 7.69E+4(4.55E+3) & 1.33E+5(4.47E+4) & 7.15E+4(8.04E+2) & 7.27E+4(5.47E+2) & 1.13E+2(1.46E+2) & 7.91E+4(2.55E+3) & 7.98E+4(8.00E+3) & 8.95E+5(1.15E+4) \\ \hline
Chaotic & GS & \textbf{1.04E-4(3.69E-5)} & 1.00E-1(1.35E-1) & 1.64E-2(1.70E-2) & 4.32E-3(4.07E-6) & 2.59E-2(1.44E-2) & 4.40E-3(8.83E-5) & 4.32E-3(1.11E-6) & 3.62E-2(2.28E-2) & 4.00E-1(2.33E-1) & 4.32E-3(4.71E-6) & 1.69E-2(1.79E-2) & 5.16E-3(1.64E-3) \\
 & KS & \textbf{1.03E+0(4.00E-3)} & 1.16E+0(2.95E-3) & 1.11E+0(5.07E-2) & 1.04E+0(6.20E-3) & 1.06E+0(1.09E-2) & 1.16E+0(1.98E-3) & 1.05E+0(1.04E-2) & 1.12E+0(8.67E-3) & 1.05E+0(2.50E-3) & 1.16E+0(4.50E-3) & 1.14E+0(2.33E-2) & 1.16E+0(5.28E-2) \\ \hline
\end{tabular}
}
\label{main-mse}

\end{minipage}
\end{table}

\end{document}